\newif\ifisTR

\isTRtrue

\documentclass{article} 

\usepackage{fullpage}

\usepackage{microtype}
\usepackage{graphicx}
\usepackage{booktabs} 
\usepackage{placeins}
\usepackage{algorithm}
\usepackage{algorithmic}
\usepackage{subcaption}
\usepackage{tcolorbox}
\usepackage{nicefrac}  
\usepackage{enumitem}
\usepackage{multirow}
\usepackage{colortbl}
\usepackage{xspace}
\usepackage{wrapfig}
\usepackage{amsmath}
\usepackage{amssymb}
\usepackage{mathtools}
\usepackage[export]{adjustbox}
\usepackage{amsthm}
\usepackage{xcolor}

\definecolor{CColor}{rgb}{0.01,0.31,0.59}
\definecolor{GGray}{rgb}{0.80,0.90,1}
\definecolor{Shady}{rgb}{0.9,0.9,0.9}
\definecolor{kaistblue}{RGB}{20,135,200}
\definecolor{kaistdarkblue}{RGB}{0,65,145}
\definecolor{urbanablue}{RGB}{19,41,75}
\definecolor{urbanaorange}{RGB}{232,74,39}
\definecolor{drp}{rgb}{0.53,0.15,0.34}
\usepackage[plainpages=false,pdfpagelabels,colorlinks=true,linkcolor=CColor,citecolor=CColor,urlcolor=CColor]{hyperref}

\usepackage{tikz}

\usepackage[capitalize,noabbrev]{cleveref}

\usepackage{overpic}

\theoremstyle{plain}
\newtheorem{theorem}{Theorem}[section]

\newtheorem{lemma}[theorem]{Lemma}
\newtheorem{corollary}[theorem]{Corollary}
\theoremstyle{definition}

\newtheorem{assumption}[theorem]{Assumption}
\theoremstyle{remark}

\definecolor{mygray}{gray}{0.85}
\definecolor{LightBlue}{cmyk}{0.06, 0.03, 0.01, 0.0}

\usepackage[textsize=tiny]{todonotes}

\usepackage[round,semicolon,sort]{natbib}

\DeclareMathOperator{\tr}{tr}	
\DeclareMathOperator{\vecc}{vec}	
\DeclareMathOperator{\st}{s.t.}

\newcommand{\RR}{\mathbb R}
\DeclareMathOperator{\RIP}{RIP}

\DeclarePairedDelimiter{\norm}{\lVert}{\rVert}

\newcommand{\projr}{\mathcal{P}_r}

\DeclareMathOperator{\dist}{dist}
\def\de{{\rm d}}

\usepackage{subcaption}
\usepackage{cleveref}

\renewcommand{\cite}[1]{\citep{#1}}
\date{}

\author{
\textbf{Xinyuan Song}\textsuperscript{1} \quad
\textbf{Ziye Ma}\textsuperscript{1†}
\\[6pt]
\textsuperscript{1}Department of Computer Science, CityUHK
}

\title{Matrix Sensing with Kernel Optimal Loss:  Robustness and Optimization Landscape}

\begin{document}
\maketitle

\def\thefootnote{†}\footnotetext{Correspondence to: Ziye Ma, ziyema@cityu.edu.hk}

\begin{abstract}
In this paper, we study how the choice of loss functions of non-convex optimization problems affects their robustness and optimization landscape, through the study of noisy matrix sensing. In traditional regression tasks, mean squared error (MSE) loss is a common choice, but it can be unreliable for non-Gaussian or heavy-tailed noise. To address this issue, we adopt a robust loss based on nonparametric regression, which uses a kernel-based estimate of the residual density and maximizes the estimated log-likelihood. This robust formulation coincides with the MSE loss under Gaussian errors but remains stable under more general settings. We further examine how this robust loss reshapes the optimization landscape by analyzing the upper-bound of restricted isometry property (RIP) constants for spurious local minima to disappear. Through theoretical and empirical analysis, we show that this new loss excels in handling large noise and remains robust across diverse noise distributions. This work provides initial insights into improving the robustness of machine learning models through simple loss modification, guided by an intuitive and broadly applicable analytical framework.
\end{abstract}

\tableofcontents

\clearpage

\vspace{-0.5em}
\section{Introduction}
\vspace{-0.5em}

Noisy low-rank matrix optimization arises in numerous settings, including matrix sensing~\cite{Recht2010}, matrix completion~\cite{Candes2009}, and robust PCA~\cite{Candes2011}. In practical applications such as recommender systems~\cite{Koren2009}, motion detection~\cite{Fattahi2020, Anderson2019}, phase synchronization and retrieval~\cite{Singer2011, Boumal2016, Shechtman2015}, and power system estimation~\cite{Zhang2017}, one often encounters an unknown positive semidefinite matrix \(M \in \mathbb{R}^{n \times n}\) of rank at most \(r\). Measurements are collected through a known linear operator \(\mathcal{A}\colon \mathbb{R}^{n\times n} \to \mathbb{R}^{m}\), and are corrupted by an additive noise vector \(w \in \mathbb{R}^{m}\) whose distribution may be unknown. A standard formulation is
\begin{equation}
\min_{M \in \mathbb{R}^{n \times n}} f(M, w)
\quad \text{subject to} \quad 
\operatorname{rank}(M) \le r,\quad M \succeq 0,
\label{eq:main}
\end{equation}
where \(f\) denotes a loss function evaluated at \((M,w)\). In recovery problems, the target quantity is \(M\), while the noise can vary widely in scale or distribution.

A widely used choice for \(f\) is the mean squared error (MSE),
\begin{equation}
f(M) = \tfrac{1}{2}\|\mathcal{A}(M) - \tilde{b}\|_{F}^{2},
\label{problem(2)}
\end{equation}
where \(\tilde{b} = \mathcal{A}(M^*) + w\) and \(\|\cdot\|_F\) is the Frobenius norm. Although the MSE objective performs well when the noise is Gaussian, its sensitivity to heavy-tailed corruption, outliers, or heterogeneous errors has been well documented~\cite{barron2019general,ghosh2017robust,hampel1986robust,huber1981robust}.

Because \(M\) is constrained to be low-rank and positive semidefinite, many algorithms adopt the Burer--Monteiro (BM) factorization~\cite{Burer2003}, which writes \(M = XX^\top\) for \(X \in \mathbb{R}^{n \times r}\). Substituting this representation into~\eqref{eq:main} yields an unconstrained non-convex problem:
\begin{equation}
\min_{X \in \mathbb{R}^{n \times r}} f\bigl(XX^\top, w\bigr).
\label{eq:burer}
\end{equation}
Theoretical studies~\cite{Li2020,Park2018,Ge2017,ma2023sharprestrictedisometryproperty} have characterized the geometry of~\eqref{eq:burer}, including properties of its local minima and its global recovery guarantees.

To improve robustness in the presence of unknown or irregular noise, we draw inspiration from nonparametric regression. In that setting~\cite{breiman2017classification, breiman2001random, cheng1984strong, devroye1994strong, nadaraya1964estimating, watson1964smooth, hall2001nonparametric, fan2018local, schumaker2007spline, berlinet2011reproducing, lv2018oracle}, one estimates an unknown function \(g\) from samples \(\{(\mathbf{X}_{i}, Y_{i})\}_{i=1}^{n}\) without assuming a parametric form for the noise. A kernel density estimator \(\hat{f}(\cdot)\) is used to construct the log-likelihood objective
\begin{equation}
\label{mle}
\hat g = \underset{g\colon \mathbb{R}^{d} \to \mathbb{R}}{\arg\max}\;
\frac{1}{n}\sum_{i=1}^{n} \log \hat{f}\bigl(Y_{i} - g(\mathbf{X}_{i})\bigr).
\end{equation}
The loss proposed in~\cite{wang2023deepregressionlearningoptimal} is Lipschitz continuous, avoids committing to a specific noise model, and reduces to the classical MSE estimator under Gaussian noise. Crucially, it continues to behave reliably when the corruption is heavy-tailed or heterogeneous.

In this work, we adapt the robust loss~\eqref{mle} to the BM formulation~\eqref{eq:burer} by applying the kernel estimator to the residuals 
\begin{equation}
\mathcal{A}(XX^\top) - \tilde{b},
\end{equation}
which play the same role as the terms \(Y_i - g(\mathbf{X}_{i})\) in nonparametric regression. In this correspondence, the unknown matrix \(M = XX^\top\) replaces the unknown function \(g\), and remains the central object of estimation throughout. We analyze the resulting optimization landscape, with particular attention to how the restricted isometry property (RIP) constants influence global recovery under noise. Both favorable and unfavorable RIP regimes are considered, and convergence properties for local search methods are derived. Numerical experiments under various corruption models demonstrate that this robust loss yields more stable estimation and improved convergence.

Our main contributions are as follows.  
(1) We provide theoretical recovery guarantees for \(M^*\) and establish convergence behavior of the kernel-based loss~\eqref{mle} in the BM setting. The results indicate enhanced robustness under heavy-tailed and heterogeneous noise.  
(2) We show that when the noise is Gaussian, the robust loss coincides with the MSE objective, thereby retaining its statistical efficiency.  
(3) Through empirical studies, we demonstrate accurate recovery and stable convergence across diverse corruption scenarios.

 \vspace{-0.5em}

\section{Preliminaries}
 \vspace{-0.5em}

\subsection{The Kernel Loss}
A second choice for the data fidelity term is a kernel-based loss derived from a log-likelihood formulation. Given observations \((X_i,Y_i)_{i=1}^n\), define:
\vspace{-0.5em}
\begin{equation}
\hat{g} = \arg\min_{g \in \mathcal{G}} \widehat{R}_n(g) = \arg\min_{g \in \mathcal{G}} n^{-1} \sum_{i=1}^n \left( -\log \frac{1}{n} \sum_{j=1}^n K_h\big(Y_j - g(X_j), Y_i - g(X_i)\big) \right).
\label{eq:newloss}
\end{equation}
This loss function is based on the log-likelihood formulation of a kernel density estimator using a kernel function \( K_h \), applied to the residuals. Unlike the mean squared error (MSE) loss, which lacks adaptability to varying noise distributions, the proposed loss function provides robustness by accommodating different noise settings. This characteristic ensures that the estimator remains statistically reliable in the large-sample regime. The proposed $\hat{\mathcal{R}}_n(g)$ involves a tuning parameter $h$. For implementation we use the exponential kernel
\vspace{-0.5em}
\begin{equation}
K_{h}(u,v) = \exp\bigl(-(u-v)^{2}/h^{2}\bigr).
\end{equation}
Then we use $\mathcal{A}(\cdot)$ to replace $g(\cdot)$ and BM factorization form to get the explicit form $\hat{g}$:
\begin{equation}\label{main0}
\hat{g} = n^{-1} \sum_{i=1}^n \left( -\log \frac{1}{n} \sum_{j=1}^n \exp\left(-\frac{((Y_j - (\mathcal{A}( X X^\top)_j)) - (Y_i - \mathcal{A}( X X^\top)_i))^2}{h^2}\right) \right).
\end{equation}
\vspace{-0.5em}
Using this loss, problem~\eqref{eq:main} becomes
\begin{equation}
\begin{aligned}
	\min_{M \in \mathbb{R}^{n \times n}} \hat{g}(M,w) \ \quad \text{s.t.}
\quad\mathrm{rank}(M) \leq r, M \succeq 0 
\end{aligned}
	\label{eq:main_noisy_problem}
\end{equation}
or, in factorized form,
\begin{equation}
	\min_{X \in \mathbb{R}^{n \times r}} \hat{g}(XX^\top,w),
    \end{equation}
where $\hat{g}$ is defined in~\eqref{eq:newloss}.

\subsection{RIP Conditions}

This line of work usually assumes the Restricted Isometry Property (RIP) for the problem, which is defined below:
\begin{equation}\label{RIP0}
(1 - \delta_p) \|X\|_F^2 \leq \|\mathcal{A} (X) \|_F^2 \leq (1 + \delta_p) \|X\|_F^2,
\end{equation}
here the $\|\cdot\|_F$ denotes the Frobenius norm, $\mathcal{A}$ is the linear operator, and  $\delta_p$ is the RIP-constant, which is usually simplified as $\delta$ in the following content.

\subsection{Lemmas on Noise}
\begin{assumption} \label{noise_perturb}
The noise $ w $ has a finite influence on the gradient and Hessian of the objective function in the sense that there exist two constants $ \zeta_1 \geq 0 $ and $ \zeta_2 \geq 0 $ such that
\begin{equation}
\bigl| \langle \nabla_M f(M, w) - \nabla_M f(M, 0), K \rangle \bigr| 
\leq 
\zeta_1 \|w\|_2 \|K\|_F, 
\label{eq:noise_perturb_grad}
\end{equation}
\begin{equation}
\bigl| \bigl[\nabla_M^2 f(M, w) - \nabla_M^2 f(M, 0)\bigr](K, L) \bigr| 
\leq 
\zeta_2 \|w\|_2 \|K\|_F \|L\|_F,
\label{eq:noise_perturb_hess}
\end{equation}
for all matrices $ M, K, L \in \mathbb{R}^{n \times n} $ with $ \text{rank}(M), \text{rank}(K), \text{rank}(L) \leq 2r $.
\end{assumption}
\vspace{-1.0em}
Assumption~\ref{noise_perturb} is the bounded noise assumption, for instance: $\|\langle \nabla_X \hat{g}(X + w) - \nabla_X \hat{g}(X), K \rangle \|$ or possibly $\|\langle \nabla^2_X \hat{g}\bigl(X + w\bigr)- \nabla^2_X \hat{g}\bigl(X\bigr), (K,L) \rangle \|$
is bounded with $\|X\|, \|K\|, \|L\|, \text{and} \|w\|$. In addition to Assumption~\ref{noise_perturb}, there exist other forms of noise assumptions used in the subsequent proof. The following two lemmas illustrate how alternative noise assumptions can be incorporated under fixed constants \( \rho \), \( \lambda_1 \), and \( \lambda_2 \). 
\begin{lemma}
\label{theorem_noise}
	There exists a constant $\rho$ such that the gradient of the function~\eqref{main0} $\hat{g}(\cdot,w)$ with respect to the first argument $M$ is $\rho-$restricted Lipschitz continuous, meaning that:
		\begin{equation}
			\| \nabla_M \hat{g}(M,w) - \nabla_M \hat{g}(M',w) \|_F \leq \rho \|M - M'\|_F
		\end{equation}
		for all matrices $M, M' \in \mathbb{R}^{n \times n}$ with $\mathrm{rank}(M) \leq r$ and $\mathrm{rank}(M') \leq r$.
\end{lemma}
\vspace{-1.0em}
The proof is provided in Section~\ref{sec:noise}.
Lemma~\ref{theorem_noise} establishes that the gradient of the function $\hat{g}(M, w)$ with respect to the variable $M$ satisfies a restricted Lipschitz continuity property over the set of rank-$r$ matrices. Specifically, there exists a constant $\rho > 0$ such that, for any pair of matrices $M, M' \in \mathbb{R}^{n \times n}$ with $\mathrm{rank}(M) \leq r$ and $\mathrm{rank}(M') \leq r$, the Frobenius norm of the gradient difference is bounded above by $\rho$ times the Frobenius norm of the matrix difference. This condition ensures that, within the low-rank manifold, the gradient of $\hat{g}$ does not change too rapidly, which is critical for the analysis of optimization algorithms constrained to low-rank structures.
\begin{lemma}
Gradient Lipschitz Continuity with Respect to Noise:
\begin{equation}
\|\nabla_M \hat{g}(M, w_1) - \nabla_M \hat{g}(M, w_2)\|_F \leq \lambda_1 \|w_1 - w_2\|_2.
\end{equation}
Hessian Lipschitz Continuity with Respect to Noise:
\begin{equation}
\|\nabla_M^2 \hat{g}(M, w_1) - \nabla_M^2 \hat{g}(M, w_2)\|_F \leq \lambda_2 \|w_1 - w_2\|_2.
\end{equation}
\label{other_assumption}
\end{lemma}
\vspace{-1.0em}
The proof is provided in Section~\ref{sec:other_assumption}. Based on the two lemmas we know that Assumption~\ref{noise_perturb} must hold in the kernel loss.

Lemma~\ref{other_assumption} formalizes the regularity of the function $\hat{g}(M, w)$ with respect to the noise variable $w$. The first inequality establishes that the gradient of $\hat{g}$ with respect to $M$ is Lipschitz continuous in $w$, with Lipschitz constant $\lambda_1$. This implies that small changes in the noise vector induce proportionally small changes in the gradient. The second inequality states that the Hessian of $\hat{g}$ with respect to $M$ is also Lipschitz continuous in $w$, with Lipschitz constant $\lambda_2$. Together, these conditions ensure that both first- and second-order derivatives of $\hat{g}$ vary in a controlled manner as a function of the noise, which is essential for stability and convergence guarantees in noise-sensitive optimization problems.

Due to space limitations, additional preliminaries and the background are provided in Appendix~\ref{prelimin}.

 \vspace{-0.5em}

\section{Comparison and Relationship of MSE Loss and the kernel loss}\label{relationships}
 \vspace{-0.5em}

\begin{assumption}\label{center}
    Assume the noise $w$ follows a centered symmetric distribution.
\end{assumption}
Based on Assumption~\ref{center} and heavy tailed assumptions and analysis in Appendix~\ref{prelimin}, we have:
\begin{theorem}[Noise sensitivity of $f(M,w)$ and $\hat{g}(M,w)$]
\label{thm:mainthm_intuition}
Let $w\in\mathbb{R}^n$ be the noise vector and assume $\|w\|_2\le\epsilon$ with probability at least $\mathbb{P}(\|w\|\le\epsilon)$ for some positive constant $\epsilon$. We measure noise sensitivity by the Euclidean norm of the gradient with respect to $w$, that is, $\|\nabla_w L(M,w)\|_2$. Then the sensitivities of the MSE loss $f(M,w)$ and the robust kernel loss $\hat{g}(M,w)$ are bounded as follows:
\begin{itemize}
\item[\textbf{(A)}]
\textbf{(Case: MSE loss $f(M,w)$)} For the MSE loss, the gradient norm satisfies
\begin{equation}
\bigl\|\nabla_w f(M,w)\bigr\|_2 = \mathcal{O} \left(\frac{\epsilon}{n}\right).
\end{equation}
\item[\textbf{(B)}]
\textbf{(Case: robust kernel loss $\hat{g}(M,w)$ in Equation~\eqref{main0})} For the robust kernel loss, there exists a constant $C>0$ (independent of $n$ and $\epsilon$) such that
\begin{equation}
\bigl\|\nabla_w \hat{g}(M,w)\bigr\|_2 = \mathcal{O} \left(\frac{\epsilon\,e^{-\epsilon^2/h^2}}{n h^2}\right),
\end{equation}
\vspace{-1.0em}
for all noise vectors $w$ with $\|w\|_2\le\epsilon$.
\end{itemize}
\end{theorem}
The detailed background together with the proof is provided in Section~\ref{sec:intuition}.

Theorem~\ref{thm:mainthm_intuition} compares the sensitivity of two loss functions-namely, the standard mean squared error (MSE) loss $ f(M, w) $ and a kernel-based robust loss $ \hat{g}(M, w) $—with respect to the noise variable $ w $. In case (A), for the MSE loss, the gradient with respect to $ w $ grows linearly with $\epsilon$, leading to a derivative of order $ \mathcal{O}(\epsilon / n) $, where $ n $ is the number of samples. This indicates that large noise directly amplifies the gradient, potentially causing instability in optimization.

In contrast, case (B) demonstrates that the derivative of the robust loss $\hat{g}(M, w)$ with respect to $ w $ is exponentially suppressed for large $\|w\|$, scaling as $ \mathcal{O}(\epsilon e^{-\epsilon^2 / h^2} / (n h^2)) $. The exponential decay introduces a natural attenuation of the influence of large noise components, making the loss function more robust to outliers. This distinction underscores the regularizing effect of the proposed loss and its improved stability in high-noise regimes.

Due to space limitations, additional preliminaries and background are provided in Appendix~\ref{relationship}.

 \vspace{-0.5em}

\section{Main Result}\label{Mainresult}
 \vspace{-0.5em}

We now officially characterize the recovery guarantees of ground truth $M^*$ under our new loss setting:
\begin{theorem}[Behavior of $\hat{g}(M,w)$]
\label{thm:mainthm}
Let $\hat{g}(\cdot,w)$ be our twice-differentiable loss function (Equation~\eqref{main0}), $M^\ast$ is the ground truth, and suppose $\nabla_M^2 \hat{g}(M^\ast,w)$ is positive definite at the ground truth point $M^\ast$. Suppose $\delta$ satisfies \eqref{bounds} in Section~\ref{condition} for the guarantee of no local minima, with probability at least $\mathbb{P}(\|w\| \leq \epsilon)$, if $M$ is a local minimizer of $\hat{g}$, then
\vspace{-1.0em}
\begin{equation}\label{optimss}
\|M-M^\ast\|_F \leq \mathcal{O}\Biggl(\max\Biggl\{1,\frac{\epsilon e^{-\epsilon^2}}{h^2}\Biggr\}\Biggr).
\end{equation}
Meanwhile, for equation:
\begin{equation}\label{thm:continue}
\|\nabla_M\hat{g}(M,w)-\nabla_M\hat{g}(M,0)\|\le\lambda \epsilon \quad \text{with:} \quad \lambda= \mathcal{O}\Big(\frac{8(1+\delta)}{h^4}\epsilon e^{-\epsilon^2} + C \Big).
\end{equation}
\vspace{-1.0em}
\end{theorem}
\vspace{-1.0em}
The proof is provided in Section~\ref{sec:mainthm}. Theorem~\ref{thm:mainthm} provides a probabilistic characterization of the behavior of the robust loss function $\hat{g}(M, w)$ in the vicinity of local minimizers. When the minimizer $M$ lies close to $M^\ast$, the estimation error $\|M - M^\ast\|_F$ is shown to be bounded above by $\mathcal{O}\left( \max \left\{ 1, \frac{\epsilon e^{-\epsilon^2}}{h^2} \right\} \right)$ with high probability, indicating that the error decays exponentially in the noise magnitude $\epsilon$. Moreover, the gradient variation with respect to noise is controlled via a Lipschitz constant $\lambda$, which itself satisfies an exponential decay bound given by $\lambda = \mathcal{O}\left(\frac{8(1+\delta)}{h^4} \epsilon e^{-\epsilon^2} + C\right)$. These results together highlight the smoothing effect of the exponential kernel in $\hat{g}$, which yields improved stability and robustness to noise in regions close to the ground truth.

\subsection{Turning Point of the Upper Bound}
Now we want to calculate the turning point of the Theorem~\ref{thm:mainthm} to illustrate the specific optimization landscape of $\|M-M^\ast\|_F$.
\begin{lemma}[Comparison of Terms in Theorem~\ref{thm:mainthm}]\label{lemma:turning_point}
Let $T_1 = \mathcal{O}(1)$ be the constant term in Equation~\eqref{optimss} and $T_2 = \mathcal{O}\left(\frac{\|w\|e^{-w^2}}{h^2}\right)$ be a noise-dependent term in Equation~\eqref{optimss}. Then with probability at least $\mathbb{P}(\|w\| \leq \epsilon)$: for very small $\epsilon$, specifically when $\epsilon \ll h^2$, we have
$T_2 = \mathcal{O}\left(\frac{\epsilon e^{-\epsilon^2}}{h^2}\right) = \mathcal{O}\left(\frac{\epsilon}{h^2}\right) \ll \mathcal{O}(1) = T_1$, Conversely, for small $\epsilon$, this roughly corresponds to $\epsilon \sim h^2$. In regimes where $h^2$ is significantly smaller than the maximum value of $\epsilon e^{-\epsilon^2}$ (which occurs at $\epsilon = \frac{1}{\sqrt{2}}$ with value $\frac{1}{\sqrt{2}}e^{-1/2}$), $T_2 \gg T_1$. When $\epsilon$ and $h$ satisfy $\epsilon e^{-\epsilon^2} \sim h^2$, then $T_2 \sim T_1$.
\end{lemma}
The proof is provided in Section~\ref{sec:turning}. Lemma~\ref{lemma:turning_point} analyzes the interplay between two key terms in the upper bound of the estimation error: a constant term $ T_1 = \mathcal{O}(1) $ and a noise-dependent term $ T_2 = \mathcal{O}\left(\frac{\epsilon e^{-\epsilon^2}}{h^2}\right) $. As illustrated in Figure~\ref{fig:loss_bound}, for very small noise magnitude $ \epsilon \ll h^2 $, the exponential term $ e^{-\epsilon^2} \sim 1 $, so $ T_2 \sim \frac{\epsilon}{h^2} \ll 1 $, making the constant term $ T_1 $ dominant. In this case, the estimation error remains effectively bounded and insensitive to small perturbations in $ \epsilon $. Conversely, as $ \epsilon $ increases and approaches the regime where $ \epsilon e^{-\epsilon^2} \sim h^2 $, the two terms become comparable, marking a critical threshold where noise starts to meaningfully affect the bound. Moreover, when $ h^2 $ becomes much smaller than the peak value of the function $ \epsilon e^{-\epsilon^2} $ (attained at $ \epsilon = 1/\sqrt{2} $), the term $ T_2 $ may become larger than $ T_1 $, and the error becomes dominated by the noise effect.


\subsection{MSE Loss Result Comparison}
Based on the above Lemma we have:

\begin{theorem}[Behavior of MSE loss function $f(M,w)$]
\label{thm:mainthm_MSE}
Let $f(\cdot,w)$ vanilla MSE loss function. $M^\ast$ is the ground truth. with probability at least $\mathbb{P}(\|w\| \leq \epsilon)$:

If $M$ is a local minimizer of $\hat{g}$, then
\begin{equation}
\|M-M^\ast\|_F \leq \mathcal{O}(\epsilon).
\end{equation}
Meanwhile, for equation:
\begin{equation}
\|\nabla_M f(M,w)-\nabla_M f(M,0)\|\le\lambda \epsilon \quad \text{with:} \quad \lambda= \mathcal{O}\Big(2\sqrt{1+\delta_p} \Big).
\end{equation}
\vspace{-1.0em}
\end{theorem}
\vspace{-1.0em}
The proofs are provided in Section~\ref{sec:MSE}. Theorem~\ref{thm:mainthm_MSE} describes the behavior of the standard MSE loss $ f(M, w) $ near the ground truth $ M^\ast $, which is a tighter bound than in~\cite{ma2023noisylowrankmatrixoptimization}. Unlike the robust loss, the MSE does not benefit from exponential decay in noise and is more sensitive to outliers. Also, in Theorem~\ref{thm:mainthm}, the new loss $\hat{g}(M,w)$ is smoother (it has a smaller Lipschitz constant $\lambda$), and consequently $M$ transitions toward $M^\ast$ when $\epsilon$ decreases. In contrast, in Theorem~\ref{thm:mainthm_MSE}, as $\epsilon$ grows, $\|M-M^\ast\|_F$ can grow linearly in $\epsilon$. Here the new loss exhibits a rougher landscape (with a larger Lipschitz constant $\rho$), and hence when $\epsilon$ decreases, $M$ may actually move away from $M^\ast$. The loss function still becomes smoother as $\epsilon$ changes, but the local minimum remains distant from the ground truth.

\vspace{-0.5em}

\section{The Condition of \texorpdfstring{$\delta$}{}}\label{condition}
 \vspace{-0.5em}

\begin{lemma}[$\delta$ condition with explicit choice of bandwidth Parameter $ h $]
\label{lemma:error_bound_with_h}
Assume Assumption~\ref{noise_perturb} and~\ref{other_assumption} hold and use the definitions in Lemma~\ref{lemma:delta_bound_with_w} in Appendix~\ref{sec:delta}. Let the bandwidth be chosen as $h = \frac{\sqrt{2} B}{\sqrt{G_{\min}}}$.
Then, the quantity $ \delta $ must satisfy
\begin{equation}
\delta \leq \sqrt{\frac{B^2}{4(G_{\min} + 2)}\left(2 - \frac{G}{\sigma_r} - \frac{L_2}{B}\right)} - 1.
\end{equation}\label{bounds}
In particular, under the conservative (worst-case) assumption, this reduces to
\vspace{-0.5em}
\begin{equation}
\delta \le \frac{B}{\sqrt{2(G_{\min} + 2)}} - 1 \sim \frac{1}{3}.
\end{equation}
\vspace{-0.5em}
\end{lemma}
The proof is provided in Section~\ref{sec:delta}. Lemma~\ref{lemma:error_bound_with_h} provides an explicit upper bound on the parameter $\delta$ in terms of the kernel bandwidth $h$, by setting $h = \frac{\sqrt{2}B}{\sqrt{G_{\min}}}$, and shows that $\delta$ depends on a combination of spectral and residual quantities; in particular, this value can be calculated explicitly as around $1/3$, which aligns well with Ma's $\delta$ bound~\cite{ma2023noisylowrankmatrixoptimization}. The bound simplifies to $\delta \le \frac{B}{\sqrt{2(G_{\min} + 2)}} - 1$, illustrating how the choice of $h$ directly regulates estimation stability through the tradeoff between noise, kernel concentration, and curvature. As shown in Figure~\ref{fig:delta_bound}, the value of $\delta$ decreases monotonically as $w$ increases when $h$ is held constant. For fixed $w$, increasing $h$ suggests a non-monotonic dependence on the smoothing parameter. 

The $\delta$ bound shown in Equation~\eqref{bounds} corresponds directly to the theoretical setting (as in Theorem~\ref{thm:mainthm}) in which the convergence guarantees for ground truth recovery. Specifically, these results are valid only when $\delta$ remains below the threshold, which aligns with the classical requirement $\delta < 1/2$, as stated in~\cite{ma2023noisylowrankmatrixoptimization}. When $\delta$ exceeds this threshold, matrix recovery is no longer guaranteed to succeed in recovering the ground truth. Therefore, in the following sections, we shift our focus to analyzing the behavior of the recovery landscape in the regime where $\delta$ is above $1/2$. We only consider the optimization landscape in a region around the ground truth and show that local minimizers are all very close to $M^*$.


\vspace{-1.0em}
\begin{figure}[htbp]
    \centering
    \begin{minipage}[t]{0.49\linewidth}
        \centering
        \includegraphics[width=\linewidth]{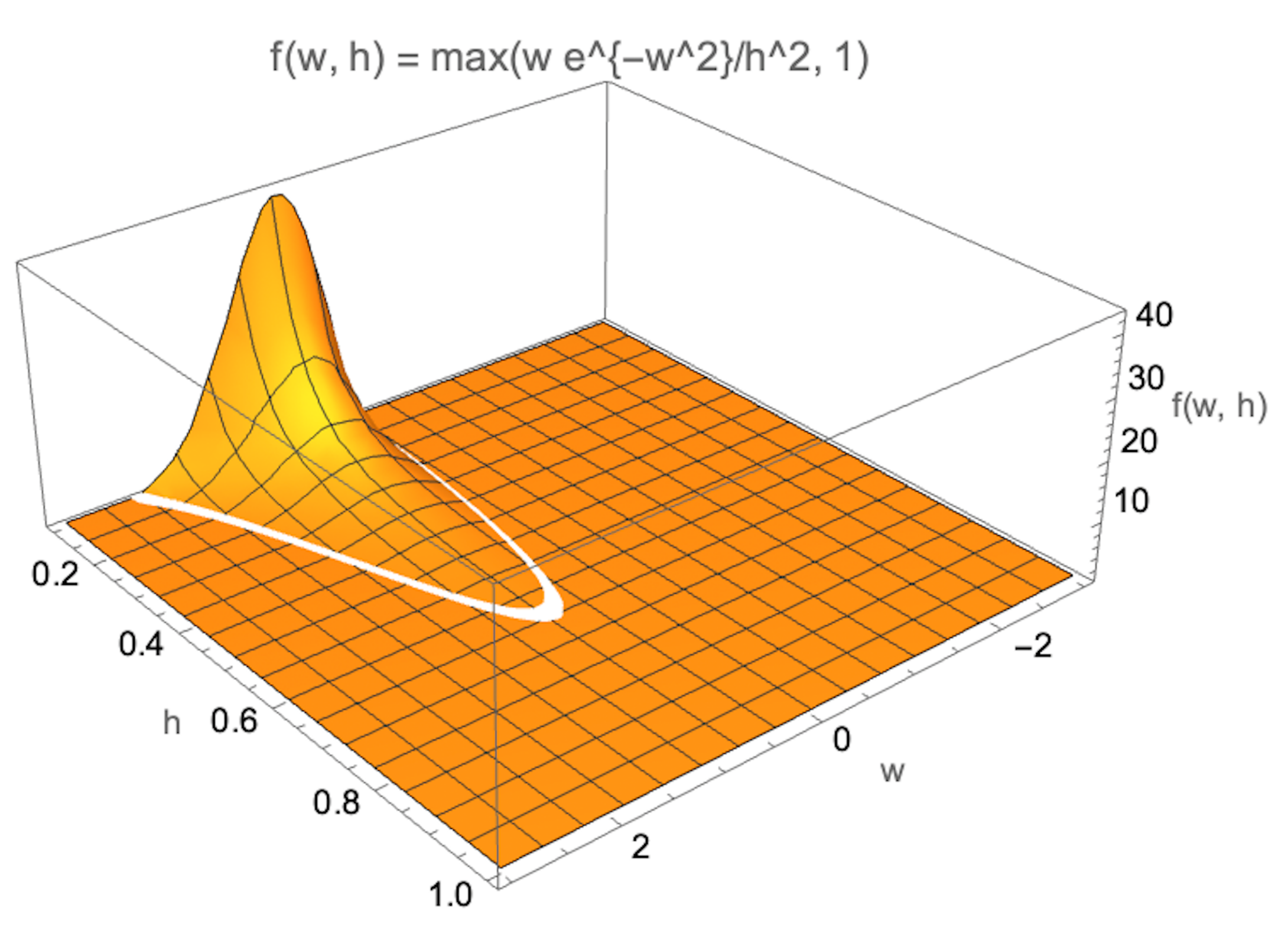}
        \caption{The kernel loss bound result: the yellow mesh depicts the three-dimensional surface of the upper bound on \(\|M - M^*\|_F\), expressed as a function of the noise magnitude \(\|w\|\) and the kernel bandwidth \(h\).}
        \label{fig:loss_bound}
    \end{minipage}
    \hfill
    \begin{minipage}[t]{0.49\linewidth}
        \centering
        \includegraphics[width=\linewidth]{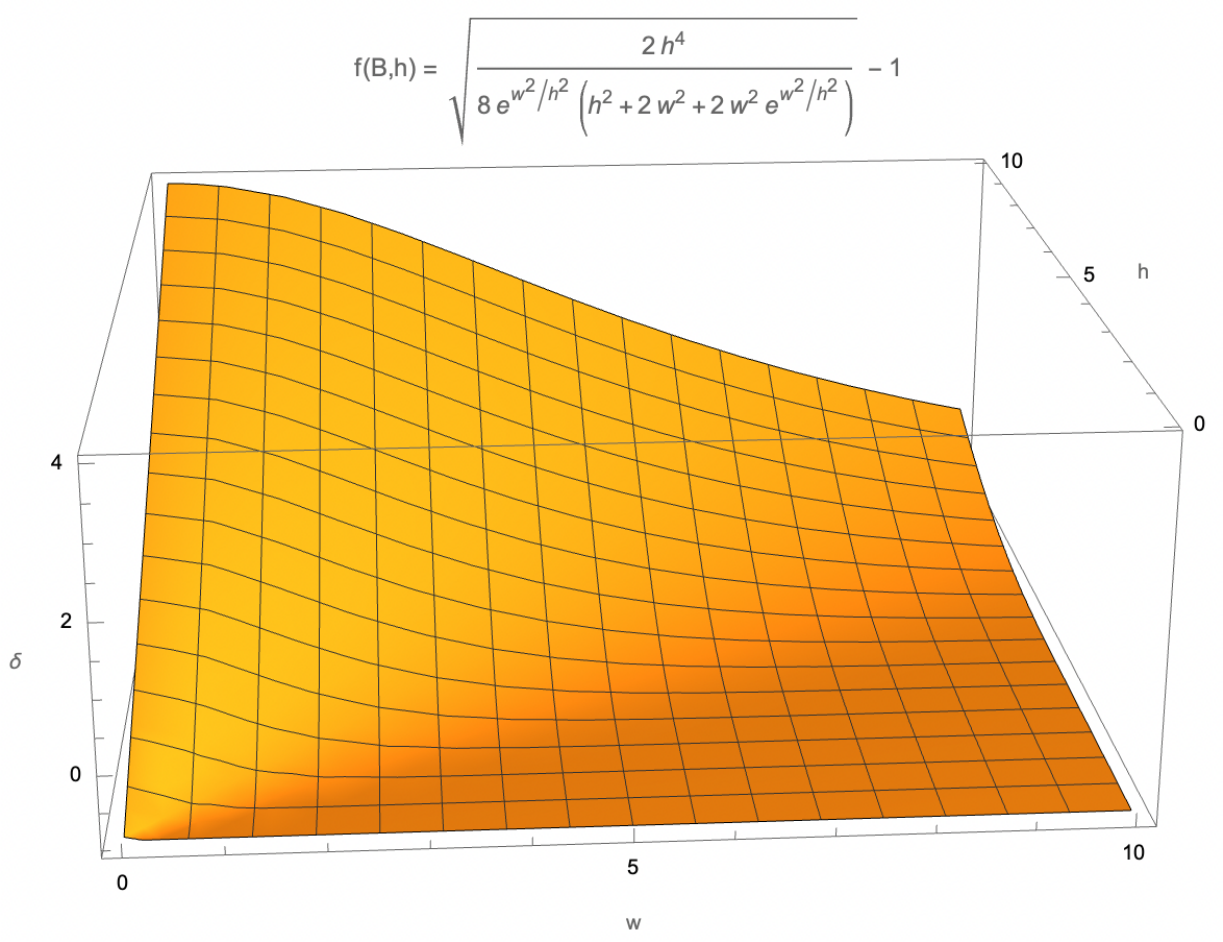}
        \caption{Our \(\delta\) bound result is illustrated in the figure, where the yellow mesh represents the boundary surface of \(\delta\) as a function of the noise magnitude \(\|w\|\) and the kernel bandwidth \(h\).}
        \label{fig:delta_bound}
    \end{minipage}
    \vspace{-1.0em}
\end{figure}

 \vspace{-0.5em}

\section{Upper Bound When \texorpdfstring{$\delta> 1/2$}{}}
 \vspace{-0.5em}

\subsection{The Kernel Loss Function}

\begin{theorem}[Upper Bound for $\delta>1/2$]\label{ours_large_delta}
Let \( \hat{M} \in \mathbb{R}^{n \times n} \) be an estimator of the true matrix \( M^\ast \), and with the new loss function (Equation~\eqref{main0}), assume that: the derivative terms \( \nabla_M u_{ij}(M) \), \( \nabla_M^2 u_{ij}(M) \) are bounded in norm by \( L_1, L_2 \), respectively, and the quantity \( G_i(M) \) is uniformly lower bounded: \( G_i(M) \ge \Gamma_{\min} > 0 \). Also assume the minimum pairwise squared residual satisfies \( u_{\min}^2 := \min_{i,j,t} u_{ij}(M_t)^2 \).
Then with probability at least \( \mathbb{P}(\|w\|_2 \leq \epsilon) \) the estimation error \( \|\hat{M} - M^\ast\|_F \) satisfies the bound:
\begin{equation}\label{eq:24}
\|\hat{M} - M^\ast\|_F \le \frac{-\left[\zeta_1(1+\delta) - B \epsilon^2\zeta_2\right] + \sqrt{\left[\zeta_1(1+\delta) - B \epsilon^2 \zeta_2\right]^2 + 8B \epsilon^2 \zeta_1\zeta_2(1-\delta)}}{4\zeta_2},
\end{equation}
where
\begin{equation}
B = \frac{\sqrt{2\lambda_{r^\ast}(\hat{X}\hat{X}^\top)}}{\|Q\|} \left[ \frac{L_1^2\delta^2}{h^4 \Gamma_{\min}^2} \exp\left(-\frac{2u_{\min}^2}{h^2}\right) + \frac{L_2\delta}{h^2 \Gamma_{\min}} \exp\left(-\frac{u_{\min}^2}{h^2}\right) \right].
\end{equation}
and \( Q := \hat{X}U^\top + U\hat{X}^\top \). Here, \( \zeta_1, \zeta_2 \) are structural noise constants as in Assumptions~\ref{other_assumption} and~\ref{noise_perturb}, and \( \lambda_{r^\ast}(\hat{X}\hat{X}^\top) \) denotes the \( r^\ast \)-th eigenvalue of the data covariance matrix.

\end{theorem}
The proofs and detailed math are provided in Section~\ref{large_delta_ours}.

\begin{corollary}
    The $\|\hat{M} - M^\ast\|_F$ in Theorem~\ref{ours_large_delta} can be roughly written as:
    \begin{equation}
\|\hat{M} - M^\ast\|_F \le \mathcal{O} \Bigl(-\left[1 - B \epsilon^2\right] + \sqrt{\left[1 - B \epsilon^2 \right]^2 +  B \epsilon^2 } \Bigr).
\end{equation}
\end{corollary}
It is a simplified result and the proof is omitted here due to page limit.

\subsection{MSE Loss Function}

\begin{lemma}[Estimation Bound under MSE Loss]\label{MSE_loss_result}
Assume that the objective function is given by the mean squared error (MSE) loss and that it satisfies Assumptions~\ref{other_assumption} and~\ref{noise_perturb}. Further suppose that the noise-free mapping \( f(M, 0) \) satisfies the \(\delta\)-restricted isometry property (RIP) for some constant \( \delta \in (0,1) \). Let \( \tau \in (0, 1 - \delta^2) \) be arbitrary. We have, with probability at least \( \mathbb{P}(\|w\|_2 \leq \epsilon) \), the refined upper bound
\begin{equation}
\| M- M^\ast \|_F \leq 
\mathcal{O}\left(\frac{\epsilon(1+ \epsilon)}{ \sqrt{1- \epsilon}} \right).
\end{equation}
\end{lemma}
\vspace{-1em}
The proofs and detailed math are provided in Section~\ref{sec:local}, which alignes well with Ma's result~\cite{ma2023noisylowrankmatrixoptimization}. The upper bounds for the proposed loss in Theorem~\ref{ours_large_delta} and MSE loss in Theorem~\ref{MSE_loss_result} illustrate a direct generalization of the results in~\cite{bi2020delocalization}, Theorem~\ref{ours_large_delta} shows that even when \(\delta > 1/2\), a recovery guarantee still holds under more strict conditions on the model and the noise structure. Hence, unlike certain MSE-based approaches that may fail to provide meaningful error bounds once \(\delta\) exceeds the \(1/2\) threshold, the new kernel-based method retains theoretical validity in this regime. For scenarios in which \(\delta\) is only marginally above \(1/2\), MSE might still be competitive if its assumptions are not severely violated, but it generally does not offer the same level of robustness provided by the new loss.

 \vspace{-0.5em}

\section{Lower Bound for the \texorpdfstring{$\|M - M^\ast\|_F$}{}}\label{lowerbounds}
 \vspace{-0.5em}

Now we want to provide the lower bound for the result of the kernel loss and MSE loss for comparison. The lower bound refers to a threshold below which the Frobenius norm of the error, \(\|M - M^\ast \|_F\), cannot decrease unless the estimated matrix \(M\) coincides with the ground truth \(M^\ast\). Thus, if \(M \neq M^\ast\), then \(\|M - M^\ast \|_F\) must lie above a positive value

\subsection{MSE Loss Function}

\begin{theorem}[Lower Bound under MSE Loss]\label{MSE_lower}
Let \( M^\ast \in \mathbb{R}^{n \times n} \) be the ground truth matrix. \( L > 0 \), \( \delta \in (0,1) \), and \( w \) is the noise term. Provided that \( L > 2(1 + \delta) \), then with probability at least \( \mathbb{P}(\|w\|_2 \leq \epsilon) \) either \( M = M^\ast \) (i.e., perfect recovery), or the Frobenius norm of the error is lower bounded by
\begin{equation}
\|M - M^\ast\|_F \ge \frac{4\sqrt{1+\delta} \epsilon}{L - 2(1+\delta)}.
\end{equation}
\end{theorem}
The proofs and mathematics are provided in~\ref{MSE_lowers}.

\subsection{The Kernel Loss Function}
\begin{theorem}[Lower Bound for the kernel loss]\label{ours_lower}
Let \( M^* \) be the ground truth matrix. Then, with probability at least \( \mathbb{P}(\|w\|_2 \leq \epsilon) \) either \( M = M^* \) (i.e., exact recovery), or the Frobenius norm of the error satisfies the lower bound:
\begin{equation}
\|M - M^*\|_F \ge \frac{2L(1 + \delta) {e^{-\epsilon^2}}}{(1 - \delta) h^2},
\end{equation}
here \( L > 0 \), and \( 0 < \delta < 1 \).
\end{theorem}
The proofs and mathematics are provided in~\ref{ours_lowers}. Comparing Theorems~\ref{MSE_lower} and~\ref{ours_lower}, we observe that the kernel loss yields a tighter and more informative lower bound for the recovery error. Specifically, the bound in Theorem~\ref{ours_lower} scales favorably with respect to both the noise level $\epsilon$ and the kernel bandwidth $h$, and decays exponentially in $\epsilon^2$, which reflects stronger robustness to small noise. In contrast, the lower bound under the MSE loss in Theorem~\ref{MSE_lower} grows linearly with $\epsilon$ and inversely with the gap $L - 2(1+\delta)$, which can be loose in high-noise or near-threshold regimes. Therefore, the proposed kernel-based formulation not only enhances recovery in practice but also enjoys a more favorable theoretical guarantee in terms of error lower bounds.

However, a closer look at Theorem~\ref{ours_lower} also shows that when \(\epsilon\) is very small, the exponential term \(e^{-\epsilon^2}\) does not significantly decrease, and the resulting lower bound can remain comparatively large. In such low-noise settings, this may hinder the solver’s ability to achieve a close approximation to \(M^*\), making the MSE approach more favorable in that specific scenario. This explains why we combine the two losses.

\vspace{-1.0em}

\section{Combined Loss}
\vspace{-1.0em}

\renewcommand{\arraystretch}{2}  
\setlength{\tabcolsep}{0pt} 
\begin{table}[!ht]
\centering
\caption{Comparison of theoretical properties the between the kernel loss and MSE loss, with probability at least \( \mathbb{P}(\|w\|_2 \leq \epsilon) \)}
\label{tab:loss_comparison}
\begin{tabular}{@{}lcc@{}}
\toprule
\midrule
\textbf{Property} & \textbf{The Kernel Loss} & \textbf{MSE Loss} \\
\midrule
Loss Behavior (Section~\ref{relationships}) &  $\mathcal{O}\Bigl(\frac{\epsilon e^{-\frac{\epsilon^2}{h^2}}}{n h^2}\Bigr)$ & $\mathcal{O}(\frac{\epsilon}{n})$  \\
Optimization Landscape (Section~\ref{Mainresult}) & $\mathcal{O}\Bigl(\max\Bigl\{1,\frac{\epsilon e^{-\epsilon^2}}{h^2}\Bigr\}\Bigr)$ & $\mathcal{O}(\epsilon)$ \\
Continuity Result (Section~\ref{Mainresult}) & $\mathcal{O}\Big(\frac{8(1+\delta)}{h^4}\epsilon e^{-\epsilon^2} + C \Big)$ & $2\sqrt{1+\delta_p}$ \\
$\delta > \frac{1}{2}$ (Section~\ref{condition}) & $\mathcal{O} \Bigl(-\left[1 - B \epsilon^2\right] + \sqrt{\left[1 - B \epsilon^2 \right]^2 +  B \epsilon^2 } \Bigr)$ & $\mathcal{O}\left(\frac{\epsilon(1+ \epsilon)}{ \sqrt{1- \epsilon}} \right)$ \\
Lower Bound (Section~\ref{lowerbounds}) & $\frac{2L(1 + \delta) {e^{-\epsilon^2}}}{(1 - \delta) h^2}$ & $\frac{4\sqrt{1+\delta} \epsilon}{L - 2(1+\delta)}$ \\
Convergence Result (Section~\ref{convergences}) & $\eta \le \frac{h^2}{ \left( 12 \rho^{1/2} C_0\right)}$ & $\eta \le \frac{1}{\left( 12 \rho^{1/2} C_0\right)}$ \\
\midrule
\bottomrule
\end{tabular}
\vspace{-1.0em}
\end{table}
 
The Convergence analysis is provided in Section~\ref{convergences}, and the analysis of the behavior of the two loss functions under non-centered noise is provided in Section~\ref{relationship}. Table~\ref{tab:loss_comparison} summarizes the theoretical properties of the kernel loss and the standard MSE loss. Compared to MSE, the kernel loss exhibits smoother behavior, provides benign optimization landscapes, and maintains more robust theoretical guarantees even when the RIP constant exceeds $\frac{1}{2}$. It also offers tighter lower bounds and provable convergence under heavy-tailed noise, highlighting its robustness and theoretical advantages.

As discussed in Table~\ref{tab:loss_comparison}, a comprehensive comparison between the mean squared error (MSE) and our proposed loss reveals that each has distinct advantages and limitations. To leverage the strengths of both, we construct a combined loss function that incorporates MSE and the proposed kernel-based loss, weighted by a learnable trade-off parameter $\lambda$. The resulting objective function is given in Equation~\eqref{combined}. A concise theoretical analysis of the combined formulation is provided in Appendix~\ref{sec:combined_theory}, where we highlight its optimization properties. In the following empirical study, we evaluate the performance of three loss functions: the MSE loss, the kernel loss, and the combined loss.
 \vspace{-0.5em}
\begin{equation}\label{combined}
\begin{aligned}
&\mathbf{L_{\mathrm{combined}}}(X) =  \lambda \cdot \frac{1}{n}\sum_{i=1}^n \bigl(Y_i - \mathcal{A}(X X^\top)_i\bigr)^2 \\
&+ (1-\lambda) \cdot n^{-1} \sum_{i=1}^n \left( -\log \frac{1}{n} \sum_{j=1}^n \exp\left(-\frac{((Y_j - (\mathcal{A}( X X^\top)_j)) - (Y_i - \mathcal{A}( X X^\top)_i))^2}{h^2}\right) \right). 
\end{aligned}
\end{equation}
\vspace{-2em}
\section{Empirical Study Of the Three Different Loss}\label{empirical}
 \vspace{-1em}

In this section, we provide an illustrative example of the theoretical results. We examine the proximity of an arbitrary local minimizer $\hat{X}$ of the three losses to the ground truth in terms of $\|\hat{X}\hat{X}^\top - M^*\|_F$, and study the effect of the step size on convergence. The setting follows the 1-bit Matrix Completion problem, a low-rank optimization task commonly used in recommendation systems~\cite{davenport2014one,ghadermarzy2018learning}, with the same experimental configuration as~\cite{ma2023noisylowrankmatrixoptimization}. Additional implementation details are given in Appendix~\ref{exp:result}. 

Figure~\ref{fig:loss-comparison} compares the bounds in Theorem~\ref{thm:mainthm} and Theorem~\ref{thm:mainthm_MSE} with $n=40$ and $r=5$. The y-axis reports the distance from an arbitrary local minimum $\hat{M}$ to $M^*$, measured in units of $\lambda_r(M^*)$, while the x-axis represents the probability lower bound, corresponding to the quantile of the noise norm $\|w\|$. The numerical results lie in the regime $\delta < 1/3$. The real error is the Frobenius norm of the recovered matrix under additive noise, and the numerical error is the corresponding computed upper bound. The kernel loss maintains stable error as the noise increases, with mild compression in some cases, consistent with Theorem~\ref{thm:mainthm}. In contrast, the MSE loss exhibits an almost linear increase in error as the noise strengthens, matching Theorem~\ref{thm:mainthm_MSE}. This confirms that the kernel loss can limit the influence of large noise and preserves robustness.

For low noise, the MSE loss achieves smaller error than the kernel loss, revealing a trade-off between precision and robustness. The composite loss combines the strengths of both objectives. As shown in the last columns of Figure~\ref{fig:loss-comparison}, it reduces the noise sensitivity of MSE while improving the low-noise performance of the kernel loss, yielding more reliable recovery across a broader range of noise levels.
\begin{figure}[!ht]
    \vspace{-1em} 
    \centering

    \includegraphics[width=0.29\textwidth]{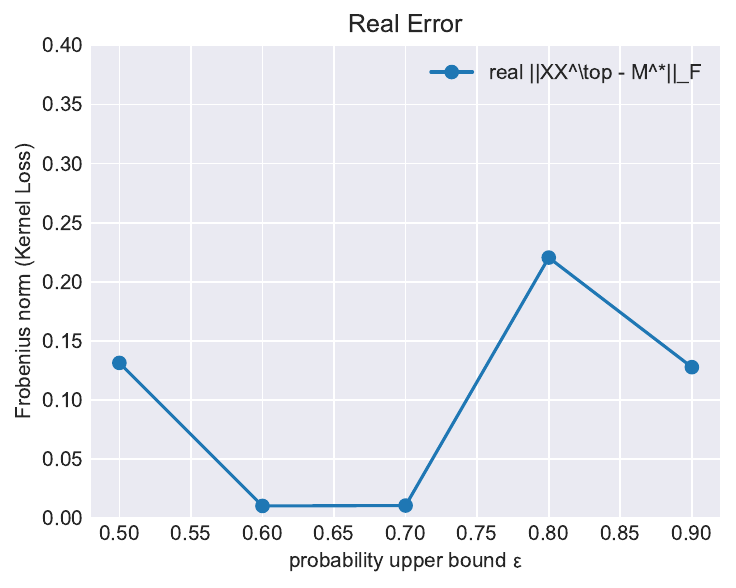}
    \includegraphics[width=0.29\textwidth]{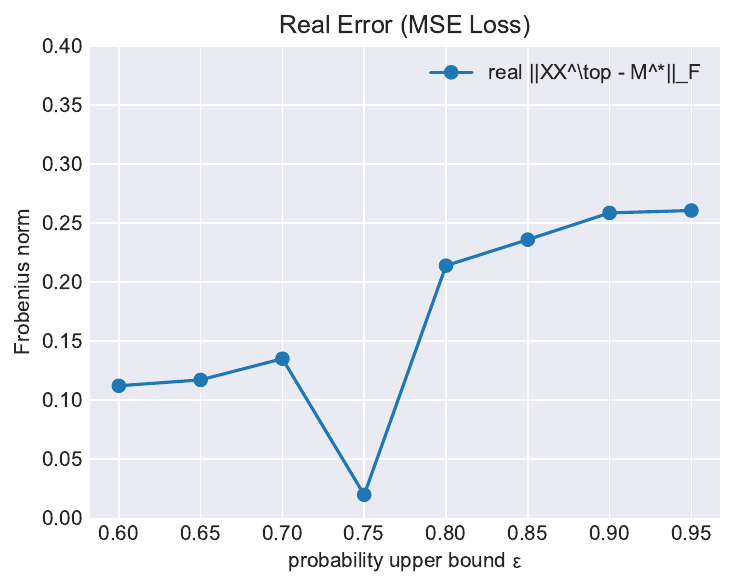}
    \includegraphics[width=0.29\textwidth]{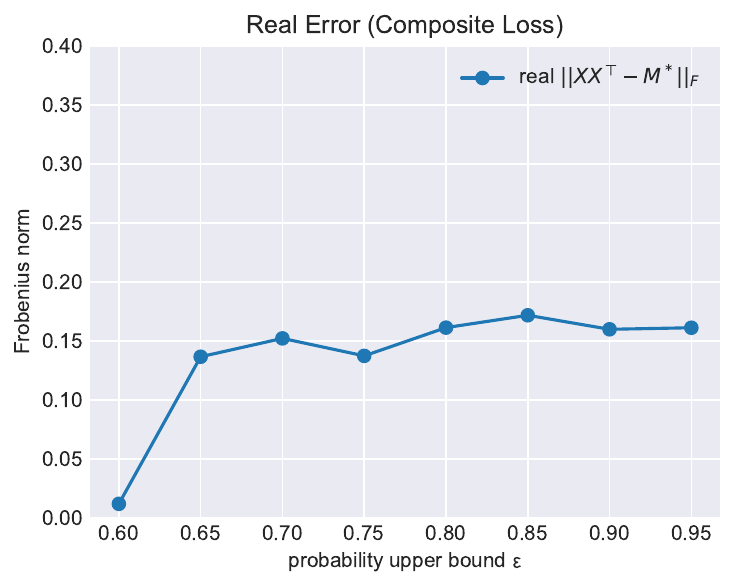}

    \includegraphics[width=0.29\textwidth]{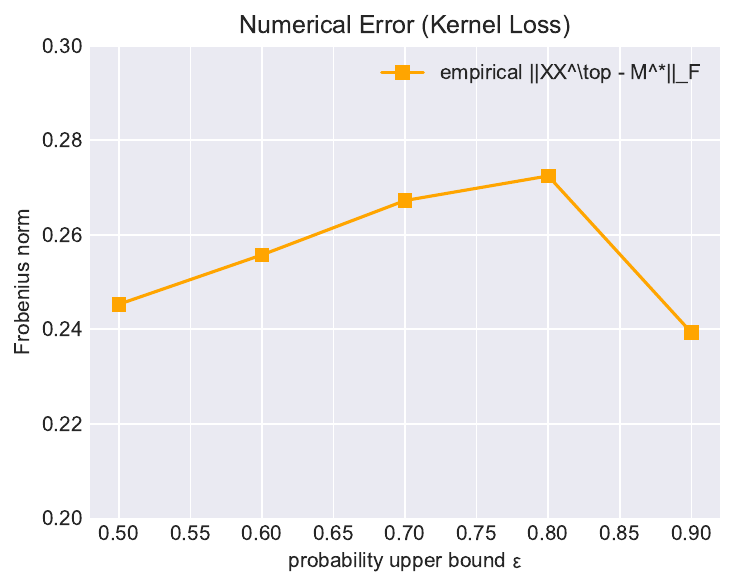}
    \includegraphics[width=0.29\textwidth]{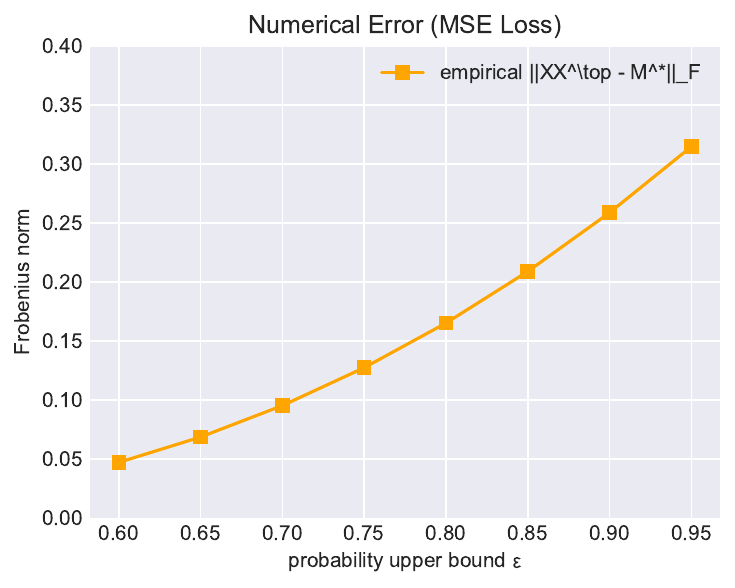}
    \includegraphics[width=0.29\textwidth]{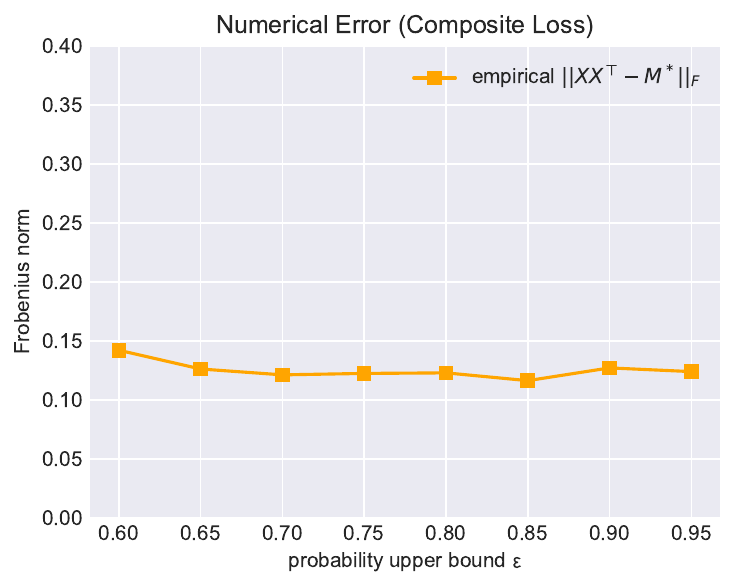}

    \vspace{-0.6em} 
    \caption{Comparison of real (top row) and numerical (bottom row) errors under different loss functions: kernel loss (left), MSE loss (middle), and composite loss (right), for $\delta < 1/3$.}
    \vspace{-1em} 
    \label{fig:loss-comparison}
\end{figure}
Readers may refer to Appendix~\ref{exp:result} for more visualizations of how the values of $\zeta_1$, $\zeta_2$, and Lipschitz constant $L$ affect Theorem~\ref{thm:mainthm}, especially under different uniform noise levels. More empirical results for $\delta>1/3$ are also provided in Appendix~\ref{exp:result}.
 \vspace{-0.5em}

\section{Conclusion}
 \vspace{-0.5em}

We compared the MSE loss and a kernel-based loss for low-rank matrix recovery in the presence of heavy-tailed noise. The exponential decay term in the proposed kernel-based loss reduces the impact of large outliers, whereas the MSE loss exhibits a linearly growing gradient. Our theoretical findings show that the kernel-based loss has more favorable upper and lower bounds in noisy regimes, and a combined formulation further balances these properties. Empirical results support these conclusions, indicating that the kernel-based and combined losses perform reliably under varying noise levels and values of $\delta$.

\bibliographystyle{plainnat}
\bibliography{name}

\clearpage
\appendix

\begin{center}
    \LARGE \textbf{Appendix}
\end{center}

\section{Convergence Behavior Under Uniform Noise}
\label{sec:convergence}

Figures~\ref{fig:loss1} and~\ref{fig:loss2} illustrate the loss curves obtained under uniform noise sampled from the interval $[0,1]$. The left plot corresponds to a setting where the final error converges to $0.5032$, with a theoretically computed upper bound of $0.803$, a Lipschitz constant of $2.2825922$, and a Hessian constant of $3.3899682$. In contrast, the right plot achieves a lower final error of $0.40832$, with a corresponding upper bound of $0.9104874$, a Lipschitz constant of $2.2984617$, and a Hessian constant of $3.390177$.

\begin{figure}[!ht]
\centering
\begin{minipage}{0.5\textwidth}
    \centering
    \resizebox{\linewidth}{!}{
        \includegraphics{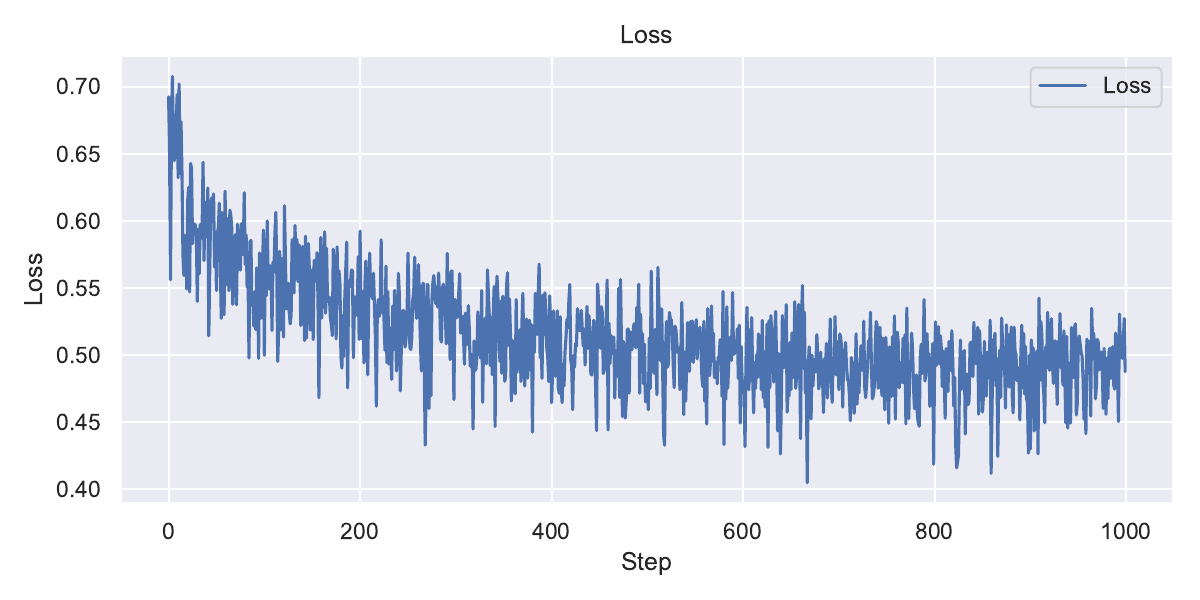}
    }
    \caption{Loss under uniform noise setting 1.}
    \label{fig:loss1}
\end{minipage}
\hspace{-0.02\textwidth} 
\begin{minipage}{0.5\textwidth}
    \centering
    \resizebox{\linewidth}{!}{
        \includegraphics{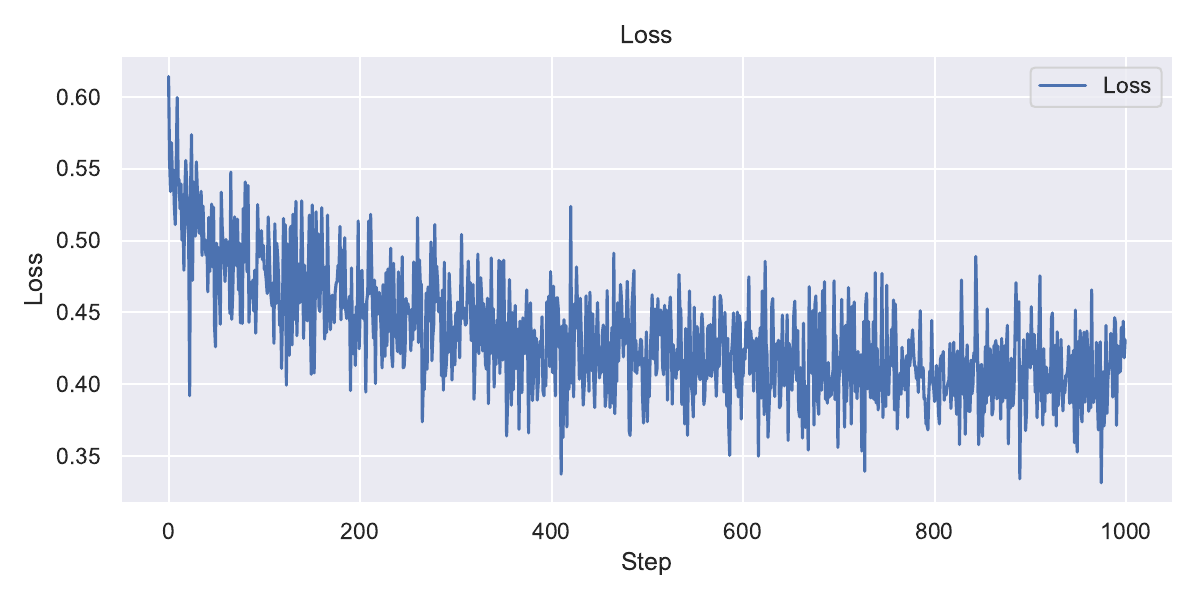}
    }
    \caption{Loss under uniform noise setting 2.}
    \label{fig:loss2}
\end{minipage}
\end{figure}

The observed decrease in the loss values over time indicates a favorable convergence behavior. In both scenarios, the empirical loss steadily approaches a stable minimum, validating the optimization stability of the proposed approach. The final errors lie significantly below the respective theoretical upper bounds, suggesting that the training dynamics are well-controlled and that the theoretical estimates are conservative. These results confirm the practical viability of the method in the presence of stochastic perturbations.

\section{Additional preliminaries}\label{prelimin}

\subsection{Difference and Geometry of Different Noise Bound}

\subsubsection{\texorpdfstring{$\bigl\| \langle \nabla_M \hat{g}(M, w) - \nabla_M \hat{g}(M, 0), K \rangle \bigr\|_F \leq \zeta_1 \|\mathbf{w}\|_2 \|K\|_F$}{}}

Here, for a fixed model parameter $M$, the change in the gradient $\nabla_M \hat{g}(M, w)$ when moving from $w = 0$ to a nonzero noise vector $\mathbf{w}$ is measured by its projection onto some direction $K$, expressed via the inner product $\langle \cdot, K\rangle$. This inequality indicates that when the noise changes from $0$ to $\mathbf{w}$, the change in the gradient (with respect to $M$) in any direction $K$ is bounded by $\zeta_1 \|\mathbf{w}\|_2 \|K\|_F$.

\subsubsection{\texorpdfstring{$\| \nabla_M \hat{g}(M,w) - \nabla_M \hat{g}(M',w) \|_F \leq \rho \|M - M'\|_F$}{}}

Here, for a fixed noise vector $w$, we compare how the gradient $\nabla_M \hat{g}(M,w)$ changes when the model parameter moves from $M$ to $M'$. This describes a $\rho$-Lipschitz continuity of the gradient with respect to $M$. This inequality shows that if the model parameter $M$ changes slightly (while $w$ is fixed), then the change in $\nabla_M \hat{g}(M,w)$ is controlled by $\|M - M'\|_F$ with the proportionality constant $\rho$.

\subsubsection{\texorpdfstring{$\|\nabla_M \hat{g}(M, w_1) - \nabla_M \hat{g}(M, w_2)\|_F \leq \lambda_1 \|w_1 - w_2\|_2$}{}}

Here, for a fixed model parameter $M$, we compare how the gradient $\nabla_M \hat{g}(M, w)$ changes when the noise vector goes from $\mathbf{w}_1$ to $\mathbf{w}_2$. This shows that when the noise vector $\mathbf{w}$ changes, the gradient $\nabla_M \hat{g}(M,w)$ with respect to $M$ changes by at most $\lambda_1 \|\mathbf{w}_1 - \mathbf{w}_2\|_2$.

\subsection{Definitions and Notations}
In this paper:
\begin{itemize}[left =0em]
\item $ I_n $ refers to the identity matrix of size $ n \times n $.
\item $ M \succeq 0 $ means that $ M $ is a symmetric and positive semidefinite matrix.
\item $ \sigma_i(M) $ denotes the $ i $-th largest singular value of a matrix $ M $, and $ \lambda_i(M) $ denotes the $ i $-th largest eigenvalue of $ M $.
\item $ \|v\| $ denotes the Euclidean norm of a vector $ v $, while $ \|M\|_F $ and $ \|M\|_2 $ denote the Frobenius norm and the operator norm of a matrix $ M $, respectively.
\item The inner product $ \langle A, B \rangle $ is defined as $ \text{tr}(A^\top B) $ for two matrices $ A $ and $ B $ of identical dimensions.
\item For a matrix $ M $, $ \text{vec}(M) $ is the usual vectorization operation by stacking the columns of $ M $ into a vector.
\end{itemize}

The Hessian of the function $ \hat{g}(\cdot, \cdot) $ with respect to the first argument $ M $, denoted as $ \nabla_M^2 \hat{g}(\cdot, \cdot) $, can be regarded as a quadratic form whose action on any two matrices $ K, L \in \mathbb{R}^{n \times n} $ is given by
\begin{equation}
[\nabla_M^2 \hat{g}(M, w)](K, L) 
= 
\sum_{i,j,k,l=1}^n \frac{\partial^2 \hat{g}}{\partial M_{ij} \partial M_{kl}}(M, w) K_{ij} L_{kl}.
\end{equation}
In this paper, $ \nabla_M^2 \hat{g}(M) $ and $ \nabla^2 \hat{g}(M, w) $ are used interchangeably since $ w $ is an unknown fixed parameter, and it is impossible to take a derivative with respect to $ w $.

\subsection{Distance to Low-Rank Matrices}
Define $ M^\ast \in \arg \min_M f(M, 0) $. We also characterize the distance of an arbitrary factorized point $ X \in \mathbb{R}^{n \times r} $ to a rank-$ r $ positive semidefinite matrix $ M $ with the function $ \text{dist}(X, M) $, defined as:
\begin{equation}
\text{dist}(X, M) = \min_{Z \in \mathcal{Z}} \|X - Z\|_F, 
\end{equation}
where
\begin{equation}
\mathcal{Z} = \{Z \in \mathbb{R}^{n \times r} \mid M = ZZ^\top\}.
\end{equation}
Given a matrix $ \hat{X} \in \mathbb{R}^{n \times r} $, define $ \hat{X} \in \mathbb{R}^{n^2 \times nr} $ to be the matrix satisfying
\begin{equation}
\hat{X}\text{vec}(U) = \text{vec}(\hat{X} U^\top + U \hat{X}^\top), 
\quad \forall U \in \mathbb{R}^{n \times r}.
\end{equation}

\subsection{Projection onto a Low-Rank Manifold}
Define $ \mathcal{P}_r(M) $ of an arbitrary matrix $ M $ to be the projection of $ M $ onto a low-rank manifold of rank at most $ r $:
\begin{equation}
\mathcal{P}_r(M) = \arg \min_{M_r \in \mathcal{M}} \|M_r - M\|_F,
\end{equation}
where
\begin{equation}
\mathcal{M} := \{M \in \mathbb{S}^{n \times n} \mid \text{rank}(M) \leq r, M \succeq 0\}.
\end{equation}
For problem~\eqref{problem(2)}, $ \mathcal{A} \in \mathbb{R}^{m \times n^2} $ is defined such that $ \mathcal{A}\text{vec}(M) = \mathcal{A}(M) $.
Finally, define:
\begin{equation}
h(X, w) := f(XX^\top, w).
\end{equation}
The old loss function, it appears this is a quadratic loss: $ h(X) = \sum_i \bigl(\mathcal{A}(\mathbf{X} \mathbf{X}^\top)_i - b_i\bigr)$. 

\begin{theorem}
\label{sym_PSD}
	The objective function $\hat{g}(\cdot,w)$ of \eqref{eq:main_noisy_problem} has a first-order critical point $M^w$ for every $w$ such that it is symmetric, positive semidefinite, and $\mathrm{rank}(M^w) \leq r$.
\end{theorem}
The proof is provided in Section~\ref{sec:PSD}.

\subsection{Heavy Tailed Analysis}
We start with
for fixed $x$, the probability $P(r_i>x) \geq P(p_i>x)$, 
\begin{equation}\label{effective}
\frac{\partial L_i}{\partial p_i} = -\frac{2}{h^2} \frac{1}{nZ_i} \sum_{j=1}^n (p_j-p_i) \exp\Bigl(-\frac{(p_j-p_i)^2}{h^2}\Bigr).
\end{equation}
\begin{equation}
\frac{\partial L_i}{\partial r_i} = -\frac{2}{h^2} \frac{1}{nZ_i} \sum_{j=1}^n (r_j-r_i) \exp\Bigl(-\frac{(r_j-r_i)^2}{h^2}\Bigr).
\end{equation}
Now we want to prove $\frac{\partial L_i}{\partial r_i} \geq \frac{\partial L_i}{\partial p_i}$. By Heavy tailed definition and we assume that $r_i$ is more heavy tail than $p_i$, for every real number $x$:
\begin{equation}
P(r_i > x) \ge P(p_i > x).
\end{equation}
here $r_i$ and $p_i$ are noise. A standard result is that if $\phi$ is any (measurable) function which is non‐decreasing then under mild integrability conditions one has
\begin{equation}
\mathbb{E}[\phi(r_i)] \ge \mathbb{E}[\phi(p_i)].
\end{equation}
For a fixed index $i$ one may imagine that the set $\{p_j\}_{j=1}^n$ (or $\{r_j\}_{j=1}^n$) provides independent draws from the underlying distribution. Then for any non‐decreasing function $\phi$ we expect, probabilistically, that
\begin{equation}
E\bigl[\phi(r_j -r_i)\bigr] \leq E\bigl[\phi(p_j-p_i)\bigr].
\end{equation}

Thus we can prove that for larger $r$, which is also $w$, the effective loss~\eqref{effective} is larger.

\section{Proof of Lemma~\ref{theorem_noise}}\label{sec:noise}
\begin{proof}
In~\eqref{main0}, we let:
\begin{equation}\label{fij}
f_{ij}(M) = \exp\left(-\frac{((Y_j - (\mathcal{A}(M)_j)) - (Y_i - \mathcal{A}(M)_i))^2}{h^2}\right),
\end{equation}
Then:
\begin{equation}\label{logfij}
\hat{g}(M) = -\frac{1}{n} \sum_{i=1}^n \log \left( \frac{1}{n} \sum_{j=1}^n f_{ij}(M) \right).
\end{equation}
So we get:
\begin{equation}
\nabla_M \hat{g}(M) = -\frac{1}{n} \sum_{i=1}^n \frac{1}{\frac{1}{n} \sum_{j=1}^n f_{ij}(M)} \cdot \frac{1}{n} \sum_{j=1}^n \nabla_M f_{ij}(M),
\end{equation}

\begin{equation}
f_{ij}(M) = \exp\left(-\frac{((Y_j - (\mathcal{A}(M)_j)) - (Y_i - \mathcal{A}(M)_i))^2}{h^2}\right).
\end{equation}
We have $u_{ij}(M) = (Y_j - (\mathcal{A}(M)_j)) - (Y_i - \mathcal{A}(M)_i)$,$f_{ij}(M) = \exp\left(-\frac{u_{ij}(M)^2}{h^2}\right).$ The gradient of $f_{ij}(M)$ with respect to $M$ is:
\begin{equation}
\nabla_M f_{ij}(M) = f_{ij}(M) \cdot \left(-\frac{2 u_{ij}(M)}{h^2}\right) \cdot \nabla_M u_{ij}(M).
\end{equation}
In this process: $u_{ij}(M) = (Y_j - (\mathcal{A}(M)_j)) - (Y_i - \mathcal{A}(M)_i).$,$\nabla_M u_{ij}(M) = -\nabla_M \mathcal{A}(M)_j + \nabla_M \mathcal{A}(M)_i.$ Using the expressions derived above, we have:
\begin{equation}
\nabla_M \hat{g}(M) = -\frac{1}{n} \sum_{i=1}^n \frac{1}{\frac{1}{n} \sum_{j=1}^n f_{ij}(M)} \cdot \frac{1}{n} \sum_{j=1}^n f_{ij}(M) \cdot \left(-\frac{2 u_{ij}(M)}{h^2}\right) \cdot \nabla_M u_{ij}(M).
\end{equation}
Similarly for $M'$:
\begin{equation}
\nabla_M \hat{g}(M') = -\frac{1}{n} \sum_{i=1}^n \frac{1}{\frac{1}{n} \sum_{j=1}^n f_{ij}(M')} \cdot \frac{1}{n} \sum_{j=1}^n f_{ij}(M') \cdot \left(-\frac{2 u_{ij}(M')}{h^2}\right) \cdot \nabla_M u_{ij}(M').
\end{equation}
And:
\begin{equation}
\begin{aligned}
\nabla_M \hat{g}(M) -& \nabla_M \hat{g}(M') = -\frac{1}{n} \sum_{i=1}^n \Bigg( 
\frac{1}{\frac{1}{n} \sum_{j=1}^n f_{ij}(M)} \cdot \frac{1}{n} \sum_{j=1}^n f_{ij}(M)  \cdot \left(-\frac{2 u_{ij}(M)}{h^2}\right) \cdot \nabla_M u_{ij}(M) \Bigg) \\
& -\Bigg( 
\frac{1}{\frac{1}{n} \sum_{j=1}^n f_{ij}(M')} \cdot \frac{1}{n} \sum_{j=1}^n f_{ij}(M')  \cdot \left(-\frac{2 u_{ij}(M')}{h^2}\right) \cdot \nabla_M u_{ij}(M') 
\Bigg).
\end{aligned}
\end{equation}
Let:
\begin{equation}
A_i(M) = \frac{1}{\frac{1}{n} \sum_{j=1}^n f_{ij}(M)},
\end{equation}
\begin{equation}
B_{ij}(M) = f_{ij}(M) \cdot \left(-\frac{2 u_{ij}(M)}{h^2}\right) \cdot \nabla_M u_{ij}(M).
\end{equation}
Then:
\begin{equation}
\nabla_M \hat{g}(M) = -\frac{1}{n} \sum_{i=1}^n A_i(M) \cdot \frac{1}{n} \sum_{j=1}^n B_{ij}(M).
\end{equation}
Similarly for $M'$:
\begin{equation}
\nabla_M \hat{g}(M') = -\frac{1}{n} \sum_{i=1}^n A_i(M') \cdot \frac{1}{n} \sum_{j=1}^n B_{ij}(M').
\end{equation}
The difference is:
\begin{equation}\label{eq:53}
\nabla_M \hat{g}(M) - \nabla_M \hat{g}(M') = -\frac{1}{n} \sum_{i=1}^n \left( A_i(M) \cdot \frac{1}{n} \sum_{j=1}^n B_{ij}(M) - A_i(M') \cdot \frac{1}{n} \sum_{j=1}^n B_{ij}(M') \right).
\end{equation}
Bounding the Difference. We can split Equation~\eqref{eq:53} into two parts:
\begin{equation}
\begin{aligned}
A_i(M) &\cdot \frac{1}{n} \sum_{j=1}^n B_{ij}(M) 
 - A_i(M') \cdot \frac{1}{n} \sum_{j=1}^n B_{ij}(M') \\
= & \left( A_i(M) - A_i(M') \right) \cdot \frac{1}{n} \sum_{j=1}^n B_{ij}(M)  + A_i(M') \cdot \left( \frac{1}{n} \sum_{j=1}^n B_{ij}(M) - \frac{1}{n} \sum_{j=1}^n B_{ij}(M') \right).
\end{aligned}
\end{equation}
Term 1: $\left( A_i(M) - A_i(M') \right) \cdot \frac{1}{n} \sum_{j=1}^n B_{ij}(M)$. Since $A_i(M) = \frac{1}{\frac{1}{n} \sum_{j=1}^n f_{ij}(M)}$, we have:
\begin{equation}
\|A_i(M) - A_i(M')\|_F \leq \left\| \frac{1}{\frac{1}{n} \sum_{j=1}^n f_{ij}(M)} - \frac{1}{\frac{1}{n} \sum_{j=1}^n f_{ij}(M')} \right\|_F.
\end{equation}
Using the mean value theorem for the function $f(x) = \frac{1}{x}$, we get:
\begin{equation}
\begin{aligned}
\left\| \frac{1}{\frac{1}{n} \sum_{j=1}^n f_{ij}(M)} - \frac{1}{\frac{1}{n} \sum_{j=1}^n f_{ij}(M')} \right\|_F \leq \frac{C}{\left( \frac{1}{n} \sum_{j=1}^n f_{ij}(M) \right)^2} \cdot \left\| \frac{1}{n} \sum_{j=1}^n f_{ij}(M) - \frac{1}{n} \sum_{j=1}^n f_{ij}(M') \right\|_F,
\end{aligned}
\end{equation}
where $C$ is a constant. Since $f_{ij}(M)$ is Lipschitz continuous with respect to $M$, we have:
\begin{equation}
\left| f_{ij}(M) - f_{ij}(M') \right| \leq L_f \|M - M'\|_F,
\end{equation}
where $L_f$ is the Lipschitz constant. Thus:
\begin{equation}\label{eq:58}
\left\| \frac{1}{n} \sum_{j=1}^n f_{ij}(M) - \frac{1}{n} \sum_{j=1}^n f_{ij}(M') \right\|_F \leq L_f \|M - M'\|_F.
\end{equation}
Combining Equation~\eqref{eq:58}, we get:
\begin{equation}
\|A_i(M) - A_i(M')\|_F \leq \frac{C L_f \|M - M'\|_F}{\left( \frac{1}{n} \sum_{j=1}^n f_{ij}(M) \right)^2}.
\end{equation}
Term 2: $A_i(M') \cdot \left( \frac{1}{n} \sum_{j=1}^n B_{ij}(M) - \frac{1}{n} \sum_{j=1}^n B_{ij}(M') \right)$. Since $B_{ij}(M)$ involves $f_{ij}(M)$, $u_{ij}(M)$, and $\nabla_M u_{ij}(M)$, we need to bound each term:
\begin{equation}
\left\| B_{ij}(M) - B_{ij}(M') \right\|_F \leq L_B \|M - M'\|_F,
\end{equation}
where $L_B$ is a Lipschitz constant that depends on the Lipschitz constants of $f_{ij}(M)$, $u_{ij}(M)$, and $\nabla_M u_{ij}(M)$. Thus:
\begin{equation}\label{eq:61}
\left\| \frac{1}{n} \sum_{j=1}^n B_{ij}(M) - \frac{1}{n} \sum_{j=1}^n B_{ij}(M') \right\|_F \leq L_B \|M - M'\|_F.
\end{equation}
Combining the bounds for both terms in Equation~\eqref{eq:61} , we get:
\begin{equation}
\begin{aligned}
&\left\| \nabla_M \hat{g}(M) - \nabla_M \hat{g}(M') \right\|_F \\
&\leq \frac{1}{n} \sum_{i=1}^n \left( \frac{C L_f \|M - M'\|_F}{\left( \frac{1}{n} \sum_{j=1}^n f_{ij}(M) \right)^2} \cdot \frac{1}{n} \sum_{j=1}^n B_{ij}(M) + A_i(M') \cdot L_B \|M - M'\|_F \right).
\end{aligned}
\end{equation}
Factoring out $\|M - M'\|_F$, we get:
\begin{equation}
\begin{aligned}
&\left\| \nabla_M \hat{g}(M) - \nabla_M \hat{g}(M') \right\|_F \\
&\leq \left( \frac{1}{n} \sum_{i=1}^n \left( \frac{C L_f}{\left( \frac{1}{n} \sum_{j=1}^n f_{ij}(M) \right)^2} \cdot \frac{1}{n} \sum_{j=1}^n B_{ij}(M) + A_i(M') \cdot L_B \right) \right) \|M - M'\|_F.
\end{aligned}
\end{equation}
Thus, we have:
\begin{equation}
\| \nabla_M \hat{g}(M) - \nabla_M \hat{g}(M') \|_F \leq \rho \|M - M'\|_F,
\end{equation}
where $\rho$ is a constant that depends on the Lipschitz constants $L_f$ and $L_B$, and the terms involving $A_i(M)$ and $B_{ij}(M)$.

\end{proof}

\section{Proof of Lemma~\ref{other_assumption}}\label{sec:other_assumption}
\subsection{Jacobian Case}
Recall that the function~\ref{main0} regard to the gradient $\nabla_M \hat{g}(M, w)$ is:
\begin{equation}
\nabla_M \hat{g}(M, w) = -\frac{1}{n} \sum_{i=1}^n \frac{1}{\sum_{j=1}^n f_{ij}(M, w)} \sum_{j=1}^n \frac{\partial f_{ij}(M, w)}{\partial M},
\end{equation}
where $f_{ij}(M, w) = \exp\left(-\frac{((Y_j + w_j - (A M)_j) - (Y_i + w_i - (A M)_i))^2}{h^2}\right)$
\begin{equation}
\begin{aligned}
&\nabla_M \hat{g}(M, w_1) - \nabla_M \hat{g}(M, w_2) = \\
&-\frac{1}{n} \sum_{i=1}^n \left( \frac{1}{\sum_{j=1}^n f_{ij}(M, w_1)} \sum_{j=1}^n \frac{\partial f_{ij}(M, w_1)}{\partial M} - \frac{1}{\sum_{j=1}^n f_{ij}(M, w_2)} \sum_{j=1}^n \frac{\partial f_{ij}(M, w_2)}{\partial M} \right).
\end{aligned}
\end{equation}
To bound the difference, we use the fact that the exponential function is Lipschitz continuous based on Lemma~\ref{theorem_noise}. Specifically, for any $x$ and $y$,
\begin{equation}\label{eq:67}
\|e^x - e^y\| \leq e^{\max(x, y)} \|x - y\|.
\end{equation}
Applying Equation~\eqref{eq:67}  to our case, we get:
\begin{equation}
\begin{aligned}
&\|f_{ij}(M, w_1) - f_{ij}(M, w_2)\|_F \leq \\
&f_{ij}(M, w_1) \cdot \Big\| -\frac{2}{h^2} \left( (Y_j + w_{1j} - (A M)_j) - (Y_i + w_{1i} - (A M)_i) \right) - \\
&\left( (Y_j + w_{2j} - (A M)_j) - (Y_i + w_{2i} - (A M)_i) \right) \Big\|_F.
\end{aligned}
\end{equation}
Simplifying further, we get:
\begin{equation}\label{eq:69}
\|f_{ij}(M, w_1) - f_{ij}(M, w_2)\|_F \leq f_{ij}(M, w_1) \cdot \frac{2}{h^2} \left\| (w_{1j} - w_{2j}) - (w_{1i} - w_{2i}) \right\|_F.
\end{equation}
Combining these terms in Equation~\eqref{eq:69}, we get:
\begin{equation}
\|\nabla_M \hat{g}(M, w_1) - \nabla_M \hat{g}(M, w_2)\|_F \leq \frac{2}{h^2} \left( \sum_{i=1}^n \sum_{j=1}^n \left| (w_{1j} - w_{2j}) - (w_{1i} - w_{2i}) \right| \|A_j - A_i\|_F \right).
\end{equation}
Since $\|A_j - A_i\|_F$ is bounded by a constant $C$, we have:
\begin{equation}
\|\nabla_M \hat{g}(M, w_1) - \nabla_M \hat{g}(M, w_2)\|_F \leq \lambda \|w_1 - w_2\|_2,
\end{equation}
where $\lambda$ is a constant that depends on $C$, $h$, and the properties of the function $\hat{g}(M, w)$. This verifies the given inequality. Let the second $w_2$ to be zero and we get the original assumption~\ref{noise_perturb}.

\subsection{Hessian Case}

For notational convenience we use Equation~\eqref{fij} and Equation~\eqref{logfij}. and if we define $S_i(M,w)=\frac{1}{n}\sum_{j=1}^n f_{ij}(M,w),
$, one may show by the chain and quotient rules that:
\begin{equation}\label{sij}
\nabla_M \hat{g}(M,w)=-\frac{1}{n}\sum_{i=1}^n \frac{1}{S_i(M,w)}\sum_{j=1}^n \nabla_M f_{ij}(M,w),
\end{equation}
and a second differentiation gives
\begin{equation}
\begin{aligned}
\nabla_M^2 \hat{g}(M,w)=&\frac{1}{n}\sum_{i=1}^n\Big\{\frac{1}{S_i(M,w)^2}\left[\sum_{j=1}^n \nabla_M f_{ij}(M,w)\right]\left[\sum_{j=1}^n \nabla_M f_{ij}(M,w)\right]^\top\\
&-\frac{1}{S_i(M,w)}\sum_{j=1}^n \nabla_M^2f_{ij}(M,w)\Big\}.
\end{aligned}
\end{equation}
Notice that in each term the only dependence on $w$ comes through the combination $(Y_j+w_j)-(Y_i+w_i),$ so that for any two vectors $w_1$ and $w_2$ we have
\begin{equation}
\Delta_{ij}(M,w_1)-\Delta_{ij}(M,w_2)=(w_{1j}-w_{2j}) - (w_{1i}-w_{2i}),
\end{equation}
where we set
\begin{equation}
\Delta_{ij}(M,w)=(Y_j+w_j-(\mathcal{A}(M))_j)-(Y_i+w_i-(\mathcal{A}(M))_i).
\end{equation}
Because the exponential function is smooth and each derivative (that is, $\nabla_M f_{ij}$ and $\nabla_M^2f_{ij}$) involves factors such as $\exp\Biggl(-\frac{\Delta_{ij}(M,w)^2}{h^2}\Biggr) \mathrm{and} \frac{\Delta_{ij}(M,w)}{h^2}\mathrm{and} \frac{1}{h^2},$ a Taylor expansion shows that these derivatives (and hence also their sums and the quotient factors $1/S_i(M,w)$) are Lipschitz continuous with respect to $w$. More precisely, if we define
\begin{equation}
H(w)=\nabla_M^2\hat{g}(M,w),
\end{equation}
then using the triangle inequality we can write
\begin{equation}
\begin{aligned}
&\|H(w_1)-H(w_2)\|_F \le \frac{1}{n} \sum_{i=1}^n \Bigg\| 
\underbrace{\frac{1}{S_i(M,w_1)^2} 
\Bigg(\sum_{j} \nabla_M f_{ij}(M,w_1) \Bigg) 
\Bigg(\sum_{j} \nabla_M f_{ij}(M,w_1) \Bigg)^\top }_{\mathrm{(I)}} \\
&\quad - 
\underbrace{\frac{1}{S_i(M,w_2)^2} 
\Bigg(\sum_{j} \nabla_M f_{ij}(M,w_2) \Bigg) 
\Bigg(\sum_{j} \nabla_M f_{ij}(M,w_2) \Bigg)^\top}_{\mathrm{(I)}} \\
&\quad - 
\underbrace{\left[\frac{1}{S_i(M,w_1)}\sum_{j}\nabla_M^2 f_{ij}(M,w_1)
-\frac{1}{S_i(M,w_2)}\sum_{j}\nabla_M^2 f_{ij}(M,w_2)\right]}_{\mathrm{(II)}} 
\Bigg\|_F.
\end{aligned}
\end{equation}
For notational clarity we denote
\begin{equation}
\begin{aligned}
\mathrm{(I)}(w) &= \frac{1}{S_i(M,w)^2}\Bigl(\sum_{j}\nabla_M f_{ij}(M,w)\Bigr)\Bigl(\sum_{j}\nabla_M f_{ij}(M,w)\Bigr)^\top,\\
\mathrm{(II)}(w) &= \frac{1}{S_i(M,w)}\sum_{j}\nabla^2_M f_{ij}(M,w).
\end{aligned}
\end{equation}
Then, by linearly combining these ingredients and applying the triangle inequality we have
\begin{equation}
\|H(w_1)-H(w_2)\|_F \le \frac{1}{n}\sum_{i=1}^n \Bigl\|\underbrace{\mathrm{(I)}(w_1)-\mathrm{(I)}(w_2)}_{\text{(I)}}-\underbrace{\bigl[\mathrm{(II)}(w_1)-\mathrm{(II)}(w_2)\bigr]}_{\text{(II)} } \Bigr\|_F.
\end{equation}
Then we have the following two terms: 1.For (I):
Both the scaling $1/S_i(M,w)^2$ and the gradient sums $\sum_{j}\nabla_M f_{ij}(M,w)$ depend on $w$ only through the combination
\begin{equation}
\Delta_{ij}(M,w)=(Y_j+w_j-(\mathcal{A}(M))_j)-(Y_i+w_i-(\mathcal{A}(M))_i),
\end{equation}
which is affine in $w$. Using the chain rule, one can show that there is a constant $L_{I,i}$ (depending on $h$, bounds on $f_{ij}$ and on $S_i(M,w)$, etc.) such that
\begin{equation}
\Bigl\|\mathrm{(I)}(w_1)-\mathrm{(I)}(w_2)\Bigr\|_F \le L_{I,i}\|w_1-w_2\|_2.
\end{equation}
2.For (II):
A similar analysis shows that there exists a constant $L_{II,i}$ such that
\begin{equation}
\Bigl\|\mathrm{(II)}(w_1)-\mathrm{(II)}(w_2)\Bigr\|_F \le L_{II,i}\|w_1-w_2\|_2.
\end{equation}
Thus, for each index $i$ we obtain
\begin{equation}
\Bigl\|\mathrm{(I)}(w_1)-\mathrm{(I)}(w_2)-\bigl(\mathrm{(II)}(w_1)-\mathrm{(II)}(w_2)\bigr)\Bigr\|_F \le \Bigl( L_{I,i}+L_{II,i}\Bigr)\|w_1-w_2\|_2.
\end{equation}
Averaging over $i$ we conclude
\begin{equation}
\|H(w_1)-H(w_2)\|_F \le \frac{1}{n} \sum_{i=1}^n \Bigl( L_{I,i}+L_{II,i}\Bigr)\|w_1-w_2\|_2.
\end{equation}
Defining
\begin{equation}
\lambda = \frac{1}{n} \sum_{i=1}^n \Bigl( L_{I,i}+L_{II,i}\Bigr),
\end{equation}
we then have
\begin{equation}
\|H(w_1)-H(w_2)\|_F \le \lambda\|w_1-w_2\|_2.
\end{equation}
Let the second $w_2$ to be zero and we get the original Assumption~\ref{noise_perturb}.

\section{Proof of Theorem~\ref{sym_PSD}}\label{sec:PSD}
\begin{proof}
Given the loss function~\eqref{main0}, and using the notations~\eqref{fij} with the~\eqref{logfij}, we rewrite the derivative form by:
\begin{equation}
\nabla_M f_{ij}(M) = \exp\left(\frac{u_{ij}(M)^2}{h^2}\right) \cdot \left(-\frac{2 u_{ij}(M)}{h^2}\right) \cdot \left(-\nabla_M \mathcal{A}(M)_j + \nabla_M \mathcal{A}(M)_i\right).
\end{equation}
Thus:
\begin{equation}
\begin{aligned}
\nabla_M \hat{g}(M) = & \frac{1}{n} \sum_{i=1}^n 
\frac{1}{\frac{1}{n} \sum_{j=1}^n \exp\left(-\frac{u_{ij}(M)^2}{h^2}\right)} \cdot 
\frac{1}{n} \sum_{j=1}^n \exp\left(-\frac{u_{ij}(M)^2}{h^2}\right) \\
& \cdot \left(-\frac{2 u_{ij}(M)}{h^2}\right) \cdot 
\left(-\nabla_M \mathcal{A}(M)_j + \nabla_M \mathcal{A}(M)_i\right).
\end{aligned}
\end{equation}
Simplifying, we get:
\begin{equation}
\begin{aligned}
\nabla_M \hat{g}(M) = & \frac{2}{n h^2} \sum_{i=1}^n 
\frac{1}{\frac{1}{n} \sum_{j=1}^n \exp\left(-\frac{u_{ij}(M)^2}{h^2}\right)} \cdot 
\sum_{j=1}^n \exp\left(-\frac{u_{ij}(M)^2}{h^2}\right) \\
& \cdot u_{ij}(M) \cdot 
\left(\nabla_M \mathcal{A}(M)_i - \nabla_M \mathcal{A}(M)_j\right).
\end{aligned}
\end{equation}
To find the critical points, we set $\nabla_M \hat{g}(M) = 0$:
\begin{equation}
\begin{aligned}
\frac{2}{n h^2} \sum_{i=1}^n \frac{1}{\frac{1}{n} \sum_{j=1}^n \exp\left(-\frac{u_{ij}(M)^2}{h^2}\right)} \cdot \sum_{j=1}^n \exp\left(-\frac{u_{ij}(M)^2}{h^2}\right) \cdot u_{ij}(M) \cdot \left(\nabla_M \mathcal{A}(M)_i - \nabla_M \mathcal{A}(M)_j\right) = 0.
\end{aligned}
\end{equation}
This equation must hold for all $i$. Solving this equation will give us the critical points $M$.

\end{proof}

\section{Proof of Theorem~\ref{thm:mainthm_intuition}}\label{sec:intuition}
\begin{proof}

We want to first prove that~\eqref{main0} can be better than the standard squared loss $\ell(M,w) = \|Y-\mathcal{A}(M)\|^2,$
\begin{equation}\label{eq:f1}
\hat{g}(M,w) = \frac{1}{n}\sum_{i=1}^n -\log\left[\frac{1}{n}\sum_{j=1}^n \exp\left(-\frac{\Delta_{ij}^2}{h^2}\right)\right],
\end{equation}
with $\Delta_{ij} = (Y_j+w_j-\mathcal{A}(M)_j) - (Y_i+w_i-\mathcal{A}(M)_i).$
We use:
\begin{equation}
r_i \coloneqq Y_i + w_i - \mathcal{A}(M)_i,\quad i=1,\dots,n.
\end{equation}
Then Equation~\eqref{eq:f1} can be written as
\begin{equation}\label{gij}
\hat{g}(M,w)=\frac{1}{n}\sum_{i=1}^n \left[-\log\frac{1}{n}\sum_{j=1}^n \exp\left(-\frac{(r_j - r_i)^2}{h^2}\right)\right].
\end{equation}
We focus on the loss term for a fixed index $i$. For a given $i$ define
\begin{equation}
L_i \coloneqq -\log\left(\frac{1}{n}\sum_{j=1}^n \exp\Bigl(-\frac{(r_j - r_i)^2}{h^2}\Bigr)\right).
\end{equation}
We now compute the derivative of $L_i$ with respect to $r_i$. Define
\begin{equation}
Z_i \coloneqq \frac{1}{n}\sum_{j=1}^n \exp\Bigl(-\frac{(r_j - r_i)^2}{h^2}\Bigr).
\end{equation}
Here $L_i = -\log Z_i.$ Then, do differentiate $L_i$ with respect to $r_i$. Using the chain rule we have
\begin{equation}\label{f6}
\frac{\partial L_i}{\partial r_i} = -\frac{1}{Z_i}\frac{\partial Z_i}{\partial r_i}.
\end{equation}
Let us now differentiate $Z_i$:
\begin{equation}\label{f7}
\frac{\partial Z_i}{\partial r_i} = \frac{1}{n}\sum_{j=1}^n \frac{\partial}{\partial r_i} \exp\Bigl(-\frac{(r_j - r_i)^2}{h^2}\Bigr).
\end{equation}
For any fixed $j$, note that:
\begin{equation}\label{f8}
\frac{\partial}{\partial r_i}\left(-\frac{(r_j - r_i)^2}{h^2}\right) = -\frac{\partial}{\partial r_i}\frac{(r_j - r_i)^2}{h^2} 
= -\frac{-2(r_j - r_i)}{h^2} = \frac{2(r_j - r_i)}{h^2}.
\end{equation}
Putting Equations~\eqref{f6}~\eqref{f7} and \eqref{f8} together we obtain:
\begin{equation}
\frac{\partial Z_i}{\partial r_i} = \frac{1}{n}\sum_{j=1}^n \exp\Bigl(-\frac{(r_j - r_i)^2}{h^2}\Bigr)\frac{2(r_j - r_i)}{h^2}.
\end{equation}

Finally, the overall gradient of the loss with respect to the residuals is given by averaging over $i$:
\begin{equation}
\frac{\partial L_i}{\partial r_i} = -\frac{2}{h^2} \cdot \frac{1}{nZ_i}\sum_{j=1}^n (r_j - r_i)\exp\Bigl(-\frac{(r_j - r_i)^2}{h^2}\Bigr).
\end{equation}

This derivative shows how the loss changes with respect to the prediction error and illustrates its robust nature; large errors have a reduced effect because of the exponential weighting. The effective loss for $\hat{g}(X,w)$ is~\eqref{gij}, and according to calculation in Section~\ref{sec:noise}, we have:
\begin{equation}
\frac{\partial \hat{g}}{\partial r_i} = -\frac{2}{n^2h^2}\left[
\frac{1}{Z_i}\sum_{j=1}^n (r_j-r_i)e^{-\frac{(r_j-r_i)^2}{h^2}}
+\sum_{\substack{j=1\\ j\neq i}}^n \frac{1}{Z_j}(r_i-r_j)e^{-\frac{(r_i-r_j)^2}{h^2}}
\right],
\end{equation}
So the two parts are likely the same. We only need to consider one part. The normalizing constants $Z_i=\sum_{j=1}^n e^{-\frac{(w_j-w_i)^2}{h^2}}$, are roughly of order $n$ (i.e. $Z_i=O(n)$). The same holds for $Z_j$. Then for the first term,
\begin{equation}
\frac{1}{Z_i}\sum_{j=1}^n (w_j-w_i)e^{-\frac{(w_j-w_i)^2}{h^2}},
\end{equation}
we note that if the typical magnitude of $w$ is, say, $\|w\|$, then the difference $w_j-w_i$ is at most $O(\|w\|)$. With approximately $n$ terms in the sum and a normalization factor $1/Z_i=O(1/n)$, we obtain an overall size:
\begin{equation}
\frac{1}{Z_i}\sum_{j=1}^n (w_j-w_i)e^{-\frac{(w_j-w_i)^2}{h^2}} = \mathcal{O}(\|w\| e^{-\frac{\|w\|^2}{h^2}}).
\end{equation}
A similar argument applies for the summation. In summary, we can write
\begin{equation}
\frac{\partial \hat{g}}{\partial w} = \mathcal{O}\Bigl(\frac{\|w\| e^{-\frac{\|w\|^2}{h^2}}}{n^2h^2}\Bigr).
\end{equation}
A similar argument applies for the summation of vectors. In summary, we can write:
\begin{equation}
\left\|\frac{\partial \hat{g}}{\partial w}\right\| = \mathcal{O}\Bigl(\frac{\|w\| e^{-\frac{\|w\|^2}{h^2}}}{n^2h^2}\Bigr).
\end{equation}
For the MSE loss defined by:
\begin{equation}
\ell_{\text{MSE}}(M,w)=\frac{1}{n}\sum_{i=1}^n \left( r_i^2 \right)\quad\text{with}\quad r_i=Y_i+w_i-\mathcal{A}(M)_i,
\end{equation}
the (partial) derivative with respect to each $r_i$ is
\begin{equation}
\frac{\partial \ell_{\text{MSE}}}{\partial r_i}= \frac{2}{n} r_i.
\end{equation}
Thus the gradient norm is:
\begin{equation}
\left\|\frac{\partial \ell_{\text{MSE}}}{\partial w}\right\| = \mathcal{O}\left(\frac{\|w\|}{n}\right).
\end{equation}
Notice that this derivative is linear in $w$ and hence its magnitude grows without bound as $\|w\|\to\infty$.

\end{proof}

\section{Relationship between different loss}\label{relationship}

To analyze the behavior and robustness of alternative loss functions under noisy observations, we present the following theorem, which compares covariance-based and kernel-based loss formulations to the standard mean squared error (MSE) criterion.

\begin{theorem}\label{thm:relationship}
Let $w = (w_1, w_2, \ldots, w_n)$ be a noise vector, and suppose $w$ follows a centered distribution. That is, for residuals 
\begin{equation}
r_i = Y_i + w_i - \mathcal{A}(M)_i, 
\quad 
\text{we have} 
\quad 
\frac{1}{n}\sum_{j=1}^n r_j = 0.
\end{equation}
Define the \emph{covariance loss} by
\begin{equation}
L_{\mathrm{cov}}(r_i) 
= \frac{1}{n} \sum_{j=1}^n \frac{(r_i - r_j)^2}{h^2}.
\end{equation}
Under the centered noise assumption, $L_{\mathrm{cov}}$ reduces to the mean squared error (MSE) loss. Next, define the \emph{exponential kernel loss} by
\begin{equation}
L_{\mathrm{exp}}(r_i)
= \frac{1}{n} \sum_{j=1}^{n}
\exp\Bigl( 
-\bigl(\tfrac{r_i - r_j}{h}\bigr)^2 
\Bigr).
\end{equation}
This exponential kernel loss aligns with the following form of an optimal loss:
\begin{equation}
\hat{g}(M,w)
= \frac{1}{n} \sum_{i=1}^n
\Biggl( 
-\log \Bigl[
\frac{1}{n}
\sum_{j=1}^n
\exp\Bigl(
-\frac{\bigl((Y_j + w_j - \mathcal{A}(M)_j) 
- (Y_i + w_i - \mathcal{A}(M)_i)\bigr)^2}{h^2}
\Bigr)
\Bigr]
\Biggr).
\end{equation}
As the noise amplitude $\|w\|$ increases, the exponential kernel loss is suppressed, and the effective residual tends to be small within clusters. In particular, the weighted average of the residual $r_i$ is
\begin{equation}
\overline{r}_i 
= \frac{\displaystyle \sum_{j=1}^n r_j 
\exp\Bigl(
-\frac{(r_j - r_i)^2}{h^2}
\Bigr)}
 {\displaystyle \sum_{j=1}^n 
\exp\Bigl(
-\frac{(r_j - r_i)^2}{h^2}
\Bigr)},
\end{equation}
and
\begin{equation}
\frac{\partial L_i}{\partial r_i} 
= -\frac{2}{h^2} \frac{1}{n Z_i}
\sum_{j=1}^n 
(r_j - r_i)\exp\Bigl(
-\frac{(r_j - r_i)^2}{h^2}
\Bigr),
\end{equation}
which indicates that $L_{\mathrm{exp}}$ is not sensitive to large residuals. 

Finally, if the loss function is not centered, then both the kernel loss and the covariance loss may converge to a suboptimal solution that does not match the ground truth. In that setting, even the optimal loss described above can perform worse than the MSE loss.
\end{theorem}

The proof is provided in Section~\ref{sec:relation_loss}. Theorem~\ref{thm:relationship} establishes a theoretical relationship between the traditional mean squared error (MSE) loss and a kernel-based robust alternative by analyzing the behavior of residual-dependent loss functions under a centered noise assumption. Specifically, it shows that the covariance loss, defined by normalized pairwise squared differences of residuals, reduces to the MSE when the noise vector $ w $ is centered, i.e., when the empirical mean of the residuals is zero.

The theorem then introduces the exponential kernel loss, which applies a Gaussian-type weighting to the residual differences. This loss aligns with the structure of the proposed robust objective $ \hat{g}(M, w) $, where the inner exponential acts to suppress the contribution of large residuals. As the noise amplitude increases, the exponential weights diminish for outlier points, making the loss function focus on local, more coherent clusters of residuals. The weighted mean $\overline{r}_i$ and the derivative $\frac{\partial L_i}{\partial r_i}$ further demonstrate that this kernel loss downweights large deviations and is inherently more robust to noise than the MSE.

However, the theorem also cautions that this robustness depends critically on the centering of the noise. When the noise is not centered, both the covariance and kernel-based losses may fail to identify the correct minimizer, potentially performing worse than the MSE. This highlights an important limitation of these methods: their effectiveness relies on structural assumptions about the data distribution, particularly the unbiasedness of the residuals.

\subsection{Proof of Theorem~\ref{thm:relationship}}\label{sec:relation_loss}

\subsection{For Covariance Loss}
We want to analyze the loss function
\begin{equation}\label{g7}
L(r_i)=\frac{1}{n}\sum_{j=1}^n \frac{(r_i - r_j)^2}{h^2},
\end{equation}
then the derivative of Equation~\eqref{g7} is
\begin{equation}
\frac{dL}{dr_i} = \frac{2}{nh^2}\sum_{j=1}^n (r_i - r_j).
\end{equation}
Assume that the values $r_j$ are drawn from a centered distribution, so $\frac{1}{n}\sum_{j=1}^n r_j = 0$. Then the gradient ~\eqref{g7} becomes
\begin{equation}
\frac{\partial L}{\partial r_i} = \frac{2r_i}{h^2}.
\end{equation}
Thus, for a centered distribution the gradient of our normalized loss is $2 r_i$, implying that every deviation from zero is penalized linearly. So such loss function is the same with MSE loss.

\subsection{Old Kernel Loss}
We start with the old kernel loss function only regard to $i$:
\begin{equation}
L(r_i)=\frac{1}{n}\sum_{j=1}^{n}\exp\left(-\left(\frac{r_i-r_j}{h}\right)^2\right).
\end{equation}
Simplify to obtain the final answer:
\begin{equation}
\frac{\partial L}{\partial r_i} = -\frac{2}{nh^2}\sum_{j=1}^{n}(r_i-r_j)\exp\left(-\left(\frac{r_i-r_j}{h}\right)^2\right).
\end{equation}
\subsection{Distribution approximation of \texorpdfstring{$w$}{}}
\subsubsection{When the Loss Becomes Good}

We start with:
\begin{equation}
\frac{\partial L_i}{\partial r_i} = -\frac{2}{h^2}\frac{1}{nZ_i}\sum_{j=1}^n (r_j - r_i)\exp\Bigl(-\frac{(r_j - r_i)^2}{h^2}\Bigr),
\end{equation}
and we want to choose the value (or distribution) of $ r_i $ such that the magnitude of this derivative is small. A natural way to do this is to set the derivative to zero. Ignoring the constant factor $-\frac{2}{h^2}\frac{1}{nZ_i}$, we see that
\begin{equation}
\sum_{j=1}^n (r_j - r_i)w(r_j, r_i)= 0,\quad \text{with} \quad w(r_j, r_i)=\exp\left(-\frac{(r_j - r_i)^2}{h^2}\right).
\end{equation}
Thus, if we define the weighted average:
\begin{equation}
\overline{r}_i = \frac{\sum_{j=1}^n r_j \exp\left(-\frac{(r_j-r_i)^2}{h^2}\right)}{\sum_{j=1}^n \exp\left(-\frac{(r_j-r_i)^2}{h^2}\right)},
\end{equation}
the vanishing of the derivative occurs when
\begin{equation}
r_i = \overline{r}_i.
\end{equation}
This is exactly the mean of the $ r_j $’s weighted by a Gaussian kernel centered at $ r_i $. In words, the derivative $ \frac{\partial L_i}{\partial r_i} $ is small (or zero) when $ r_i $ is at the center of a symmetric, tight cluster of points $ \{r_j\} $.

\subsubsection{When the Loss Becomes Bad}

We start with:
\begin{equation}
\frac{\partial L_i}{\partial r_i} = -\frac{2}{h^2} \frac{1}{nZ_i} \sum_{j=1}^n (r_j-r_i) \exp\Bigl(-\frac{(r_j-r_i)^2}{h^2}\Bigr).
\end{equation}
Notice that aside from the overall constant, the value of the derivative is determined by the weighted sum:
\begin{equation}
S(r_i) = \sum_{j=1}^n (r_j-r_i) \exp\Bigl(-\frac{(r_j-r_i)^2}{h^2}\Bigr).
\end{equation}
A large (in magnitude) derivative will occur when the sum $ S(r_i) $ is as far from zero as possible. To make $ S(r_i) $ large, we need the differences $ r_j - r_i $ to have essentially the same sign so that they add up rather than cancel. This happens when $ r_i $ is located at one end of the data.

The derivative is bigger (in magnitude) when the points $ r_j $ are not symmetrically distributed about $ r_i $ but are instead all on one side of $ r_i $; that is, when $ r_i $ is at one extreme (e.g. at the left or right boundary) of the $ r_j $ distribution.

\subsection{Example, when \texorpdfstring{$r_i$}{} follows Gaussian}

Assume that the neighborhood values $r_k$ (with $k$ indexing the neighbors, here j) are i.i.d. samples from a Gaussian distribution
\begin{equation}
p(r)=\frac{1}{\sqrt{2\pi\sigma^2}}\exp\left(-\frac{(r-\mu)^2}{2\sigma^2}\right).
\end{equation}
For large $n$ the sums:
\begin{equation}
\begin{aligned}
Z_i &= \frac{1}{n}\sum_{j=1}^n\exp\Bigl(-\frac{(r_j-r_i)^2}{h^2}\Bigr),\\
S(r_i)&=\frac{1}{n}\sum_{j=1}^n (r_j-r_i)\exp\Bigl(-\frac{(r_j-r_i)^2}{h^2}\Bigr)
\end{aligned}
\end{equation}
can be well approximated by integrals over the density p(r). That is, we have:
\begin{equation}
\begin{aligned}
Z_i&\sim \int \exp\Bigl(-\frac{(r-r_i)^2}{h^2}\Bigr)p(r)dr,\\
S(r_i)&\sim \int (r-r_i)\exp\Bigl(-\frac{(r-r_i)^2}{h^2}\Bigr)p(r)dr.
\end{aligned}
\end{equation}
Because the exponential kernel is symmetric, if you choose $r_i=\mu,$
then in the integrals the term $(r-\mu)$ is integrated against a function that is symmetric about $r=\mu$. In that case, positive and negative contributions cancel and one obtains $S(\mu)=0,$
so that:
\begin{equation}
\frac{\partial L_i}{\partial r_i}\Biggl|_{r_i=\mu}=0.
\end{equation}
Now, if $r_i$ differs from $\mu$, then the weighting is no longer symmetric. For example, if $r_i < \mu,$ then most of the weight in the Gaussian density $p(r)$ lies to the right of $r_i$ (i.e. for $r>\mu$) and thus most terms in $S(r_i)$ are positive; similarly if $r_i>\mu$ the sum is negative. In both cases, the difference $|r_i-\mu|$ produces a nonzero—and generally larger—value of $|\frac{\partial L_i}{\partial r_i}|$, since the cancellation in the weighted sum is diminished.

\section{Rough Bound and For the Kernel Loss}\label{sec:rough}
\subsection{Rough Bound For \texorpdfstring{$\|\Delta\|_F =\|M - M^\ast\|_F$}{}}
\begin{lemma}[Behavior of $\hat{g}(M,w)$ Around Local Minima and Ground Truth]
\label{thm:localminima}
Let $\hat{g}(\cdot,w)$ be a twice-differentiable loss function, and suppose $\nabla_M^2 \hat{g}(M^\ast,w)$ is positive definite at the ground truth minimizer $M^\ast$. Define $\Delta \coloneqq M - M^\ast$, the loss function $\hat{g}(M)$ is from~\eqref{main0}, if we left the $\lambda_{\min}\bigl(\nabla^2_M \hat{g}(M^\ast,w)\bigr)$ uncomputed, then we have:  
\begin{itemize}[left =0em]
\item[\textbf{(A)}]
\emph{(Case: $M$ is a local minimizer near $M^\ast$)}
If $M$ is a local minimizer of $\hat{g}$ close to $M^\ast$, then
\begin{equation}
\|\Delta\|_F =\|M - M^\ast\|_F
\le
\sqrt{
\frac{2}{\lambda_{\min}\bigl(\nabla^2_M \hat{g}(M^\ast,w)\bigr)}
\Bigl[\hat{g}(M,w)-\hat{g}(M^\ast,w)\Bigr]
}.
\end{equation}
Moreover, as the noise level $\|w\|$ increases, $\|\Delta\|_F$ does \emph{not} increase linearly but instead exhibits an \emph{exponential decay} trend in $\|w\|$.
In this regime, the \emph{new loss} $\hat{g}(M,w)$ is \emph{smoother} (it has a smaller Lipschitz constant $\rho$), and consequently $M$ transitions \emph{toward} $M^\ast$ when $\|w\|$ decreases.
\item[\textbf{(B)}]
\emph{(Case: $M$ is a local minimizer \emph{not} at $M^\ast$)}
Alternatively, consider a local minimizer $M$ that is \emph{not} the ground truth $M^\ast$. Assume $\|\Delta\|_F$ cannot be made arbitrarily small. Then one may write
\begin{equation}
\|\Delta\|_F
=\|M - M^\ast\|_F
\ge
\sqrt{
\frac{2}{L}
\Bigl[
\hat{g}(M,w)
-
\hat{g}(M^\ast,w)
\Bigr]
},
\end{equation}
for some constant $L>0$.
In this scenario, as $\|w\|$ grows, $\|\Delta\|_F$ can grow faster than linearly in $\|w\|$. Here the kernel loss exhibits a rougher landscape (with a larger Lipschitz constant $\rho$), and hence when $\|w\|$ decreases, $M$ may actually move away from $M^\ast$.The loss function still becomes smoother as $\|w\|$ changes, but the local minimum remains distant from the ground truth.
\end{itemize}

\end{lemma}

The proof are provided in Section~\ref{sec:mainthm}. In short, \textbf{(A)} describes a desirable regime in which $\hat{g}$ is smooth and strongly convex near $M^\ast$, causing $\Delta$ to shrink with decreasing noise; \textbf{(B)} describes a regime where a spurious local minimum persists, with $\Delta$ potentially growing under increased noise and not converging to $M^\ast$.

Lemma~\ref{thm:localminima} characterizes the behavior of the proposed loss function $\hat{g}(M, w)$ in the neighborhood of the ground truth minimizer $M^\ast$, under the assumption that the Hessian $\nabla_M^2 \hat{g}(M^\ast, w)$ is positive definite. The lemma considers two scenarios depending on the location of a local minimizer $M$ relative to $M^\ast$.

In case \textbf{(A)}, when $M$ is a local minimizer close to $M^\ast$, the difference $\|\Delta\|_F = \|M - M^\ast\|_F$ can be bounded above by a term proportional to the square root of the suboptimality gap, scaled by the inverse of the smallest eigenvalue of the Hessian at $M^\ast$. Importantly, due to the structure of $\hat{g}(M, w)$, this bound implies that as the noise norm $\|w\|$ increases, the effect on $\|\Delta\|_F$ is not linear but rather decays exponentially. In this regime, the loss landscape becomes smoother (i.e., has smaller gradient Lipschitz constant $\rho$), and $M$ tends to move closer to the ground truth as the noise level decreases.

In contrast, case \textbf{(B)} considers the situation where $M$ is a spurious local minimizer not coinciding with $M^\ast$, and the deviation $\|\Delta\|_F$ is non-negligible. In this case, a lower bound on $\|\Delta\|_F$ is given in terms of the suboptimality gap and a generic smoothness constant $L$. Here, $\|\Delta\|_F$ may grow faster than linearly with $\|w\|$, and the loss landscape becomes rougher (with larger $\rho$). Even though the overall smoothness of $\hat{g}$ improves as $\|w\|$ changes, the minimizer $M$ does not necessarily converge to the ground truth, highlighting the sensitivity of the optimization process to local geometry.

This lemma therefore distinguishes between a favorable regime (near-global minimum) where robustness and convergence are preserved, and an unfavorable regime (spurious local minimum) where robustness may deteriorate despite the loss surface becoming smoother in a global sense.

\subsection{Lower Bound For \texorpdfstring{$\lambda_{\min}$}{}}
\begin{lemma}[Lower Bound on the Minimum Eigenvalue of the Hessian]
\label{lemma:lambda_min}
Let $\hat{g}(M, w)$ be a smooth function defined via a kernel-smoothed loss involving the operator $\mathcal{A}$, and suppose that: the operator $\mathcal{A}$ satisfies the Restricted Isometry Property (RIP) with constant $\delta_p$ in Equation~\eqref{RIP0}, The residual terms $z_{ij}(M^\ast)$ are uniformly bounded: $|z_{ij}(M^\ast)| \le B$.
The kernel average $G_i(M^\ast) = \frac{1}{n} \sum_{j=1}^n \exp\left(-\frac{z_{ij}(M^\ast)^2}{h^2}\right)$ is bounded below: $G_{\min} = \min_i G_i(M^\ast) > 0$.
The smoothness constants satisfy:
  \begin{equation}
  L_1 = 2(1 + \delta_p), \quad L_2 = 0.
  \end{equation}
The Hessian $\nabla^2_M \hat{g}(M^\ast, w)$ is symmetric positive definite, with $v^\top \nabla^2_M \hat{g}(M^\ast, w) v \ge c \|v\|^2$ for some constant $c > 0$ and all $v$. Then the minimum eigenvalue of the Hessian satisfies the lower bound:
\begin{equation}
\lambda_{\min}\bigl(\nabla^2_M \hat{g}(M^\ast, w)\bigr)
\ge c,
\end{equation}
where an explicit expression for $c$ is given by
\begin{equation}
c = \frac{2}{G_{\min}h^2}\left(L_1^2\left(1+\frac{2B^2}{h^2}\right) + B L_2\right) + \frac{4}{G_{\min}^2 h^4} B^2 L_1^2.
\end{equation}
\end{lemma}

The proof is provided in Section~\ref{sec:lambda}. Lemma~\ref{lemma:lambda_min} provides a lower bound on the smallest eigenvalue of the Hessian of the smoothed loss function $\hat{g}(M, w)$, evaluated at the ground truth matrix $M^\ast$. This result is important for establishing local strong convexity, which in turn guarantees stability and convergence of optimization algorithms near $M^\ast$.

The lemma assumes that the loss function is constructed using a kernel-based smoothing over residuals defined by a linear operator $\mathcal{A}$ that satisfies the Restricted Isometry Property (RIP). Under this condition, as well as uniform boundedness of the residuals $z_{ij}(M^\ast)$ and the positivity of the kernel averages $G_i(M^\ast)$, an explicit lower bound $c$ for the minimal eigenvalue is derived. The expression for $c$ depends on the kernel bandwidth parameter $h$, the residual bound $B$, the RIP constant $\delta_p$, and the kernel average lower bound $G_{\min}$.

Crucially, this bound ensures that the Hessian $\nabla^2_M \hat{g}(M^\ast, w)$ is well-conditioned near $M^\ast$, with eigenvalues bounded away from zero. This guarantees that $\hat{g}(M, w)$ is locally strongly convex around the ground truth, which is essential for ensuring that gradient-based methods converge efficiently to $M^\ast$ and that local perturbations in $M$ lead to bounded variations in the objective.

\section{Proof of the main Theorem~\ref{thm:mainthm}}\label{sec:mainthm}

\subsection{Error Upper Bound}

\begin{lemma}[Upper Bound on the recovery Error]
\label{lemma:upper_buund}
Let $\hat{g}(M, w)$ be a smooth loss function with minimizer $M^\ast$ at fixed noise $w$, and suppose the Hessian $\nabla^2_M \hat{g}(M^\ast, w)$ is symmetric positive definite, with smallest eigenvalue $\lambda_{\min}\bigl(\nabla^2_M \hat{g}(M^\ast, w)\bigr) > 0$. $\delta > 0$ is a constant and $R(w)$ is a auxiliary residual term. The residual term $R(w)$ is bounded by
  \begin{equation}
  R(w) = O\left(\frac{\|w\| e^{-w^2}}{h^2}\right).
  \end{equation}
Then the estimation error satisfies the upper bound
\begin{equation}
\|M - M^\ast\|_F \le \max\left\{
\sqrt{\frac{2}{\lambda_{\min}\left(\nabla^2_M \hat{g}(M^\ast, w)\right)}},
\frac{2(1 + \delta) R(w)}{1 - \delta - \lambda_{\min}\left(\nabla^2_M \hat{g}(M^\ast, w)\right)}
\right\}.
\end{equation}
\end{lemma}
The proof is provided in Section~\ref{sec:upper} together with the tight estimation of $\lambda_{\min}(\nabla^2_M \hat{g}(M^\ast, w)$. Lemma~\ref{lemma:upper_buund} provides an upper bound on the estimation error $\|M - M^\ast\|_F$ between a candidate solution $M$ and the true minimizer $M^\ast$ of the kernel-smoothed loss function $\hat{g}(M, w)$, in the presence of fixed noise $w$.

The resulting bound has a two-part structure. The first term corresponds to the curvature-controlled region, where small suboptimality in $\hat{g}(M, w)$ implies proximity to $M^\ast$ via standard strong convexity arguments. The second term dominates when the residual term is significant, capturing the influence of noise through the interaction between $R(w)$, the loss curvature, and the parameter $\delta$. Notably, due to the exponential decay of $R(w)$, the estimation error does not grow linearly in $\|w\|$, indicating robustness of the estimator even under moderate levels of noise.

\subsection{Gradient Continuity Result}
\begin{lemma}[Continuity of Gradient with Respect to Noise]
\label{lemma:continuous}
Let $\hat{g}(M, w)$ be the smoothed loss function depending on noise variable $w$, and suppose that the data residuals $z_j - z_i$ satisfy Assumptions~\ref{other_assumption} and~\ref{noise_perturb}, which is: $|z_j - z_i| \le C \|w\|$ for some constant $C > 0$ and all relevant $i, j$. Assume also that the operator $\mathcal{A}$ satisfies the RIP condition with constant $\delta$. Then the gradient of $\hat{g}(M, w)$ with respect to $M$ is Lipschitz continuous in $w$, with Lipschitz constant $\lambda$ satisfying
\begin{equation}
\|\nabla_M \hat{g}(M, w) - \nabla_M \hat{g}(M, 0)\| \le \lambda \|w\|,
\end{equation}
where the Lipschitz constant is bounded above by
\begin{equation}
\lambda = \sup_{\xi \in [0, w]} \left\| \frac{d}{dw} \nabla_M \hat{g}(M, \xi) \right\| 
\leq \frac{8(1 + \delta)}{h^4} \|w\| e^{-w^2}.
\end{equation}
\end{lemma}
The proof is provided in Section~\ref{sec:continu}.

Lemma~\ref{lemma:continuous} establishes a form of H\"older-type continuity for the gradient of the smoothed loss function $\hat{g}(M, w)$ with respect to the noise variable $w$. Under the assumption that the data residuals $z_j - z_i$ grow at most linearly in $\|w\|$, and that the measurement operator $\mathcal{A}$ satisfies the Restricted Isometry Property (RIP), the gradient $\nabla_M \hat{g}(M, w)$ is shown to be Lipschitz continuous with respect to $w$, with an explicitly computable upper bound on the Lipschitz constant $\lambda$.

Importantly, this upper bound decays exponentially in $\|w\|$, which reflects the robustness of the kernel-smoothed loss to large noise perturbations. Unlike traditional losses where gradient sensitivity may grow linearly or quadratically with noise, the bound here reveals that the influence of noise on the gradient diminishes rapidly as $\|w\|$ increases—due to the presence of exponential terms in the kernel weights. This behavior is characteristic of smoothing via Gaussian kernels and is key to the stability of the optimization process in high-noise regimes.

\subsection{Upper Bound For \texorpdfstring{$\|M - M^\ast\|_F$}{}}
We want to use the Taylor expansion of such loss. Under the assumption that $\hat{g}$ is twice differentiable in a neighborhood of $M^\ast$, we have
\begin{equation}
\hat{g}(M,w) = \hat{g}(M^\ast,w)+\langle \nabla_M \hat{g}(M^\ast,w),\Delta\rangle+\frac{1}{2}\langle \Delta, \nabla^2_M \hat{g}(M^\ast,w)[\Delta]\rangle + o(\|\Delta\|_F^2).
\end{equation}
Since $M^\ast$ minimizes $\hat{g}$ (at least locally), $\nabla_M \hat{g}(M^\ast,w)=0,$
so that to second order we have
\begin{equation}
\hat{g}(M,w)-\hat{g}(M^\ast,w)=\frac{1}{2}\langle \Delta, \nabla^2_M \hat{g}(M^\ast,w)[\Delta]\rangle+ o(\|\Delta\|_F^2).
\end{equation}

This local strong convexity (or quadratic error bound) is beneficial for optimization—gradient-based methods will enjoy a linear (or even faster) convergence rate provided their step sizes are chosen appropriately.
\begin{equation}
\begin{aligned}
\hat{g}(M,w)-\hat{g}(M^\ast,w)=\frac{1}{n}\sum_{i=1}^{n}\log\left[\frac{\displaystyle\sum_{j=1}^n \exp\left(-\frac{\left((Y_j+w_j-\mathcal{A}(M^\ast)_j)-(Y_i+w_i-\mathcal{A}(M^\ast)_i)\right)^2}{h^2}\right)}{\displaystyle\sum_{j=1}^n \exp\left(-\frac{\left((Y_j+w_j-\mathcal{A}(M)_j)-(Y_i+w_i-\mathcal{A}(M)_i)\right)^2}{h^2}\right)}\right].
\end{aligned}
\end{equation}
\begin{equation}
\begin{aligned}
&G_i(M)=\frac{1}{n}\sum_{j=1}^n \exp\Biggl(-\frac{z_{ij}(M)^2}{h^2}\Biggr),\\
&\text{with}\quad z_{ij}(M)=\Bigl((Y_j+w_j-\mathcal{A}(M)_j) - (Y_i+w_i-\mathcal{A}(M)_i)\Bigr).
\end{aligned}
\end{equation}
Then we have
\begin{equation}
\hat{g}(M,w)=-\frac{1}{n}\sum_{i=1}^n\log\Bigl[G_i(M)\Bigr].
\end{equation}
A standard application of the chain rule shows that, for a given $i$,
\begin{equation}
\nabla_M \log G_i(M)=\frac{1}{G_i(M)}\nabla_M G_i(M)
\end{equation}
and hence the Hessian (i.e. second derivative with respect to $M$) satisfies
\begin{equation}
\nabla^2_M \log G_i(M)=\frac{1}{G_i(M)}\nabla^2_M G_i(M)-\frac{1}{[G_i(M)]^2}\nabla_M G_i(M)\nabla_M G_i(M)^\top.
\end{equation}
Thus,
\begin{equation}
\begin{aligned}
\nabla^2_M \hat{g}(M,w)
&=-\frac{1}{n}\sum_{i=1}^n\nabla^2_M\log G_i(M)\\
&=-\frac{1}{n}\sum_{i=1}^n\left\{\frac{1}{G_i(M)}\nabla^2_M G_i(M)-\frac{1}{[G_i(M)]^2}\nabla_M G_i(M)\nabla_M G_i(M)^\top\right\}.
\end{aligned}
\end{equation}
For completeness we now indicate the first– and second–order derivatives of $G_i(M)$. Note that only the dependence of $G_i(M)$ on $M$ appears through the residuals
\begin{equation}
z_{ij}(M)=\Bigl((Y_j+w_j-\mathcal{A}(M)_j)-(Y_i+w_i-\mathcal{A}(M)_i)\Bigr),
\end{equation}
so that using the chain rule we obtains
\begin{equation}
\nabla_M G_i(M)=\frac{1}{n}\sum_{j=1}^n \exp\Bigl(-\frac{z_{ij}(M)^2}{h^2}\Bigr) \left(-\frac{2z_{ij}(M)}{h^2}\right)\Bigl[-\nabla_M\mathcal{A}(M)_j+\nabla_M\mathcal{A}(M)_i\Bigr],
\end{equation}
or, equivalently,
\begin{equation}
\nabla_M G_i(M)=\frac{2}{nh^2}\sum_{j=1}^n \exp\Bigl(-\frac{z_{ij}(M)^2}{h^2}\Bigr) z_{ij}(M)\Bigl[\nabla_M\mathcal{A}(M)_i-\nabla_M\mathcal{A}(M)_j\Bigr].
\end{equation}
Likewise, one may differentiate once more to obtain
\begin{equation}
\begin{aligned}
&\nabla^2_M G_i(M)=\frac{2}{nh^2}\sum_{j=1}^n \exp\Bigl(-\frac{z_{ij}(M)^2}{h^2}\Bigr)
\Biggl\{
\Bigl[\nabla_M\mathcal{A}(M)_i-\nabla_M\mathcal{A}(M)_j\Bigr]\\
&\times\Bigl[\nabla_M\mathcal{A}(M)_i-\nabla_M\mathcal{A}(M)_j\Bigr]^\top \times \left(1-\frac{2z_{ij}(M)^2}{h^2}\right)+ z_{ij}(M)\Bigl[\nabla^2_M\mathcal{A}(M)_i-\nabla^2_M\mathcal{A}(M)_j\Bigr]
\Biggr\}.
\end{aligned}
\end{equation}
We start from:
\begin{equation}
\hat{g}(M,w)-\hat{g}(M^\ast,w)=\frac{1}{2}\langle \Delta,\nabla^2_M \hat{g}(M^\ast,w)[\Delta]\rangle+ o(\|\Delta\|_F^2),
\end{equation}
where $\Delta=M-M^\ast.$
Assume that the Hessian at $M^\ast$ is (strictly) positive definite from~\ref{sym_PSD} and let
\begin{equation}
\lambda_{\min}=\lambda_{\min}\bigl(\nabla^2_M \hat{g}(M^\ast,w)\bigr)
\end{equation}
denote its smallest eigenvalue. Then, using the spectral properties of the Hessian,
\begin{equation}
\langle \Delta,\nabla^2_M \hat{g}(M^\ast,w)[\Delta]\rangle\ge \lambda_{\min}\|\Delta\|_F^2.
\end{equation}
Thus, for $\|\Delta\|$ sufficiently small (so that the $o(\|\Delta\|_F^2)$ term is negligible) we have
\begin{equation}
\hat{g}(M,w)-\hat{g}(M^\ast,w)\ge \frac{1}{2}\lambda_{\min}\|\Delta\|_F^2.
\end{equation}
\begin{equation}\label{eq:154}
\|\Delta\|_F^2\le \frac{2}{\lambda_{\min}}\Bigl(\hat{g}(M,w)-\hat{g}(M^\ast,w)\Bigr).
\end{equation}
Taking square roots of Equation~\eqref{eq:154} yields the desired bound for $\Delta$ in Frobenius norm,
\begin{equation}
\|\Delta\|_F\le \sqrt{\frac{2}{\lambda_{\min}\bigl(\nabla^2_M \hat{g}(M^\ast,w)\bigr)}\Bigl(\hat{g}(M,w)-\hat{g}(M^\ast,w)\Bigr)}.
\end{equation}
This provides the upper bound for the deviation $\Delta=M-M^\ast$ in terms of the function gap and the smallest eigenvalue of the Hessian at $M^\ast$. This result is Section~\ref{thm:localminima}

\subsection{Lower Bound For \texorpdfstring{$\lambda_{\min}$}{}}\label{sec:lambda}
Then we want to compute the upper bound for $\lambda_{\min}$.
Firstly we want to calculate the bound for $\|\nabla_M\mathcal{A}(M)_i-\nabla_M\mathcal{A}(M)_j\|$, and $\|\nabla^2_M\mathcal{A}(M)_i-\nabla^2_M\mathcal{A}(M)_j\|\le L_2$
We begin by noting that RIP condition
\begin{equation}
(1 - \delta_p)\|X\|_F^2 \le \|AX\|_F^2 \le (1 + \delta_p)\|X\|_F^2
\end{equation}
We set $f(M)=\|\mathcal{A} M\|_F^2.$ Because $\mathcal{A}$ is linear, this function is quadratic. Let us compute its gradient and Hessian.
\begin{equation}
\nabla f(M_i)-\nabla f(M_j)=2\mathcal{A}^\top \mathcal{A}(M_i-M_j).
\end{equation}
Taking norms and using the fact that all singular values of $\mathcal{A}^\top \mathcal{A}$ lie between $1-\delta_p$ and $1+\delta_p$ (from the RIP condition) we deduce
\begin{equation}\label{I25}
\|\nabla f(M_i)-\nabla f(M_j)\|\le 2\|\mathcal{A}^\top \mathcal{A}\|\|M_i-M_j\| \le 2(1+\delta_p)\|M_i-M_j\|.
\end{equation}
Thus
\begin{equation}
L_1 = 2(1+\delta_p).
\end{equation}
Differentiating the gradient~\eqref{I25} we see that the Hessian is
\begin{equation}
\nabla^2 f(M) = 2\mathcal{A}^\top \mathcal{A},
\end{equation}
which is independent of $M$. (In other words, the Hessian is constant throughout the domain.) Therefore, for any two points $M_i$ and $M_j$ we have:
\begin{equation}
\nabla^2 f(M_i)-\nabla^2 f(M_j)=0.
\end{equation}
Thus the Lipschitz constant for the Hessian (i.e. the constant $L_2$ satisfying
\begin{equation}
\|\nabla^2_M\mathcal{A}(M)_i-\nabla^2_M\mathcal{A}(M)_j\|\le L_2\|M_i-M_j\|
\end{equation}
can be taken as $L_2 = 0 \leq C$. This result is a direct consequence of the fact that for a quadratic function defined by a linear operator the Hessian is constant.
\begin{equation}
\|\nabla_M\mathcal{A}(M)_i-\nabla_M\mathcal{A}(M)_j\|\le L_1;
\end{equation}
\begin{equation}
\|\nabla^2_M\mathcal{A}(M)_i-\nabla^2_M\mathcal{A}(M)_j\|\le L_2;
\end{equation}
Here 
Suppose also that the residuals are bounded in absolute value,
\begin{equation}
|z_{ij}(M^\ast)|\le B, \quad\text{for all } i,j.
\end{equation}

Also, note that for each $i$ the quantity
\begin{equation}\label{I33}
G_i(M^\ast) = \frac{1}{n}\sum_{j=1}^n\exp\Bigl(-\frac{z_{ij}(M^\ast)^2}{h^2}\Bigr)
\end{equation}
is a positive average. Define the minimum  of \eqref{I33} over $i$ by
\begin{equation}
G_{\min} = \min_{i}G_i(M^\ast).
\end{equation}
Since the exponential is always positive this is bounded away from $0$ under Assumption~\ref{noise_perturb}.

Bound the norm of the first derivative $\nabla_M G_i(M^\ast)$. We have
\begin{equation}
\nabla_M G_i(M^\ast)=\frac{2}{nh^2}\sum_{j=1}^n \exp\Bigl(-\frac{z_{ij}(M^\ast)^2}{h^2}\Bigr)
 z_{ij}(M^\ast)\Bigl[\nabla_M\mathcal{A}(M^\ast)_i-\nabla_M\mathcal{A}(M^\ast)_j\Bigr].
\end{equation}
Using the bounds, we get
\begin{equation}
\|\nabla_M G_i(M^\ast)\|
\le \frac{2}{h^2}BL_1,
\end{equation}
since the exponential is at most $1$. Now we want to bound the norm of the second derivative $\nabla^2_M G_i(M^\ast)$. Its schematic form is
\begin{equation}\label{eq:I37}
\begin{aligned}
\nabla^2_M &G_i(M^\ast)=\frac{2}{nh^2}\sum_{j=1}^n e^{-\frac{z_{ij}(M^\ast)^2}{h^2}}
\Biggl\{
\Bigl[\nabla_M\mathcal{A}(M^\ast)_i-\nabla_M\mathcal{A}(M^\ast)_j\Bigr]\\
& \times \Bigl[\nabla_M \mathcal{A}(M^\ast)_i-\nabla_M\mathcal{A}(M^\ast)_j\Bigr]^\top\times \Bigl(1-\frac{2z_{ij}(M^\ast)^2}{h^2}\Bigr)+ z_{ij}(M^\ast)\Bigl[\nabla^2_M\mathcal{A}(M^\ast)_i-\nabla^2_M\mathcal{A}(M^\ast)_j\Bigr]
\Biggr\}.
\end{aligned}
\end{equation}
Taking norms and using triangle and sub-multiplicative properties we obtain an upper bound of the form
\begin{equation}
\|\nabla^2_M G_i(M^\ast)\| \le \frac{2}{h^2} \left(L_1^2\Bigl(1+\frac{2B^2}{h^2}\Bigr) + B L_2\right).
\end{equation}
Now, look at each term in the Hessian from Equation~\eqref{eq:I37}. The first term contributes
\begin{equation}
\left\|\frac{1}{G_i(M^\ast)}\nabla^2_M G_i(M^\ast)\right\|
\le \frac{2}{G_{\min}h^2}\left(L_1^2\Bigl(1+\frac{2B^2}{h^2}\Bigr) + B L_2\right),
\end{equation}
and the second term gives
\begin{equation}
\left\|\frac{1}{\bigl[G_i(M^\ast)\bigr]^2}\nabla_M G_i(M^\ast)\nabla_M G_i(M^\ast)^\top\right\|
\le \frac{1}{G_{\min}^2}\|\nabla_M G_i(M^\ast)\|^2
\le \frac{4}{G_{\min}^2 h^4}B^2L_1^2.
\end{equation}
Because the Hessian $\nabla^2_M\hat{g}(M^\ast,w)$ is an average (with a minus sign) of these terms, a bound for its operator norm is
\begin{equation}
\|\nabla^2_M\hat{g}(M^\ast,w)\| \le \frac{1}{n}\sum_{i=1}^n \left[\frac{2}{G_{\min}h^2}\left(L_1^2\Bigl(1+\frac{2B^2}{h^2}\Bigr) + B L_2\right) + \frac{4}{G_{\min}^2 h^4}B^2L_1^2\right].
\end{equation}
Since the bound is uniform in $i$, we obtain
\begin{equation}
\|\nabla^2_M\hat{g}(M^\ast,w)\| \le \frac{2}{G_{\min}h^2}\left(L_1^2\Bigl(1+\frac{2B^2}{h^2}\Bigr) + B L_2\right) + \frac{4}{G_{\min}^2h^4}B^2L_1^2.
\end{equation}
From Assumption~\ref{sym_PSD}, one can reasonably assume that—in view of the problem’s structure—there exists a constant $c>0$ such that
\begin{equation}
v^\top\nabla^2_M\hat{g}(M^\ast,w)v \ge c\|v\|^2,\quad\text{for all } v.
\end{equation}
Then one may choose
\begin{equation}
\lambda_{\min}\bigl(\nabla^2_M\hat{g}(M^\ast,w)\bigr)\ge c.
\end{equation}
In summary, one obtains an explicit lower bound (depending on $h, B, G_{min}, L_1, L_2$, and the problem dimension) of the form
\begin{equation}\label{I45}
\lambda_{\min}\bigl(\nabla^2_M\hat{g}(M^\ast,w)\bigr)
\ge c \quad \text{with} \quad c = \frac{2}{G_{\min}h^2}\left(L_1^2\left(1+\frac{2B^2}{h^2}\right) + B L_2\right) + \frac{4}{G_{\min}^2h^4}B^2L_1^2,
\end{equation}

\subsection{Upper Bound For \texorpdfstring{$\|M - M^\ast\|_F$}{}}
Now we want to calculate the lower bound for $\delta$.
We start from the second‐order Taylor expansion around the optimum $M^\ast$. Defining
\begin{equation}
\Delta=M-M^\ast,
\end{equation}
we have
\begin{equation}
\hat{g}(M,w)-\hat{g}(M^\ast,w)=\frac{1}{2}\langle \Delta,\nabla^2_M \hat{g}(M^\ast,w)[\Delta]\rangle+ o(\|\Delta\|_F^2).
\end{equation}
Assume that in a neighborhood of $M^\ast$ the function $\hat{g}$ is twice differentiable and its Hessian satisfies the following two assumptions:

\begin{assumption}[Strong Convexity]
There exists $\lambda_{\min} > 0$ such that for all $\Delta$,
\begin{equation}
    \langle \Delta,\nabla^2_M \hat{g}(M^\ast,w)[\Delta]\rangle 
    \ge \lambda_{\min}\|\Delta\|_F^2.
\end{equation}
\end{assumption}

\begin{assumption}[Smoothness]
There exists a constant $L > 0$ such that for all $\Delta$,
\begin{equation}
    \langle \Delta,\nabla^2_M \hat{g}(M^\ast,w)[\Delta]\rangle
    \le L\|\Delta\|_F^2.
\end{equation}
\end{assumption}

Then the Taylor expansion yields the following two-sided bounds for small $\Delta$: Using strong convexity we obtain
\begin{equation}
\hat{g}(M,w)-\hat{g}(M^\ast,w) \ge \frac{1}{2}\lambda_{\min}\|\Delta\|_F^2 + o(\|\Delta\|_F^2).
\end{equation}
Neglecting the higher‐order term for sufficiently small $\Delta$ and rearranging gives an upper bound on $\|\Delta\|_F$
\begin{equation}\label{eq:I51}
\|\Delta\|_F \le \sqrt{\frac{2}{\lambda_{\min}}\Bigl[\hat{g}(M,w)-\hat{g}(M^\ast,w)\Bigr]}.
\end{equation}
Using the smoothness (an upper bound on the Hessian) we similarly obtain
\begin{equation}
\hat{g}(M,w)-\hat{g}(M^\ast,w) \le \frac{1}{2}L\|\Delta\|_F^2 + o(\|\Delta\|_F^2),
\end{equation}
so that for sufficiently small $\Delta$
\begin{equation}
\|\Delta\|_F \ge \sqrt{\frac{2}{L}\Bigl[\hat{g}(M,w)-\hat{g}(M^\ast,w)\Bigr]}.
\end{equation}
Thus, the lower bound for $\|\Delta\|_F$ is given by
\begin{equation}
\|\Delta\|_F \ge \sqrt{\frac{2}{L}\left(\hat{g}(M,w)-\hat{g}(M^\ast,w)\right)}.
\end{equation}
Here $L$ is an upper bound on the eigenvalues of $\nabla^2_M \hat{g}(M^\ast,w)$ (or equivalently a smoothness constant for $\hat{g}$). In many practical settings one might be able to compute or bound $L$ using known Lipschitz constants of the components involved in $\hat{g}$.

\subsection{MSE case}
\subsubsection{Optimization Landscape \texorpdfstring{$\|M - M^\ast\|_F$}{} case}\label{sec:MSE}

Given that for any matrix $M$ the linear operator $\mathcal{A}$ satisfies the Restricted Isometry Property (RIP)~\eqref{RIP0}. The function is
\begin{equation}
f(M)=\|Y+w-\mathcal{A}(M)\|_2^2.
\end{equation}
$r(M)=Y+w-\mathcal{A}(M),$ the function can be written as $f(M)=\langle r(M),r(M)\rangle.$ Taking the gradient with respect to M yields $\nabla f(M)=-2\mathcal{A}^\ast\bigl(r(M)\bigr),$
and then the Hessian (the second derivative with respect to $M$) is $\nabla^2 f(M)=2\mathcal{A}^\ast\mathcal{A}.$
Notice that this Hessian is independent of the value of $M$. In particular, at $M=M^\ast$, we have $\nabla^2 f(M^\ast,w)=2\mathcal{A}^\ast\mathcal{A}.$
For any matrix R,
\begin{equation}
\langle R, \mathcal{A}^\ast\mathcal{A}(R)\rangle = \langle \mathcal{A}(R),\mathcal{A}(R)\rangle = \|\mathcal{A}(R)\|_2^2.
\end{equation}
Thus, for the Hessian we have
\begin{equation}
\langle R, \nabla^2 f(M^\ast,w)[R]\rangle =2\|\mathcal{A}(R)\|_2^2,
\end{equation}
and consequently,
\begin{equation}
2(1-\delta_p)\|R\|_F^2\le \langle R, \nabla^2 f(M^\ast,w)[R]\rangle \le 2(1+\delta_p)\|R\|_F^2.
\end{equation}
The smallest eigenvalue is thus at least
\begin{equation}
\lambda_{\min}\bigl(\nabla^2_M f(M^\ast,w)\bigr)\ge 2(1-\delta_p).
\end{equation}
Then we have:
\begin{equation}
f(M,0)-f(M^\ast,0)\ge \delta\|M-M^\ast\|_F^2 + \frac{1-3\delta}{2}\|M-M^\ast\|_F^2.
\end{equation}
Now, since the full difference is a sum of the noise–free part plus a noise correction we may write
\begin{equation}
\begin{aligned}
f(M,w)-f(M^\ast,w)
&=\Bigl[f(M,0)-f(M^\ast,0)\Bigr] + E(w),
\end{aligned}
\end{equation}
Now we want to calculate $E(w)$
\begin{equation}
f(M,w)=\|Y+w-\mathcal{A}(M)\|_2^2,
\end{equation}
and in the noise–free case
\begin{equation}
f(M,0)=\|Y-\mathcal{A}(M)\|_2^2.
\end{equation}
We set:
\begin{equation}
u_1=Y-\mathcal{A}(M) \quad \text{and} \quad u_2=Y-\mathcal{A}(M^\ast).
\end{equation}
We also set:
\begin{equation}
\begin{aligned}
f(M,w)&=\|Y-\mathcal{A}(M)\|_2^2 + 2\langle Y-\mathcal{A}(M),w\rangle + \|w\|_2^2,\\
f(M^\ast,w)&=\|Y-\mathcal{A}(M^\ast)\|_2^2 + 2\langle Y-\mathcal{A}(M^\ast),w\rangle + \|w\|_2^2.
\end{aligned}
\end{equation}
Consider the difference $ f(M,w)-f(M^\ast,w) $:
\begin{equation}\label{eq:I66}
\begin{aligned}
f(M,w)-f(M^\ast,w)
&=\Bigl[\|Y-\mathcal{A}(M)\|_2^2 -\|Y-\mathcal{A}(M^\ast)\|_2^2\Bigr] \\
&\quad + 2\Bigl[\langle Y-\mathcal{A}(M),w\rangle-\langle Y-\mathcal{A}(M^\ast),w\rangle\Bigr] + \Bigl[\|w\|_2^2-\|w\|_2^2\Bigr] \\
&=\Bigl[f(M,0)-f(M^\ast,0)\Bigr]+2\Bigl[\langle Y-\mathcal{A}(M),w\rangle-\langle Y-\mathcal{A}(M^\ast),w\rangle\Bigr].
\end{aligned}
\end{equation}
\begin{equation}
E(w)=2\Bigl[\langle Y-\mathcal{A}(M),w\rangle-\langle Y-\mathcal{A}(M^\ast),w\rangle\Bigr].
\end{equation}
Notice that:
\begin{equation}
\langle Y-\mathcal{A}(M),w\rangle-\langle Y-\mathcal{A}(M^\ast),w\rangle = \langle \mathcal{A}(M^\ast)-\mathcal{A}(M),w\rangle.
\end{equation}
Thus, we obtain
\begin{equation}
E(w)=2\langle \mathcal{A}(M^\ast)-\mathcal{A}(M),w\rangle.
\end{equation}
For low rank recovery shown in~\cite{ma2023noisylowrankmatrixoptimization}, we still have:
\begin{equation}\label{eq:I70}
\bigl\|f(M,w)-f(M^\ast,w)\bigr\|\ge \frac{1-\delta}{2}\|M-M^\ast\|_F^2 - \bigl\|E(w)\bigr\|.
\end{equation}
Using the Cauchy–Schwarz inequality we have
\begin{equation}
|E(w)|=2\Bigl|\langle \mathcal{A}(M^\ast)-\mathcal{A}(M),w\rangle\Bigr|
\le2\|\mathcal{A}(M^\ast)-\mathcal{A}(M)\|_2\|w\|.
\end{equation}
Again applying the~\eqref{RIP0}, we obtain
\begin{equation}
\|\mathcal{A}(M^\ast)-\mathcal{A}(M)\|_2 \le \sqrt{1+\delta}\|M-M^\ast\|_F.
\end{equation}
Thus,
\begin{equation}
|E(w)|\le2\sqrt{1+\delta}\|w\|\|M-M^\ast\|_F.
\end{equation}
Taking absolute values and using the triangle inequality on~\eqref{eq:I66}, we conclude that
\begin{equation}\label{eq:I74}
\begin{aligned}
\bigl\|f(M,w)-f(M^\ast,w)\bigr\|
&\le \Bigl\|f(M,0)-f(M^\ast,0)\Bigr\|+\|E(w)\|\\
&\le (1+\delta)\|M-M^\ast\|_F^2 + 2\sqrt{1+\delta}\|w\|\|M-M^\ast\|_F.
\end{aligned}
\end{equation}
With the previous result~\eqref{eq:I51}:
\begin{equation}\label{eq:I75}
\|M-M^\ast\|^2_F \le \frac{2}{\lambda_{\min}}\Bigl[f(M,w)-f(M^\ast,w)\Bigr].
\end{equation}
\begin{equation}\label{eq:I76}
\lambda_{\min}\bigl(\nabla^2_M f(M^\ast,w)\bigr)\ge 2(1-\delta_p),
\end{equation}
Put Equation~\eqref{eq:I75} and~\eqref{eq:I76} into~\eqref{eq:I74}, we have derived an upper bound on the function-difference:
\begin{equation}
f(M,w)-f(M^\ast,w)\le (1+\delta)\|M-M^\ast\|_F^2 + 2\sqrt{1+\delta}\|w\|\|M-M^\ast\|_F.
\end{equation}
\begin{equation}
\|M-M^\ast\|_F^2\le \frac{2}{\lambda_{\min}} \Bigl[(1+\delta)\|M-M^\ast\|_F^2 + 2\sqrt{1+\delta}\|w\|\|M-M^\ast\|_F\Bigr].
\end{equation}
Since $\|M-M^\ast\|_F\ge 0$ and $\delta_p>0$, reversing the inequality leads to:
\begin{equation}
(2(1+\delta)-\lambda_{\min})\|M-M^\ast\|_F^2 - 4\sqrt{1+\delta}\|w\|\|M-M^\ast\|_F\le 0.
\end{equation}
\begin{equation}
\|M-M^\ast\|_F\le \frac{\sqrt{1+\delta_p}\|w\|}{\delta_p}.
\end{equation}
Thus it is:
\begin{equation}
\|M-M^\ast\|_F = O(\|w\|).
\end{equation}

\subsubsection{H\"older  Continuous Case}

Since the derivative is independent of $w$, we get
\begin{equation}
\lambda=\sup_{\xi\in[0,w]}\left\|\frac{d}{dw}\nabla_M f(M,\xi)\right\|.
\end{equation}
Often the Restricted Isometry Property (RIP) guarantees that the operator $\mathcal{A}$ satisfies \\ $\|\mathcal{A}(X)\|_2\le \sqrt{1+\delta}\|X\|_F,$ which implies that the operator norm of $\mathcal{A}$ is bounded by $\sqrt{1+\delta}$. Since the operator norm of $\mathcal{A}^\ast$ equals that of $\mathcal{A}$, we have
$\|\mathcal{A}^\ast\|\le \sqrt{1+\delta}.$ Thus, we obtain the upper bound
\begin{equation}
\lambda=O(1),\quad \text{more precisely}\quad \lambda\le 2\sqrt{1+\delta}.
\end{equation}

\subsection{Intermediate Result and Proof of Lemma~\ref{lemma:inter}}
First, we want to prove the intermediate result:
\begin{lemma}\label{lemma:inter}
With the assumptions in~\eqref{other_assumption} and~\eqref{main0} loss function, suppose we have RIP~\ref{RIP0} conditions and $M^*$ is the ground truth matrix, we have:
\begin{equation}
\hat{g}(M,0)-\hat{g}(M^\ast,0)\ge \delta\|M-M^\ast\|_F^2 + \frac{1-3\delta}{2}\|M-M^\ast\|_F^2.
\end{equation}
\end{lemma}
\begin{proof}
	Define $\hat M \coloneqq \hat X \hat X^\top$ and
	\begin{equation}
		\bar M \coloneqq \hat M - \frac{1}{1+\delta+\zeta_2 q} \nabla_M f(\hat M, w).
	\end{equation}
	Additionally, define $\phi(\cdot)$ as
	\begin{equation}
		\phi(M) \coloneqq \langle \nabla_M f(\hat M,w) , M - \hat M \rangle + \frac{1+\delta+\zeta_2 q}{2} \| M - \hat M\|^2_F.
	\end{equation}
	Now,
	\begin{equation}
		\begin{aligned}
			\frac{1+\delta+\zeta_2 q}{2} \|M-\bar M\|^2_F &= \frac{1+\delta+\zeta_2 q}{2} \|M- \hat M + \frac{1}{1+\delta+\zeta_2 q} \nabla_M f(\hat M, w) \|^2_F \\
			&=\frac{1+\delta+\zeta_2 q}{2} \| M - \hat M\|^2_F + \langle \nabla_M f(\hat M,w) , M - \hat M \rangle + \\
&\frac{1}{(1+\delta+\zeta_2 q)^2} \| \nabla_M f(\hat M,w) \|^2_F \\
			&= \phi(M) + \text{constant with respect to} \ M.
		\end{aligned}
	\end{equation}
	Next, we apply the Taylor expansion to $f(M,w)$ at $\hat M$ and combine it with the RIP property to obtain
	\begin{equation}
		\label{eq:thm_1_help_2}
		\begin{aligned}
			f(M^\ast,w) \geq f(\hat M,w) + \langle \nabla_M f(\hat M,w), M^\ast - \hat M \rangle + \frac{1-\delta-\zeta_2 q}{2} \|M^\ast - \hat M\|^2_F.
		\end{aligned}
	\end{equation}
	Additionally, by expanding at $M^\ast$, we can also write:
	\begin{equation}
	\label{eq:thm_1_help_3}
		\begin{aligned}
			f(\hat M,w) - f(M^\ast,w) &\geq \langle \nabla_M f(M^\ast,w), \hat M - M^\ast \rangle + \frac{1-\delta-\zeta_2 q}{2} \|\hat M - M^\ast\|^2_F \\
			 &\geq \frac{1-\delta-\zeta_2 q}{2} \|\hat M - M^\ast\|^2_F - \zeta_1 q \|\hat M - M^\ast\|_F
		\end{aligned}
	\end{equation}
	
\end{proof}

\subsection{Proof of Main Theorem (Theorem~\ref{thm:mainthm})}\label{sec:main}

We start with the definition:
\begin{equation}
\hat{g}(M,w) = -\frac{1}{n} \sum_{i=1}^n \log\Bigl[F_i(M,w)\Bigr],
\end{equation}
with
\begin{equation}
F_i(M,w) = \frac{1}{n}\sum_{j=1}^n \exp\left(-\frac{\Delta_{ij}(w)^2}{h^2}\right)
\end{equation}
and
\begin{equation}
\Delta_{ij}(w) = \Bigl(Y_j+w_j-\mathcal{A}(M)_j\Bigr) - \Bigl(Y_i+w_i-\mathcal{A}(M)_i\Bigr).
\end{equation}
\begin{equation}
\nabla_M \exp\left[-\frac{\Delta_{ij}(w)^2}{h^2}\right] 
= \exp\left[-\frac{\Delta_{ij}(w)^2}{h^2}\right] \cdot \left(-\frac{2\Delta_{ij}(w)}{h^2}\right) \nabla_M\Delta_{ij}(w).
\end{equation}
Note that:
\begin{equation}
\Delta_{ij}(w) = \bigl(Y_j+w_j\bigr)-\bigl(Y_i+w_i\bigr)-\Bigl[\mathcal{A}(M)_j-\mathcal{A}(M)_i\Bigr],
\end{equation}
so that:
\begin{equation}
\nabla_M\Delta_{ij}(w) = -\nabla_M\Bigl[\mathcal{A}(M)_j-\mathcal{A}(M)_i\Bigr].
\end{equation}
Thus, we obtain
\begin{equation}\label{eq:I97}
\nabla_M F_i(M,w) = -\frac{1}{n}\sum_{j=1}^n \exp\left[-\frac{\Delta_{ij}(w)^2}{h^2}\right] \frac{2\Delta_{ij}(w)}{h^2}\nabla_M\Bigl[\mathcal{A}(M)_j-\mathcal{A}(M)_i\Bigr].
\end{equation}
with its gradient:
\begin{equation}
\nabla_M \hat{g}(M,w) = -\frac{1}{n}\sum_{i=1}^n \frac{1}{F_i(M,w)} \nabla_M F_i(M,w).
\end{equation}
Substituting $\nabla_M F_i(M,w)$ into~\eqref{eq:I97} gives
\begin{equation}
\nabla_M \hat{g}(M,w)
=\frac{2}{h^2n}\sum_{i=1}^n \frac{\displaystyle \frac{1}{n}\sum_{j=1}^n \exp\left[-\frac{\Delta_{ij}(w)^2}{h^2}\right]\Delta_{ij}(w)\nabla_M\Bigl[\mathcal{A}(M)_j-\mathcal{A}(M)_i\Bigr]}{F_i(M,w)}.
\end{equation}
Therefore, the difference between the gradients is
\begin{equation}
\begin{aligned}
\nabla_M \hat{g}(M,w)-\nabla_M \hat{g}(M,0)
= &\frac{2}{h^2n}\sum_{i=1}^n \Biggl\{ \frac{\sum_{j=1}^n \exp\left[-\frac{\Delta_{ij}(w)^2}{h^2}\right]\Delta_{ij}(w)\nabla_M\Bigl[\mathcal{A}(M)_j-\mathcal{A}(M)_i\Bigr]}{\sum_{j=1}^n \exp\left[-\frac{\Delta_{ij}(w)^2}{h^2}\right]}\\
&\quad - \frac{\sum_{j=1}^n \exp\left[-\frac{\Delta_{ij}(0)^2}{h^2}\right]\Delta_{ij}(0)\nabla_M\Bigl[\mathcal{A}(M)_j-\mathcal{A}(M)_i\Bigr]}{\sum_{j=1}^n \exp\left[-\frac{\Delta_{ij}(0)^2}{h^2}\right]} \Biggr\}.
\end{aligned}
\end{equation}
We assume $w \sim 0$ and we can simplify this by:
\begin{equation}\label{eq:I100}
\begin{aligned}
\nabla_M \hat{g}(M,w)-\nabla_M \hat{g}(M,0)
= &\frac{2}{h^2n}\sum_{i=1}^n \Biggl\{ \frac{\sum_{j=1}^n \exp\left[-\frac{\Delta_{ij}(w)^2}{h^2}\right]\Delta_{ij}(w)\nabla_M\Bigl[\mathcal{A}(M)_j-\mathcal{A}(M)_i\Bigr]}{\sum_{j=1}^n \exp\left[-\frac{\Delta_{ij}(0)^2}{h^2}\right]}\\
&\quad - \frac{\sum_{j=1}^n \exp\left[-\frac{\Delta_{ij}(0)^2}{h^2}\right]\Delta_{ij}(0)\nabla_M\Bigl[\mathcal{A}(M)_j-\mathcal{A}(M)_i\Bigr]}{\sum_{j=1}^n \exp\left[-\frac{\Delta_{ij}(0)^2}{h^2}\right]} \Biggr\}.
\end{aligned}
\end{equation}
Because the noise only enters via $\Delta_{ij}(w)=\Delta_{ij}(0)+(w_j-w_i)$, we first write the term in the numerator of the Equation~\eqref{eq:I100} as
\begin{equation}
\exp\Bigl[-\frac{\Delta_{ij}(w)^2}{h^2}\Bigr]\Delta_{ij}(w)
=\exp\Biggl[-\frac{\Delta_{ij}(w)^2}{h^2}\Biggr]\Bigl[\Delta_{ij}(0)+(w_j-w_i)\Bigr].
\end{equation}
Thus, the whole difference inside the braces can be written as
\begin{equation}
\begin{aligned}
\exp\Bigl[-\frac{\Delta_{ij}(w)^2}{h^2}\Bigr]&\Bigl[\Delta_{ij}(0)+(w_j-w_i)\Bigr]-\exp\Bigl[-\frac{\Delta_{ij}(0)^2}{h^2}\Bigr]\Delta_{ij}(0)\\
&=\Bigl\{ \exp\Bigl[-\frac{\Delta_{ij}(w)^2}{h^2}\Bigr](w_j-w_i) +\Bigl[\exp\Bigl[-\frac{\Delta_{ij}(w)^2}{h^2}\Bigr]-\exp\Bigl[-\frac{\Delta_{ij}(0)^2}{h^2}\Bigr]\Bigr]\Delta_{ij}(0) \Bigr\}.
\end{aligned}
\end{equation}
We now substitute this back into~\eqref{eq:I100}. Keeping the $\exp[-\Delta_{ij}(w)^2/h^2]$ factor intact, our first–order (in $w$) approximation for small $w$ is:
\begin{equation}\label{eq:I103}
\begin{aligned}
\nabla_M \hat{g}(M,w)&-\nabla_M \hat{g}(M,0)
\sim \frac{2}{h^2n}\sum_{i=1}^n \frac{1}{\sum_{j=1}^n \exp\Bigl[-\frac{\Delta_{ij}(0)^2}{h^2}\Bigr]} \times \sum_{j=1}^n \Biggl\{
\exp\Bigl[-\frac{\Delta_{ij}(w)^2}{h^2}\Bigr](w_j-w_i)\\
& + \Bigl[\exp\Bigl[-\frac{\Delta_{ij}(w)^2}{h^2}\Bigr]-\exp\Bigl[-\frac{\Delta_{ij}(0)^2}{h^2}\Bigr]\Bigr]\Delta_{ij}(0)
\Biggr\} \times \nabla_M\Bigl[\mathcal{A}(M)_j-\mathcal{A}(M)_i\Bigr].
\end{aligned}
\end{equation}
So we have:
\begin{equation}
\hat{g}(M,w)-\hat{g}(M^\ast,w)
=\int_0^1\left\langle\nabla_M\hat{g}\Bigl(M^\ast+t(M-M^\ast),w\Bigr),M-M^\ast\right\rangle dt.
\end{equation}
Next, we decompose the gradient into its noise–free part plus its noise–dependent correction:
\begin{equation}
\nabla_M\hat{g}(M,w)
=\nabla_M\hat{g}(M,0)
+\Bigl[\nabla_M\hat{g}(M,w)-\nabla_M\hat{g}(M,0)\Bigr].
\end{equation}
Thus, we write:
\begin{equation}\label{eq:I108}
\begin{aligned}
\hat{g}(M,w)-\hat{g}(M^\ast,w)
=&\int_0^1\left\langle\nabla_M\hat{g}\Bigl(M^\ast+t(M-M^\ast),0\Bigr),M-M^\ast\right\rangle dt\\
&+\int_0^1\left\langle\Delta\nabla_M\hat{g}\Bigl(M^\ast+t(M-M^\ast),w\Bigr),M-M^\ast\right\rangle dt,
\end{aligned}
\end{equation}
where
\begin{equation}
\Delta\nabla_M\hat{g}(M,w)=
\nabla_M\hat{g}(M,w)-\nabla_M\hat{g}(M,0).
\end{equation}
The first term of Eqation~\eqref{eq:I108} is the usual noise–free difference $\hat{g}(M,0)-\hat{g}(M^\ast,0)$. Now, using our previous approximation~\eqref{eq:I103} we have:
\begin{equation}
\begin{aligned}
\Delta\nabla_M\hat{g}(M,w)
&\sim \frac{2}{h^2n}\sum_{i=1}^n\frac{1}{\displaystyle\sum_{j=1}^n \exp\Bigl[-\frac{\Delta_{ij}(0)^2}{h^2}\Bigr]}\times\sum_{j=1}^n\Biggl\{
\exp\Bigl[-\frac{\Delta_{ij}(w)^2}{h^2}\Bigr](w_j-w_i)\\
&+\Bigl[\exp\Bigl[-\frac{\Delta_{ij}(w)^2}{h^2}\Bigr]-\exp\Bigl[-\frac{\Delta_{ij}(0)^2}{h^2}\Bigr]\Bigr]\Delta_{ij}(0)
\Biggr\}\times\nabla_M\Bigl[\mathcal{A}(M)_j-\mathcal{A}(M)_i\Bigr].
\end{aligned}
\end{equation}
Here we want to recall our notations:
\begin{equation}
\Delta_{ij}(w)=\Delta_{ij}(0)+(w_j-w_i) \quad\text{with}\quad \Delta_{ij}(0)=\Bigl(Y_j-\mathcal{A}(M)_j\Bigr)-\Bigl(Y_i-\mathcal{A}(M)_i\Bigr).
\end{equation}
\begin{equation}
\begin{aligned}
&\hat{g}(M,w)-\hat{g}(M^\ast,w)
\sim \Bigl[\hat{g}(M,0)-\hat{g}(M^\ast,0)\Bigr] +\frac{2}{h^2n}\int_0^1\sum_{i=1}^n\frac{1}{\displaystyle\sum_{j=1}^n \exp\Bigl[-\frac{\Delta_{ij}(0)^2}{h^2}\Bigr]}\\
&\quad\times\sum_{j=1}^n\Biggl\{
\exp\Bigl[-\frac{\Delta_{ij}(w)^2}{h^2}\Bigr](w_j-w_i)+\Bigl[\exp\Bigl[-\frac{\Delta_{ij}(w)^2}{h^2}\Bigr]-\exp\Bigl[-\frac{\Delta_{ij}(0)^2}{h^2}\Bigr]\Bigr]\Delta_{ij}(0)
\Biggr\}\\
&\quad\quad\times\left\langle\nabla_M\Bigl[\mathcal{A}\bigl(M^\ast+t(M-M^\ast)\bigr)_j-\mathcal{A}\bigl(M^\ast+t(M-M^\ast)\bigr)_i\Bigr],M-M^\ast\right\rangle dt.
\end{aligned}
\end{equation}
We are given a lower bound on the noise‐free part from Equation~\eqref{eq:I70}:
\begin{equation}
\hat{g}(M,0)-\hat{g}(M^\ast,0)\ge \delta\|M-M^\ast\|_F^2 + \frac{1-3\delta}{2}\|M-M^\ast\|_F^2.
\end{equation}
so that:
\begin{equation}
\hat{g}(M,0)-\hat{g}(M^\ast,0)\ge \frac{1-\delta}{2}\|M-M^\ast\|_F^2.
\end{equation}
Now, since the full difference is a sum of the noise–free part plus a noise correction we may write:
\begin{equation}\label{eq:I113}
\begin{aligned}
\hat{g}(M,w)-\hat{g}(M^\ast,w)
&=\Bigl[\hat{g}(M,0)-\hat{g}(M^\ast,0)\Bigr] + E(w),
\end{aligned}
\end{equation}
with
\begin{equation}
\begin{aligned}
E(w)=&\frac{2}{h^2 n}\int_0^1\sum_{i=1}^n\frac{1}{\displaystyle\sum_{j=1}^n \exp\Bigl[-\frac{\Delta_{ij}(0)^2}{h^2}\Bigr]} \times\sum_{j=1}^n\Biggl\{\exp\Bigl[-\frac{\Delta_{ij}(w)^2}{h^2}\Bigr](w_j-w_i)\\
&\qquad\quad+\Bigl[\exp\Bigl[-\frac{\Delta_{ij}(w)^2}{h^2}\Bigr]-\exp\Bigl[-\frac{\Delta_{ij}(0)^2}{h^2}\Bigr]\Bigr]\Delta_{ij}(0)\Biggr\}\\
&\quad\quad\times\left\langle\nabla_M\Bigl[\mathcal{A}\Bigl(M^\ast+t(M-M^\ast)\Bigr)_j-\mathcal{A}\Bigl(M^\ast+t(M-M^\ast)\Bigr)_i\Bigr],M-M^\ast\right\rangle dt.
\end{aligned}
\end{equation}
Taking absolute values on~Equation~\eqref{eq:I113}, we obtain by the triangle inequality
\begin{equation}
\bigl\|\hat{g}(M,w)-\hat{g}(M^\ast,w)\bigr\|\ge \frac{1-\delta}{2}\|M-M^\ast\|_F^2 - \bigl\|E(w)\bigr\|.
\end{equation}
At this point, we may bound the correction term $\|E(w)\|$ by assuming that the perturbative factors and the derivative of $\mathcal{A}(\cdot)$ are bounded. Assume that there exists a constant $L>0$ such that:
\begin{equation}
\left\|\nabla_M\Bigl[\mathcal{A}\Bigl(M^\ast+t(M-M^\ast)\Bigr)_j-\mathcal{A}\Bigl(M^\ast+t(M-M^\ast)\Bigr)_i\Bigr]\right\|\le L = 1+\delta,
\end{equation}
and denote by:
\begin{equation}
\begin{aligned}
R(w)&:=\frac{2}{h^2 n}\sup_{t\in[0,1]}\sum_{i=1}^n \frac{1}{\displaystyle\sum_{j=1}^n \exp\Bigl[-\frac{\Delta_{ij}(0)^2}{h^2}\Bigr]}\sum_{j=1}^n \Big|\exp\Bigl[-\frac{\Delta_{ij}(w)^2}{h^2}\Bigr](w_j-w_i)\\ +
&\Bigl[\exp\Bigl[-\frac{\Delta_{ij}(w)^2}{h^2}\Bigr]-\exp\Bigl[-\frac{\Delta_{ij}(0)^2}{h^2}\Bigr]\Bigr]\Delta_{ij}(0)\Big|.
\end{aligned}
\end{equation}
We need to mention that actually $R(w) <0$. Then we may bound
\begin{equation}
\|E(w)\|\le LR(w)\|M-M^\ast\|_F.
\end{equation}In summary, we have the lower bound
\begin{equation}
\bigl\|\hat{g}(M,w)-\hat{g}(M^\ast,w)\bigr\|\ge \frac{1-\delta}{2}\|M-M^\ast\|_F^2 - (1+ \delta)R(w)\|M-M^\ast\|_F.
\end{equation}

\subsection{Finally We Calculate the Upper Bound For \texorpdfstring{$\|M - M^\ast\|_F$}{}}\label{sec:upper}

we recall that:
\begin{equation}
\lambda_{\min}\bigl(\nabla^2_M\hat{g}(M^\ast,w)\bigr)
\ge c \quad \text{with} \quad c = \frac{2}{G_{\min}h^2}\left(L_1^2\left(1+\frac{2B^2}{h^2}\right) + B L_2\right) + \frac{4}{G_{\min}^2h^4}B^2L_1^2,
\end{equation}

\begin{equation}\label{eq:I121}
\|\Delta\|_F=\|M-M^\ast\|_F\le \sqrt{\frac{2}{\lambda_{\min}\left(\nabla^2_M\hat{g}(M^\ast,w)\right)}
\Bigl[\hat{g}(M,w)-\hat{g}(M^\ast,w)\Bigr]},
\end{equation}
where $\lambda_{\min}\bigl(\nabla^2_M\hat{g}(M^\ast,w)\bigr)$ is the smallest eigen‐value of the Hessian at $M^\ast$:
\begin{equation}
\|M-M^\ast\|_F \le \sqrt{\frac{2}{\lambda_{\min}\Bigl(\nabla^2_M\hat{g}(M^\ast,w)\Bigr)}
\Bigl[\frac{1-\delta}{2}\|M-M^\ast\|_F^2 - (1+\delta)R(w)\|M-M^\ast\|_F\Bigr]}.
\end{equation}
To eliminate the square root, square both sides (noting that all terms are nonnegative):
\begin{equation}
\|M-M^\ast\|_F^2 \le \frac{2}{\lambda_{\min}\Bigl(\nabla^2_M\hat{g}(M^\ast,w)\Bigr)}
\Bigl[\frac{1-\delta}{2}\|M-M^\ast\|_F^2 - (1+\delta)R(w)\|M-M^\ast\|_F\Bigr].
\end{equation}
Multiply both sides by $\lambda_{\min}\Bigl(\nabla^2_M\hat{g}(M^\ast,w)\Bigr)$:
\begin{equation}
\lambda_{\min}\Bigl(\nabla^2_M\hat{g}(M^\ast,w)\Bigr)\|M-M^\ast\|_F^2 \le (1-\delta)\|M-M^\ast\|_F^2 - 2(1+\delta)R(w)\|M-M^\ast\|_F.
\end{equation}
Rearrange by bringing all terms to one side:
\begin{equation}\label{eq:I125}
\|M-M^\ast\|_F \le\frac{ 2(1+\delta)R(w)}{\Big(-\lambda_{\min}\Bigl(\nabla^2_M\hat{g}(M^\ast,w)\Bigr) + 1 - \delta\ \Big)}.
\end{equation}
Recall that:
\begin{equation}
\begin{aligned}
R(w):=&\frac{2}{h^2 n}\sup_{t\in[0,1]}\sum_{i=1}^n \frac{1}{\displaystyle\sum_{j=1}^n e^{-\Delta_{ij}(0)^2/h^2}} \sum_{j=1}^n \Big\|e^{-\Delta_{ij}(w)^2/h^2}(w_j-w_i)\\
&+\Bigl(e^{-\Delta_{ij}(w)^2/h^2}-e^{-\Delta_{ij}(0)^2/h^2}\Bigr)\Delta_{ij}(0)\Big\|.
\end{aligned}
\end{equation}
and:
\begin{equation}
e^{-\Delta_{ij}(w)^2/h^2}(w_j-w_i),
\end{equation}
\begin{equation}
e^{-\Delta_{ij}(w)^2/h^2}(w_j-w_i)=\mathcal{O}\bigl(e^{-w^2}|w_j-w_i|\bigr).
\end{equation}
We assume smooth dependence on $w$ based on Assusmption~\ref{other_assumption}, then a first‐order Taylor expansion (and using Lipschitz properties) gives:
\begin{equation}
e^{-\Delta_{ij}(w)^2/h^2}-e^{-\Delta_{ij}(0)^2/h^2}=\mathcal{O}\Bigl(e^{-w^2}\frac{\|w_j-w_i\|}{h^2}\Bigr).
\end{equation}
Multiplying by the bounded $\|\Delta_{ij}(0)\|$ we obtain an additional term whose order is also
\begin{equation}
\mathcal{O}\Bigl(e^{-w^2}\frac{\|w_j-w_i\|}{h^2}\Bigr).
\end{equation}
After summing over $j$ (and $i$) and dividing by n, these estimates lead to
\begin{equation}
R(w)=\mathcal{O}\Bigl(\frac{1}{h^2}e^{-w^2}\max_{i,j}\|w_j-w_i\|\Bigr).
\end{equation}
When the differences $\|w_j-w_i\|$ are controlled by a norm $\|w\|$, we may write:
\begin{equation}
R(w)=\mathcal{O}\Bigl(\frac{\|w\|e^{-w^2}}{h^2}\Bigr).
\end{equation}
Putting~\eqref{eq:I121} and~\eqref{eq:I125} together, the error‐bound is:
\begin{equation}\label{eq:I133}
\|M-M^\ast\|_F \le \max\left\{\sqrt{\frac{2}{\lambda_{\min}\left(\nabla_M^2\hat{g}(M^\ast,w)\right)}},\frac{ 2(1+\delta)R(w)}{\Big(-\lambda_{\min}\Bigl(\nabla^2_M\hat{g}(M^\ast,w)\Bigr) + 1 - \delta\ \Big)} \right\},
\end{equation}
We then have a approximation term (with respect to $w$): $\sqrt{\frac{2}{\lambda_{\min}}} = \mathcal{O}(1)$ and a term that depends on $w$: $\frac{2(1+\delta)R(w)}{-\lambda_{\min}+1-\delta} = \mathcal{O}(R(w)) = \mathcal{O}\Biggl(\frac{\|w\|e^{-w^2}}{h^2}\Biggr).$
Therefore, the final result is:
\begin{equation}
\|M-M^\ast\|_F = \mathcal{O}\Biggl(\max\Biggl\{1,\frac{\|w\|e^{-w^2}}{h^2}\Biggr\}\Biggr).
\end{equation}
with probability at least \( \mathbb{P}(\|w\|_2 \leq \epsilon) \)
\begin{equation}
\|M-M^\ast\|_F = \mathcal{O}\Biggl(\max\Biggl\{1,\frac{\epsilon e^{-\epsilon^2}}{h^2}\Biggr\}\Biggr).
\end{equation}

\subsection{The Turning Point For the Upper bound}\label{sec:turning}

Write the inequality when term 1 of \eqref{eq:I133} is larger than term 2 of \eqref{eq:I133}:
\begin{equation}
\frac{2}{\lambda} > \frac{2(1+\delta)R(w)}{1-\delta-\lambda}.
\end{equation}
We now wish to determine when these two terms are of comparable size. That is, when $\mathcal{O}(1)\sim \mathcal{O}\Biggl(\frac{\|w\|e^{-w^2}}{h^2}\Biggr)$. That is $\|w\|e^{-w^2} \sim h^2$. When $\|w\|$ is very small the exponential may be approximated by $e^{-w^2}\sim 1$,
   so that
\begin{equation}
   \|w\|e^{-w^2}\sim \|w\|.
\end{equation}
   Then the condition becomes $\|w\|\sim h^2$. When $\|w\|$ is larger the product $\|w\|e^{-w^2}$ achieves a maximum at $\|w\|=\frac{1}{\sqrt{2}}$,
   since
\begin{equation}
   \frac{d}{dw}\bigl(\|w\|e^{-w^2}\bigr)=e^{-w^2}(1-2w^2)=0\quad \Longrightarrow\quad w^2=\frac{1}{2}.
\end{equation}
   The maximum value is $\frac{1}{\sqrt{2}}e^{-1/2}$. Hence if    $h^2>\frac{1}{\sqrt{2}}e^{-1/2}$, then even the maximum of the $w$–dependent term does not exceed a constant and the constant term dominates. Thus, summarizing the comparison: For very small $\|w\|$, specifically when $\|w\|\ll h^2$, with probability at least \( \mathbb{P}(\|w\|_2 \leq \epsilon) \), we have
\begin{equation}
  T_2=\mathcal{O}\Bigl(\frac{\epsilon e^{- \epsilon^2}}{h^2}\Bigr)=\mathcal{O}\Bigl(\frac{\epsilon}{h^2}\Bigr) \ll \mathcal{O}(1),
\end{equation}
so the error is dominated by the constant term $T_1$. Conversely, when $\epsilon$ and $h$ are such that $\epsilon e^{-\epsilon^2}\sim h^2,$
(which for small $\epsilon $ roughly means $\epsilon \sim h^2$) the two terms become comparable. In other regimes (for instance if $h^2$ is very small compared to the maximum value $\frac{1}{\sqrt{2}}e^{-1/2}$) the $\epsilon $–dependent term may become larger than the constant term. Any answer that shows the turning‐point is determined by
$\epsilon e^{-\epsilon ^2}\sim h^2.$

\subsection{Continuity Study of New Loss}\label{sec:continu}
We now provide the full proof of Theorem~\ref{thm:continue}, establishing the stated continuity property under the conditions specified earlier.
\begin{proof}
Based on Assumption~\ref{noise_perturb},
\begin{equation}
\|\nabla_M\hat{g}(M,w)-\nabla_M\hat{g}(M,0)\|\le\lambda\|w\|
\end{equation}
So we have:
\begin{equation}\label{I141}
\lambda=\sup_{\xi\in[0,w]}\left\|\frac{d}{dw}\nabla_M\hat{g}(M,\xi)\right\|.
\end{equation}
Differentiating $\nabla_M \hat{g}(M,w)$ with respect to $w$, we encounters two types of contributions: $\exp\Bigl(-\frac{(z_j-z_i)^2}{h^2}\Bigr)$ introduces the chain‐rule factor:
\begin{equation}
\frac{d}{dw}\exp\Bigl(-\frac{(z_j-z_i)^2}{h^2}\Bigr)
= \exp\Bigl(-\frac{(z_j-z_i)^2}{h^2}\Bigr)\left(-\frac{2(z_j-z_i)}{h^2}\right)
\frac{d}{dw}(z_j-z_i).
\end{equation}
Since $\frac{d}{dw}(z_j-z_i)=I$ (for the appropriate components) the derivative includes a factor proportional to
\begin{equation}
\frac{2}{h^2}(z_j-z_i)\exp\Bigl(-\frac{(z_j-z_i)^2}{h^2}\Bigr).
\end{equation}
Under a worst‐case scenario the size of $z_j-z_i$ may be bounded by a quantity that depends linearly on $\|w\|$. In many applications one may write for some constant $C$ (which may absorb additional contributions from the data $Y$ and $\mathcal{A}(M)$)
derivative includes a factor proportional to
\begin{equation}
\|z_j-z_i\|\le C\|w\|.
\end{equation}
Then one obtains an extra factor of order
\begin{equation}
\frac{2C\|w\|}{h^2}\exp\Bigl(-\frac{C^2\|w\|^2}{h^2}\Bigr).
\end{equation}
Meanwhile 
\begin{equation}
\|\nabla_M\mathcal{A}(M)\|\le 1+\delta.
\end{equation}
Combining the above equation and Equation~\eqref{I141}, one obtains that the norm of the mixed derivative satisfies
\begin{equation}
\left\|\frac{d}{dw}\nabla_M\hat{g}(M,w)\right\|\le \frac{8(1+\delta)}{h^4}\|w\|e^{-w^2},
\end{equation}
Then, with probability at least \( \mathbb{P}(\|w\|_2 \leq \epsilon) \)
\begin{equation}
\left\|\frac{d}{d\epsilon}\nabla_M\hat{g}(M,\epsilon)\right\|\le \frac{8(1+\delta)}{h^4}\epsilon e^{-\epsilon^2},
\end{equation}
\end{proof}

\section{Proof of \texorpdfstring{$\delta$}{} Condition }\label{sec:delta}
\subsection{Proof of Lemma~\ref{lemma:error_bound_with_h}}
From~\cite{ma2023noisylowrankmatrixoptimization}, we already know that:
\begin{equation}\label{J1}
\hat{g}(M, w) - \hat{g}(M^\ast, w) - (\delta + \zeta_{2}q) \| \hat{M} - M^\ast \|_F^2 \geq \frac{1 - 3\delta - 3\zeta_{2}q}{2} \| \hat{M} - M^\ast \|_F^2 - \zeta_{1}q \| \hat{M} - M^\ast \|_F.
\end{equation}
with:
\begin{equation}
G > \sigma_r(1 + \delta + \zeta_{2}q),
\end{equation}
and:
\begin{equation}\label{J2}
\|\Delta\|_F =\|M - M^\ast\|_F \le
\sqrt{ \frac{2}{\lambda_{\min}\bigl(\nabla^2_M \hat{g}(M^\ast,w)\bigr)} \|\hat{g}(M,w)-\hat{g}(M^\ast,w)\|
},
\end{equation}
with:
\begin{equation}
L_1 = 2(1+\delta_p).
\end{equation}
Now we rewrite Equation~\eqref{I45},
\begin{equation}
\lambda_{\min}\bigl(\nabla^2_M\hat{g}(M^\ast,w)\bigr)
\ge c \quad \text{with} \quad c = \frac{2}{G_{\min}h^2}\left(L_1^2\left(1+\frac{2B^2}{h^2}\right) + B L_2\right) + \frac{4}{G_{\min}^2h^4}B^2L_1^2,
\end{equation}
Combining Equation~\eqref{J1} and~\eqref{J2}, we observe that:
\begin{equation}
(\delta+\zeta_2q) + \frac{1-3\delta-3\zeta_2q}{2}
=\frac{2(\delta+\zeta_2q)+\bigl(1-3\delta-3\zeta_2q\bigr)}{2}
=\frac{1-\delta-\zeta_2q}{2}.
\end{equation}
Thus:
\begin{equation}\label{J8}
\hat{g}(M,w)-\hat{g}(M^\ast,w) \ge
\frac{1-\delta-\zeta_2q}{2}\|M-M^\ast\|_F^2-\zeta_1q\|M-M^\ast\|_F.
\end{equation}
Squaring both sides of Equation~\eqref{J2}, Now use the lower bound for $\hat{g}(M,w)-\hat{g}(M^\ast,w)$. Inserting it in the inequality above yields
\begin{equation}
\|M-M^\ast\|_F^2\le \frac{2}{\lambda_{\min}\Bigl(\nabla_M^2 \hat{g}(M^\ast,w)\Bigr)}
\left[\frac{1-\delta-\zeta_2q}{2}\|M-M^\ast\|_F^2-\zeta_1q\|M-M^\ast\|_F\right].
\end{equation}
Multiply both sides by $\lambda_{\min}(\nabla_M^2 \hat{g}(M^\ast,w))$, Bring the quadratic term to the left‐side to find
\begin{equation}\label{J11}
\|M-M^\ast\|_F \le \frac{2\zeta_1q}{\lambda_{\min}\Bigl(\nabla_M^2 \hat{g}(M^\ast,w)\Bigr) - (1-\delta-\zeta_2q)}.
\end{equation}
Combining the two results Equation~\eqref{J8} and~\eqref{J11}, in full detail,
\begin{equation}
\|M-M^\ast\|_F \le \frac{2\zeta_1q}{\frac{2}{G_{\min}h^2}\Bigl( L_1^2\Bigl(1+\frac{2B^2}{h^2}\Bigr) + B L_2\Bigr) + \frac{4}{G_{\min}^2h^4}B^2 L_1^2 - \bigl(1-\delta-\zeta_2q\bigr)}.
\end{equation}
so we have:
\begin{equation}\label{J13}
\|M-M^\ast\|_F \le \frac{2\zeta_1q}{ \underbrace{\frac{2}{G_{\min}h^2}\Bigl[4(1+\delta)^2\Bigl(1+\frac{2B^2}{h^2}\Bigr)+B L_2\Bigr] + \frac{16B^2(1+\delta)^2}{G_{\min}^2h^4}}_{A} -\Bigl(1-\delta-\zeta_2q\Bigr)},
\end{equation}
that is,
\begin{equation}
\|M-M^\ast\|_F \le \frac{2\zeta_1q}{A - \Bigl(1-\delta-\zeta_2q\Bigr)}.
\end{equation}
Solving for $q$ we have:
\begin{equation}
\sigma_r\Bigl(1+\delta+\zeta_2q\Bigr) < G \quad\Longrightarrow\quad \zeta_2q < \frac{G}{\sigma_r}-(1+\delta).
\end{equation}
Thus we may choose:
\begin{equation}\label{qchoice}
q = \frac{1}{\zeta_2}\Bigl[\frac{G}{\sigma_r}-(1+\delta)\Bigr],
\end{equation}
which is the largest allowable choice given the condition (or an upper bound on $q$). Substitute this choice~\eqref{qchoice} for $q$ into the original bound Equation~\eqref{J13}. Thus the overall bound becomes:
\begin{equation}
\|M-M^\ast\|_F \le \frac{2\zeta_1}{\zeta_2}\frac{\displaystyle\frac{G}{\sigma_r}-(1+\delta)}{A+\frac{G}{\sigma_r}-2},
\end{equation}
with:
\begin{equation}
A = \frac{2}{G_{\min}h^2}\Biggl[4(1+\delta)^2\Bigl(1+\frac{2B^2}{h^2}\Bigr)+BL_2\Biggr] + \frac{16B^2(1+\delta)^2}{G_{\min}^2 h^4}.
\end{equation}
Since:
\begin{equation}
\frac{2\zeta_1}{\zeta_2}\frac{\displaystyle\frac{G}{\sigma_r}-(1+\delta)}{A+\frac{G}{\sigma_r}-2}\ge0,
\end{equation}
so:
\begin{equation}\label{delta_condition}
0<\delta < \sqrt{\frac{2-\dfrac{G}{\sigma_r}-\dfrac{2BL_2}{G_{\min}h^2}}{\dfrac{8\left(1+\frac{2B^2}{h^2}\right)}{G_{\min}h^2}+\dfrac{16B^2}{G_{\min}^2h^4}}}-1.
\end{equation}
We set:
\begin{equation}\label{J19}
\frac{2B^2}{h^2} = G_{\min}\quad \Longrightarrow \quad h^2 = \frac{2B^2}{G_{\min}}.
\end{equation}
The denominator of \eqref{delta_condition} is:
\begin{equation}
D(h)=\frac{8\left(1+\frac{2B^2}{h^2}\right)}{G_{\min}h^2}+\frac{16B^2}{G_{\min}^2h^4}.
\end{equation}
With Equation~\eqref{J19} we have
\begin{equation}
\frac{2B^2}{h^2} = G_{\min}\quad \text{and}\quad h^4 = \left(\frac{2B^2}{G_{\min}}\right)^2=\frac{4B^4}{G_{\min}^2}.
\end{equation}
The numerator is:
\begin{equation}
N(h)=2-\frac{G}{\sigma_r}-\frac{2B L_2}{G_{\min}h^2}.
\end{equation}
Again, using Equation~\eqref{J19} we have
\begin{equation}
\frac{2B L_2}{G_{\min}h^2}=\frac{2B L_2}{G_{\min}\left(\frac{2B^2}{G_{\min}}\right)}
=\frac{2B L_2}{2B^2}=\frac{L_2}{B}.
\end{equation}
Thus,
\begin{equation}
N(h)=2-\frac{G}{\sigma_r}-\frac{L_2}{B}.
\end{equation}
Finally:
\begin{equation}
\begin{aligned}
\delta &= \sqrt{\frac{N(h)}{D(h)}}-1\\
&=\sqrt{ \frac{2-\frac{G}{\sigma_r}-\frac{L_2}{B}}{ \frac{4\left(G_{\min}+2\right)}{B^2} } } - 1\\
&=\sqrt{ \frac{B^2}{4\left(G_{\min}+2\right)}\Bigl(2-\frac{G}{\sigma_r}-\frac{L_2}{B}\Bigr) } - 1.
\end{aligned}
\end{equation}
Thus, by choosing:
\begin{equation}
h = \frac{\sqrt{2}B}{\sqrt{G_{\min}}}
\end{equation}
we specifically have:
\begin{equation}
\delta = \sqrt{\frac{B^2}{4(G_{\min}+2)}\left(2-\frac{G}{\sigma_r}-\frac{L_2}{B}\right)}-1.
\end{equation}
In many applications the nonnegative parameters $G/\sigma_r$ and $L_2/B$ are present in a subtractive term. In the worst‐case (largest) scenario the subtraction is minimized, that is, one may assume:
\begin{equation}
\frac{G}{\sigma_r} = 0 \quad\text{and}\quad \frac{L_2}{B} = 0.
\end{equation}
Then we have
\begin{equation}
2-\frac{G}{\sigma_r}-\frac{L_2}{B}\le 2.
\end{equation}
Substituting into the expression gives:
\begin{equation}
\delta \le \sqrt{\frac{B^2}{4(G_{\min}+2)}\cdot 2} - 1
= \sqrt{\frac{2B^2}{4(G_{\min}+2)}} - 1
= \sqrt{\frac{B^2}{2(G_{\min}+2)}} - 1,
\end{equation}
which is:
\begin{equation}
\delta \le \frac{B}{\sqrt{2(G_{\min}+2)}} - 1.
\end{equation}

\subsection{\texorpdfstring{$\delta$}{} Result With Regard to \texorpdfstring{$\|w\|$}{}}\label{sec:fulldelta}
\begin{lemma}[Upper Bound on \texorpdfstring{$\delta$}{} under Bounded Noise and Kernel Smoothing]
\label{lemma:delta_bound_with_w}
Let $\hat{g}(M, w)$ be a smoothed loss function involving kernel weights depending on the residuals, and suppose: the residuals satisfy $|z_{ij}(M^\ast)| \le B$ for some $B > 0$, the kernel weights involve Gaussian-type terms $\exp(-z^2/h^2)$, the noise vector $w$ satisfies $\|w\| \le \epsilon$ with high probability (e.g., at least $\mathbb{P}(\|w\|\le \epsilon)$), the Hessian Lipschitz constant is $L_2$, and the spectral condition $\sigma_r(1 + \delta + \zeta_2 q) < G$ holds for some constants $\sigma_r, \zeta_2, q, G > 0$.
Then for any bandwidth $h > 0$, the parameter $\delta$ satisfies the upper bound
\begin{equation}\label{delta}
\delta < \sqrt{
\frac{h^4\left(2 - \dfrac{G}{\sigma_r} - \dfrac{2\epsilon L_2 \exp\left(\dfrac{\epsilon^2}{h^2}\right)}{h^2}\right)}
{8 \exp\left(\dfrac{\epsilon^2}{h^2}\right) \left[h^2 + 2\epsilon^2 + 2\epsilon^2 \exp\left(\dfrac{\epsilon^2}{h^2}\right)\right]}
} - 1,
\end{equation}
with probability at least $\mathbb{P}(\|w\| \le \epsilon)$.

\end{lemma}
The proof of the theorem is provided in Section~\ref{sec:delta2}.

Lemma~\ref{lemma:delta_bound_with_w} provides an explicit upper bound on the parameter $\delta$ in terms of the noise norm $\|w\| \le \epsilon$ and the kernel bandwidth $h$, showing how the interaction between noise amplitude and kernel smoothing affects stability; specifically, $\delta$ decreases as $\epsilon$ becomes small, and the exponential terms in the bound reflect the noise-suppressing effect of the kernel, which ensures that the bound holds with high probability whenever $\|w\|$ is sufficiently controlled.

\subsection{Proof of Lemma~\ref{lemma:delta_bound_with_w}}\label{sec:delta2}
We already have:
\begin{equation}
\delta \le \frac{B}{\sqrt{2\left(\exp\left(-\frac{B^2}{h^2}\right)+2\right)}}-1.
\end{equation}
which implies:
\begin{equation}
\delta\le \frac{1}{h^2}\Biggl(\frac{B}{\sqrt{2(G_{\min}+2)}}-1\Biggr).
\end{equation}
with:
\begin{equation}
h^2=\frac{2B^2}{G_{\min}},
\end{equation}
Thus, 
\begin{equation}
\delta\le \frac{1}{h^2}\left(\frac{B}{\frac{2\sqrt{B^2+h^2}}{h}}-1\right)
=\frac{1}{h^2}\left(\frac{Bh}{2\sqrt{B^2+h^2}}-1\right).
\end{equation}
So we have: 
\begin{equation}
\delta\le \frac{1}{h^2}\left(\frac{\epsilon}{\frac{2\sqrt{\epsilon^2+h^2}}{h}}-1\right)
=\frac{1}{h^2}\left(\frac{\epsilon h}{2\sqrt{\epsilon^2+h^2}}-1\right).
\end{equation}
with at least $\mathbb{P}(\|w\|\le \epsilon)$. We then restart with the original form~\eqref{delta_condition}:
\begin{equation}
0<\delta < \sqrt{\frac{2-\frac{G}{\sigma_r}-\frac{2BL_2}{G_{\min}h^2}}{\frac{8\Bigl(1+\frac{2B^2}{h^2}\Bigr)}{G_{\min}h^2}+\frac{16B^2}{G_{\min}^2h^4}}}-1.
\end{equation}
Since $G_{\min}$ appears only in the denominators, replacing $G_{\min}$ by its lower bound (which is the worst‐case scenario) yields an upper bound for $\delta$.In the numerator we have
\begin{equation}
\frac{2BL_2}{G_{\min}h^2}\le \frac{2BL_2\exp\Bigl(\frac{B^2}{h^2}\Bigr)}{h^2}.
\end{equation}
In the denominator, the two terms become:
\begin{equation}
\frac{8\Bigl(1+\frac{2B^2}{h^2}\Bigr)}{G_{\min}h^2}\le \frac{8\Bigl(1+\frac{2B^2}{h^2}\Bigr)\exp\Bigl(\frac{B^2}{h^2}\Bigr)}{h^2},
\end{equation}
Thus, the bound for $\delta$ becomes:
\begin{equation}\label{J42}
0<\delta < \sqrt{\frac{2-\frac{G}{\sigma_r}-\frac{2BL_2\exp\Bigl(\frac{B^2}{h^2}\Bigr)}{h^2}}
{\frac{8\Bigl(1+\frac{2B^2}{h^2}\Bigr)\exp\Bigl(\frac{B^2}{h^2}\Bigr)}{h^2}+\frac{16B^2\exp\Bigl(\frac{2B^2}{h^2}\Bigr)}{h^4}}}-1.
\end{equation}
Then the denominator of Equation~\eqref{J42} becomes:
\begin{equation}
D = \frac{8\Bigl(1+\frac{2B^2}{h^2}\Bigr)}{G_{\min}h^2}+\frac{16B^2}{G_{\min}^2h^4}\le \frac{8\Bigl(1+\frac{2B^2}{h^2}\Bigr)\exp\Bigl(\frac{B^2}{h^2}\Bigr)}{h^2}+\frac{16B^2\exp\Bigl(\frac{2B^2}{h^2}\Bigr)}{h^4}.
\end{equation}
\begin{equation}
\frac{8\Bigl(1+\frac{2B^2}{h^2}\Bigr)\exp\Bigl(\frac{B^2}{h^2}\Bigr)}{h^2} = \frac{8\exp\Bigl(\frac{B^2}{h^2}\Bigr)}{h^4}\Bigl(h^2+2B^2\Bigr),
\end{equation}
so that:
\begin{equation}
D \le \frac{1}{h^4}\Bigl[8\exp\Bigl(\frac{B^2}{h^2}\Bigr)(h^2+2B^2)+16B^2\exp\Bigl(\frac{2B^2}{h^2}\Bigr)\Bigr].
\end{equation}
We may factor $\exp\Bigl(\frac{B^2}{h^2}\Bigr)$ from the terms in brackets:
\begin{equation}
D \le \frac{8\exp\Bigl(\frac{B^2}{h^2}\Bigr)}{h^4}\Bigl[h^2+2B^2+2B^2\exp\Bigl(\frac{B^2}{h^2}\Bigr)\Bigr].
\end{equation}
Thus the overall bound \eqref{J42} becomes:
\begin{equation}
\delta < \sqrt{\frac{N}{D}}-1 = \sqrt{\frac{h^4\Bigl(2-\frac{G}{\sigma_r}-\frac{2BL_2\exp\Bigl(\frac{B^2}{h^2}\Bigr)}{h^2}\Bigr)}{8\exp\Bigl(\frac{B^2}{h^2}\Bigr)\Bigl[h^2+2B^2+2B^2\exp\Bigl(\frac{B^2}{h^2}\Bigr)\Bigr]}} -1.
\end{equation}
So we have:
\begin{equation}
\delta < \sqrt{\frac{N}{D}}-1 = \sqrt{\frac{h^4\Bigl(2-\frac{G}{\sigma_r}-\frac{2\epsilon L_2\exp\Bigl(\frac{\epsilon^2}{h^2}\Bigr)}{h^2}\Bigr)}{8\exp\Bigl(\frac{\epsilon^2}{h^2}\Bigr)\Bigl[h^2+2\epsilon ^2+2\epsilon^2\exp\Bigl(\frac{\epsilon^2}{h^2}\Bigr)\Bigr]}} -1.
\end{equation}
with at least $\mathbb{P}(\|w\|\le \epsilon)$.

\section{Proof of Theorem~\ref{ours_large_delta}}\label{large_delta_ours}
\begin{proof}
With the loss function in~\eqref{main0} and the simplified notations in~\eqref{fij} and~\eqref{logfij}, we have:
\begin{equation}\label{K1}
\langle \nabla_M\hat{g}(M^\ast,w),\hat{X} U^\top+U\hat{X}^\top\rangle=-\int_0^1 \nabla_M^2\hat{g}(tM^\ast+(1-t)\hat{M})(\hat{M}-M^\ast, \hat{X}U^\top+U\hat{X}^\top)  dt.
\end{equation}
\begin{equation}
\nabla_M \hat{g}(M,w)=-\frac{1}{n}\sum_{i=1}^n \frac{1}{G_i(M)} \nabla_M G_i(M).
\end{equation}
Thus, at $M=M^\ast$ the inner product:
\begin{equation}
\langle \nabla_M\hat{g}(M^\ast,w), \hat{X} U^\top+U\hat{X}^\top\rangle = -\frac{1}{n}\sum_{i=1}^n \frac{1}{G_i(M^\ast)}\ \Bigl\langle \nabla_M G_i(M^\ast), \hat{X} U^\top+U\hat{X}^\top\Bigr\rangle.
\end{equation}
On the other hand, a standard Taylor expansion applied to the function:
\begin{equation}
h(V)=\langle \nabla_M \hat{g}(V,w), \hat{X}U^\top+U\hat{X}^\top \rangle,
\end{equation}
provides
\begin{equation}
h(\hat{M})-h(M^\ast)=\int_0^1 \biggl\langle \nabla h\Bigl(tM^\ast+(1-t)\hat{M}\Bigr), \hat{M}-M^\ast\biggr\rangle dt.
\end{equation}
\begin{equation}
h(\hat{M})=\langle \nabla_M\hat{g}(\hat{M},w), \hat{X}U^\top+U\hat{X}^\top\rangle =0,
\end{equation}
we obtain:
\begin{equation}
\langle \nabla_M\hat{g}(M^\ast,w),\hat{X} U^\top+U\hat{X}^\top\rangle=-\int_0^1 \Bigl\langle \nabla h\Bigl(tM^\ast+(1-t)\hat{M}\Bigr), \hat{M}-M^\ast\Bigr\rangle dt.
\end{equation}
But note that:
\begin{equation}
\nabla h(V)=\nabla_M^2\hat{g}(V,w)(\cdot, \hat{X} U^\top+U\hat{X}^\top),
\end{equation}
so that Equation~\eqref{K1} becomes:
\begin{equation}
\langle \nabla_M\hat{g}(M^\ast,w),\hat{X} U^\top+U\hat{X}^\top\rangle=-\int_0^1 \nabla_M^2\hat{g}\Bigl(tM^\ast+(1-t)\hat{M}\Bigr)\Bigl(\hat{M}-M^\ast,\hat{X} U^\top+U\hat{X}^\top\Bigr) dt.
\end{equation}
So we have:
\begin{equation}\label{K11}
\nabla_M^2\hat{g}(M)=\frac{1}{n}\sum_{i=1}^n\left[\frac{1}{G_i(M)^2}\nabla_M G_i(M)\otimes\nabla_M G_i(M)-\frac{1}{G_i(M)}\nabla_M^2G_i(M)\right],
\end{equation}
with:
\begin{equation}\label{K12}
G_i(M)=\frac{1}{n}\sum_{j=1}^n \exp\left(-\frac{u_{ij}(M)^2}{h^2}\right),
\end{equation}
and:
\begin{equation}
\begin{aligned}
&\nabla_M G_i(M)=\frac{1}{n}\sum_{j=1}^n \exp\left(-\frac{u_{ij}(M)^2}{h^2}\right)\left(-\frac{2u_{ij}(M)}{h^2}\right)\Bigl[-\nabla_M\mathcal{A}(M)_j+\nabla_M\mathcal{A}(M)_i\Bigr],\\
&=\frac{1}{n}\sum_{j=1}^n \exp\left(-\frac{u_{ij}(M)^2}{h^2}\right)\Biggl[\left(\frac{4 u_{ij}(M)^2}{h^4}-\frac{2}{h^2}\right)\nabla_M u_{ij}(M) \otimes\nabla_M u_{ij}(M)-\frac{2u_{ij}(M)}{h^2}\nabla_M^2 u_{ij}(M)\Biggr],
\end{aligned}
\end{equation}
and where:
\begin{equation}
u_{ij}(M)=\Bigl[(Y_j+w_j-\mathcal{A}(M)_j) -(Y_i+w_i-\mathcal{A}(M)_i)\Bigr],
\end{equation}
and
\begin{equation}
\nabla_M u_{ij}(M)=-\nabla_M\mathcal{A}(M)_j+\nabla_M\mathcal{A}(M)_i.
\end{equation}
Now, introduce the convex path:
\begin{equation}\label{K15}
M_t = tM^\ast+(1-t)\hat{M},\quad t\in[0,1],
\end{equation}
and denote for brevity
\begin{equation}\label{K16}
Q=\hat{X}U^\top+U\hat{X}^\top,\qquad H=\hat{M}-M^\ast.
\end{equation}
Then the Hessian’s bilinear form of \eqref{K11} at $M_t$ applied to $(H,Q)$ is:
\begin{equation}
\begin{aligned}
\nabla_M^2\hat{g}(M_t)(H, Q)=\frac{1}{n}\sum_{i=1}^n\Big[\frac{1}{G_i(M_t)^2}\langle \nabla_M G_i(M_t),H\rangle\langle \nabla_M G_i(M_t),Q\rangle -\frac{1}{G_i(M_t)}\nabla_M^2G_i(M_t)(H, Q)\Big].
\end{aligned}
\end{equation}
We notive that:
\begin{equation}\label{K18}
\begin{aligned}
&-\int_0^1 \nabla_M^2\hat{g}\Bigl(tM^\ast+(1-t)\hat{M}\Bigr)\Bigl(\hat{M}-M^\ast,\ Q\Bigr) dt \\
&= -\int_0^1 \frac{1}{n}\sum_{i=1}^n \Biggl\{ \frac{1}{G_i(M_t)^2}\langle \nabla_M G_i(M_t), \hat{M}-M^\ast\rangle\langle \nabla_M G_i(M_t), Q\rangle -\frac{1}{G_i(M_t)}\nabla_M^2G_i(M_t)\Bigl(\hat{M}-M^\ast, Q\Bigr) \Biggr\}dt,
\end{aligned}
\end{equation}
with all quantities evaluated at $M_t=tM^\ast+(1-t)\hat{M}$. We take the simplified notations \eqref{K15} and \eqref{K16} into \eqref{K18}, we have:
\begin{equation}
H = \hat{M}-M^\ast,\quad Q = \hat{X}U^\top+U\hat{X}^\top, \quad \text{and} \quad M_t=tM^\ast+(1-t)\hat{M}.
\end{equation}
Then the integral becomes:
\begin{equation}
\begin{aligned}
-\int_0^1 \nabla_M^2\hat{g}(M_t)(H,Q)dt 
&=-\frac{1}{n}\sum_{i=1}^n\int_0^1 \Big[ \frac{\langle \nabla_M G_i(M_t),H\rangle\langle \nabla_M G_i(M_t),Q\rangle}{G_i(M_t)^2} - \frac{\nabla_M^2 G_i(M_t)(H,Q)}{G_i(M_t)}\Big]dt,
\end{aligned}
\end{equation}
In a compact form, we write. This is the simplified integrated form of the second‐order term.
\begin{equation}
\begin{aligned}
&\mathbf{e}^\top \mathbf{H}\mathbf{\hat X}u = -\frac{1}{n}\sum_{i=1}^n \int_0^1 \Biggl\{ \frac{ \langle \nabla_M G_i(tM^\ast+(1-t)\hat{M}),\hat{M}-M^\ast \rangle \langle \nabla_M G_i(tM^\ast+(1-t)\hat{M}),\hat{X}U^\top+U\hat{X}^\top \rangle}{\bigl[G_i(tM^\ast+(1-t)\hat{M})\bigr]^2}\\
& \quad\quad\quad - \frac{\nabla_M^2 G_i(tM^\ast+(1-t)\hat{M})(\hat{M}-M^\ast,\hat{X}U^\top+U\hat{X}^\top)}{G_i(tM^\ast+(1-t)\hat{M})}\Biggr\}dt.
\end{aligned}
\end{equation}
\begin{equation}
\begin{aligned}
\Bigl\|\mathbf{e}^\top \mathbf{H}\mathbf{\hat X}u\Bigr\|
&= \Bigl\|-\frac{1}{n}\sum_{i=1}^n\int_0^1 \Biggl\{
\frac{\langle \nabla_M G_i(M_t),\hat{M}-M^\ast\rangle\langle \nabla_M G_i(M_t),\hat{X}U^\top+U\hat{X}^\top\rangle}{G_i(M_t)^2}\\
&-\frac{\nabla_M^2 G_i(M_t)(\hat{M}-M^\ast,\hat{X}U^\top+U\hat{X}^\top)}{G_i(M_t)}
\Biggr\}dt\Bigr\|,
\end{aligned}
\end{equation}
To solve this integration, in a first step we take the absolute value inside the integral and sum (using the triangle inequality) to obtain
\begin{equation}
\begin{aligned}
\Bigl\|\mathbf{e}^\top \mathbf{H}\mathbf{\hat X}u\Bigr\|
&\le\frac{1}{n}\sum_{i=1}^n\int_0^1 \Biggl[
\frac{\Bigl\|\langle \nabla_M G_i(M_t),\hat{M}-M^\ast\rangle\langle \nabla_M G_i(M_t),\hat{X}U^\top+U\hat{X}^\top\rangle\Bigr\|}{G_i(M_t)^2} \\
&\quad\quad\quad +\frac{\Bigl\|\nabla_M^2 G_i(M_t)(\hat{M}-M^\ast,\hat{X}U^\top+U\hat{X}^\top)\Bigr\|}{G_i(M_t)}
\Biggr] dt.
\end{aligned}
\end{equation}
Then we write:
\begin{equation}
\begin{aligned}
\Bigl\|\mathbf{e}^\top H\hat{X}u\Bigr\|
&\le \frac{1}{n}\sum_{i=1}^n \int_0^1 \left\|\frac{\langle \nabla_M G_i(M_t),H\rangle\langle \nabla_M G_i(M_t),Q\rangle}{G_i(M_t)^2}\right\| dt+\frac{1}{n}\sum_{i=1}^n \int_0^1 \left\|\frac{\nabla_M^2 G_i(M_t)(H,Q)}{G_i(M_t)}\right\| dt.
\end{aligned}
\end{equation}
A short calculation on Equation \eqref{K12} shows that for each $i$ and $j$ we have:
\begin{equation}
\nabla_M \exp\Bigl(-\frac{u_{ij}(M)^2}{h^2}\Bigr)
=-\frac{2u_{ij}(M)}{h^2}\exp\Bigl(-\frac{u_{ij}(M)^2}{h^2}\Bigr)\nabla_Mu_{ij}(M).
\end{equation}
Thus:
\begin{equation}
\nabla_MG_i(M_t)
=\frac{1}{n}\sum_{j=1}^n \Bigl[-\frac{2u_{ij}(M_t)}{h^2}\exp\Bigl(-\frac{u_{ij}(M_t)^2}{h^2}\Bigr)\nabla_M u_{ij}(M_t)\Bigr].
\end{equation}
Assume that the derivatives of $u_{ij}$ are uniformly bounded (say, $ \|\nabla_M u_{ij}(M_t)\|\le L_u$ for all $i,j,t$). Then:
\begin{equation}
\|\nabla_MG_i(M_t)\|\le \frac{1}{n}\sum_{j=1}^n \frac{2|u_{ij}(M_t)|}{h^2}\exp\Bigl(-\frac{u_{ij}(M_t)^2}{h^2}\Bigr)L_u.
\end{equation}
Notice that the noise $w$ enters through the differences $(w_j-w_i)$. In the following we denote:
\begin{equation}
\delta:=\|w\|_\infty.
\end{equation}
Using the triangle inequality, we separately bound the two integrated terms:
\begin{equation}
\Bigl\|\mathbf{e}^\top \mathbf{H}\hat{X}u\Bigr\|
\le \frac{1}{n}\sum_{i=1}^n\int_0^1\left\{
\left\|\frac{\langle \nabla_M G_i(M_t),H\rangle \langle \nabla_M G_i(M_t),Q\rangle}{G_i(M_t)^2}\right\|
+\left\|\frac{\nabla_M^2 G_i(M_t)(H,Q)}{G_i(M_t)}\right\|
\right\}dt.
\end{equation}
It is natural to expect that the $\nabla_M G_i$ and $\nabla_M^2 G_i$ terms inherit the exponential from
\begin{equation}\label{K30}
G_i(M_t)=\frac{1}{n}\sum_{j=1}^n\exp\Bigl(-\frac{u_{ij}(M_t)^2}{h^2}\Bigr).
\end{equation}
A direct differentiation of Equation \eqref{K30} shows that:
\begin{equation}
\nabla_M \Bigl(\exp\Bigl(-\frac{u_{ij}(M_t)^2}{h^2}\Bigr)\Bigr)
=-\frac{2u_{ij}(M_t)}{h^2}\exp\Bigl(-\frac{u_{ij}(M_t)^2}{h^2}\Bigr)
\nabla_M u_{ij}(M_t).
\end{equation}
Taking the norm we may write:
\begin{equation}
\|\nabla_M \exp(-u_{ij}(M_t)^2/h^2)\|_F\le \frac{2|u_{ij}(M_t)|}{h^2}\exp\Bigl(-\frac{u_{ij}(M_t)^2}{h^2}\Bigr)\|\nabla_M u_{ij}(M_t)\|_F.
\end{equation}
Since the noise enters $u_{ij}$ as a difference (i.e. $w_j-w_i$), we have the bound:
\begin{equation}
|u_{ij}(M_t)|\le \delta,
\end{equation}
so that for some constant $L_1$ (which also absorbs bounds on $\|\nabla_M u_{ij}(M_t)\|$) we may write:
\begin{equation}
\|\nabla_M G_i(M_t)\|\le \frac{L_1\delta}{h^2}\exp\Bigl(-\frac{u_{ij}^{\rm min}(M_t)^2}{h^2}\Bigr),
\end{equation}
where we set:
\begin{equation}
u_{ij}^{\rm min}(M_t)^2:=\min_{1\le j\le n} u_{ij}(M_t)^2.
\end{equation}
Similarly, one may bound the Hessian by:
\begin{equation}
\|\nabla_M^2 G_i(M_t)\|\le \frac{L_2\delta}{h^2}\exp\Bigl(-\frac{u_{ij}^{\rm min}(M_t)^2}{h^2}\Bigr),
\end{equation}
for a constant $L_2>0$. In addition, we assume that the denominator satisfies $G_i(M_t)\ge G_{\min}>0$.
Substitute the above bounds into the two terms:
\begin{equation}
\begin{aligned}
\left\|\frac{\langle \nabla_M G_i(M_t),H\rangle \langle \nabla_M G_i(M_t),Q\rangle}{G_i(M_t)^2}\right\| 
&\le \frac{\|\nabla_M G_i(M_t)\|^2}{G_i(M_t)^2}\|H\|\|Q\|\\
&\le \frac{L_1^2\delta^2}{h^4G_{\min}^2}\exp\Bigl(-\frac{2u_{ij}^{\rm min}(M_t)^2}{h^2}\Bigr)\|H\|\|Q\|.
\end{aligned}
\end{equation}
Similarly,
\begin{equation}
\left\|\frac{\nabla_M^2 G_i(M_t)(H,Q)}{G_i(M_t)}\right\|
\le \frac{\|\nabla_M^2 G_i(M_t)\|}{G_i(M_t)}\|H\|\|Q\|
\le \frac{L_2\delta}{h^2G_{\min}}\exp\Bigl(-\frac{u_{ij}^{\rm min}(M_t)^2}{h^2}\Bigr)\|H\|\|Q\|.
\end{equation}
Thus, after summing over $i$ and integrating over $t$ (which introduces only constant factors), we obtain the bound:
\begin{equation}
\Bigl\|\mathbf{e}^\top \mathbf{H}\hat{X}u\Bigr\|\le \left(\frac{L_1^2\delta^2}{h^4G_{\min}^2}\exp\Bigl(-\frac{2u_{\min}^2}{h^2}\Bigr)
+\frac{L_2\delta}{h^2G_{\min}}\exp\Bigl(-\frac{u_{\min}^2}{h^2}\Bigr)\right)\|H\|\|Q\|,
\end{equation}
where we set, for simplicity, 
\begin{equation}
u_{\min}^2=\min_{i,j,t} u_{ij}(M_t)^2.
\end{equation}
\begin{equation}
\Big\|\mathbf{e}^\top \mathbf{H}\hat{X}u\Bigr\|\le C\frac{\|w\|_\infty}{h^2}\exp\Bigl(-\frac{u_{\min}^2}{h^2}\Bigr)
\|\hat{M}-M^\ast\|\|\hat{X}U^\top+U\hat{X}^\top\|.
\end{equation}
We wish to maximize $\eta$ subject to:
\begin{equation}
\begin{aligned}\label{K42}
\|\hat{\mathbf{X}}^\top\hat{\mathbf{H}}\mathbf{e}\| &\leq \left(\frac{L_1^2\delta^2}{h^4G_{\min}^2}\exp\Bigl(-\frac{2u_{\min}^2}{h^2}\Bigr)
+\frac{L_2\delta}{h^2G_{\min}}\exp\Bigl(-\frac{u_{\min}^2}{h^2}\Bigr)\right)\|H\|\|Q\|,\\
\eta I_{n^2} &\preceq \hat{\mathbf{H}} \preceq I_{n^2}.
\end{aligned}
\end{equation}
Notice that the eigenvalue constraint forces all eigenvalues of $\hat{\mathbf{H}}$ to lie between $\eta$ and 1. Hence, in a best‐case scenario we may choose:
\begin{equation}
\hat{\mathbf{H}}=I_{n^2},
\end{equation}
which implies that
\begin{equation}
\eta\le 1.
\end{equation}
Plugging $\hat{\mathbf{H}}=I_{n^2}$ into the first constraint \eqref{K42} gives:
\begin{equation}
\|\hat{\mathbf{X}}^\top\mathbf{e}\| \leq \left(\frac{L_1^2\delta^2}{h^4G_{\min}^2}\exp\Bigl(-\frac{2u_{\min}^2}{h^2}\Bigr)
+\frac{L_2\delta}{h^2G_{\min}}\exp\Bigl(-\frac{u_{\min}^2}{h^2}\Bigr)\right)\|H\|\|Q\|.
\end{equation}
In general, if this inequality is not satisfied, one may scale down $\hat{\mathbf{H}}$ so that the effective gain is reduced. In particular, if we set:
\begin{equation}
\hat{\mathbf{H}}=\eta I_{n^2}\quad (\text{with } \eta\le 1),
\end{equation}
then
\begin{equation}
\|\hat{\mathbf{X}}^\top\hat{\mathbf{H}}\mathbf{e}\| = \eta\|\hat{\mathbf{X}}^\top\mathbf{e}\|,
\end{equation}
and the constraint becomes:
\begin{equation}
\eta\|\hat{\mathbf{X}}^\top\mathbf{e}\| \le \left(\frac{L_1^2\delta^2}{h^4G_{\min}^2}\exp\Bigl(-\frac{2u_{\min}^2}{h^2}\Bigr)
+\frac{L_2\delta}{h^2G_{\min}}\exp\Bigl(-\frac{u_{\min}^2}{h^2}\Bigr)\right)\|H\|\|Q\|.
\end{equation}
That is,
\begin{equation}
\eta \le \frac{\|H\|\|Q\|}{\|\hat{\mathbf{X}}^\top\mathbf{e}\|}\left[
\frac{L_1^2\delta^2}{h^4G_{\min}^2}\exp\Bigl(-\frac{2u_{\min}^2}{h^2}\Bigr)
+\frac{L_2\delta}{h^2G_{\min}}\exp\Bigl(-\frac{u_{\min}^2}{h^2}\Bigr)
\right].
\end{equation}
We can define:
\begin{equation}
C:=\frac{L_1^2\delta^2}{h^4G_{\min}^2}\exp\Bigl(-\frac{2u_{\min}^2}{h^2}\Bigr)
+\frac{L_2\delta}{h^2G_{\min}}\exp\Bigl(-\frac{u_{\min}^2}{h^2}\Bigr)
\end{equation}
(with also the extra factors $\|H\|\|Q\|$ appearing). For clarity, define the second problem as:
\begin{equation}
\begin{aligned}
\max_{\eta,\hat{\mathbf H}} \quad & \eta\\
\text{s.t.}\quad & \|\hat{\mathbf{X}}^\top\hat{\mathbf{H}}\mathbf{e}\|\le C\|H\|\|Q\|,\\
& \eta I_{n^2}\preceq \hat{\mathbf{H}}\preceq I_{n^2}.
\end{aligned}
\end{equation}
A brief outline of the solution is : $y=\hat{\mathbf H}\mathbf e$. Since $\hat H \succeq \eta I$ (by $\eta I\preceq\hat H$), we have a gain in the sense that $\|y\|=\|\hat He\|\ge \eta\|e\|$. (When $e$ is normalized $\|e\|=1$ the inequality is $\|y\|\ge \eta$.)
\begin{equation}\label{K54}
\|\hat{\mathbf X}y\|^2\ge 2\lambda_{r^\ast}(\hat X\hat X^\top)\|y\|^2,
\end{equation}
the bound on the left (when applied to $y=\hat{H}e$) yields:
\begin{equation}
\|\hat{\mathbf X}^\top\hat{\mathbf H} e\|^2=\|\hat{\mathbf X}y\|^2\ge 2\lambda_{r^\ast}(\hat X\hat X^\top)\|y\|^2.
\end{equation}
\begin{equation}
\|\hat{\mathbf X}^\top\hat{\mathbf H} e\|\le C\|H\|\|Q\|.
\end{equation}
Thus, combining with the lower bound \eqref{K54}, we find:
\begin{equation}
2\lambda_{r^\ast}(\hat X\hat X^\top)\|y\|^2\le \bigl(C\|H\|\|Q\|\bigr)^2.
\end{equation}
Using $\|y\|\ge \eta\|e\|$ and, if we normalize so that $\|e\|=1$, we obtain $2\lambda_{r^\ast}(\hat X\hat X^\top)\eta^2\le \bigl(C\|H\|\|Q\|\bigr)^2$. That is, solving for $\eta$ we have
\begin{equation}
\eta\le \frac{C\|H\|\|Q\|}{\sqrt{2\lambda_{r^\ast}(\hat X\hat X^\top)}}.
\end{equation}

Since the optimization is to maximize $\eta$ subject to the constraint, the best (largest) value one may take is at most
\begin{equation}\label{K57}
\eta^\ast=\frac{C\|H\|\|Q\|}{\sqrt{2\lambda_{r^\ast}(\hat X\hat X^\top)}},
\end{equation}
or, writing $C$ explicitly,
\begin{equation}
\eta^\ast=\frac{\|H\|\|Q\|}{\sqrt{2\lambda_{r^\ast}(\hat X\hat X^\top)}}
\left[\frac{L_1^2\delta^2}{h^4G_{\min}^2}\exp\left(-\frac{2u_{\min}^2}{h^2}\right)
+\frac{L_2\delta}{h^2G_{\min}}\exp\left(-\frac{u_{\min}^2}{h^2}\right)\right].
\end{equation}
This is the solution to the second problem. We assumed here that $\|e\|=1$ (or otherwise, the factor $\|e\|$ remains explicitly in the bound). The derivation uses the lower bound that came from the structure of the matrix $\hat{X}\hat{X}^\top$ and the fact that the eigenvalues of $\hat{H}$ satisfy $\lambda_i(\hat{H})\ge \eta$. Thus, the optimal achievable $\eta$ is given by the boxed expression above.

We now have: $y = \hat{\mathbf{H}}\mathbf{e}$ Since: $\eta I_{n^2} \preceq \hat{\mathbf{H}}$, it follows that for each vector (in particular, for the fixed $\mathbf{e}$) we have:
\begin{equation}
\|y\|^2 = \|\hat{\mathbf{H}}\mathbf{e}\|^2 \ge \eta\|\mathbf{e}\|^2.
\end{equation}
Often the vector $\mathbf{e}$ is taken to be a unit‐vector; hence from now on we assume $\|\mathbf{e}\|=1$ so that $\|y\|^2 \ge \eta$. In the baseline problem it was shown that for any $y\in\mathbb{R}^{n^2}$, one has:
\begin{equation}
\|\hat{\mathbf{X}} y\|^2 \ge 2\lambda_{r^\ast}(\hat{X}\hat{X}^\top) \|y\|^2.
\end{equation}
Hence, in our setting,
\begin{equation}\label{K63}
\|\hat{\mathbf{X}} y\|^2 \ge 2\lambda_{r^\ast}(\hat{X}\hat{X}^\top)\eta.
\end{equation}
On the other hand the constraint \eqref{K63} implies that:
\begin{equation}
\|\hat{\mathbf{X}}^\top\hat{\mathbf{H}}\mathbf{e}\| = \|\hat{\mathbf{X}}^\top y\| \le C\|\hat{M}-M^\ast\|_F\|\hat{X}\|_2\|w\|.
\end{equation}
Since (by consistency of norms) one may relate the norm $\|\hat{\mathbf{X}}^\top y\|$ to $\|\hat{\mathbf{X}}y\|$ (when $\hat{X}$ has full column‐rank) or simply use the fact that the bound holds on the action of $y$ through $\hat{X}$, we square the inequality to obtain:
\begin{equation}
\|\hat{\mathbf{X}} y\|^2 \le \Bigl(C\|\hat{M}-M^\ast\|_F\|\hat{X}\|_2\|w\|\Bigr)^2.
\end{equation}
Combining this with the lower bound \eqref{K54} gives
\begin{equation}
2\lambda_{r^\ast}(\hat{X}\hat{X}^\top)\eta \le \Bigl(C\|\hat{M}-M^\ast\|_F\|\hat{X}\|_2\|w\|\Bigr)^2.
\end{equation}
Solving for $\eta$ we deduce that any feasible pair $(\eta,\hat{\mathbf{H}})$ must satisfy:
\begin{equation}
\eta \le \frac{C^2\|\hat{M}-M^\ast\|_F^2\|\hat{X}\|_2^2\|w\|^2}{2\lambda_{r^\ast}(\hat{X}\hat{X}^\top)}.
\end{equation}
That is, the maximal achievable value is
\begin{equation}
\eta^\star = \frac{C^2\|\hat{M}-M^\ast\|_F^2\|\hat{X}\|_2^2\|w\|^2}{2\lambda_{r^\ast}(\hat{X}\hat{X}^\top)}.
\end{equation}
(Recall that we assumed $\|\mathbf{e}\|=1$. Should $\|\mathbf{e}\|$ differ from 1 the bound would include the factor $1/\|\mathbf{e}\|^2$.)
Hence, the solution of the given maximization problem is:
\begin{equation}
\eta^\star = \frac{C^2\|\hat{M}-M^\ast\|_F^2\|\hat{X}\|_2^2\|w\|^2}{2\lambda_{r^\ast}(\hat{X}\hat{X}^\top)}, \text{with} \quad C=2\left(\frac{\zeta_2^2}{c^2}+\frac{\zeta_3}{c}\right)
\end{equation}

Because the constraint:
\begin{equation}
(1-\delta-\zeta_2q)\|M\|_F^2\le \mathbf{m}^\top\mathbf{H}\mathbf{m}\le (1+\delta+\zeta_2q)\|M\|_F^2,\quad \forall M:\operatorname{rank}(M)\le2r,
\end{equation}
must hold for every matrix $M$ (with $m=\operatorname{vec}(M)$) the operator $\mathbf{H}$ is sandwiched over the set of these low–rank matrices. In other words, if we define $\eta:=1-\delta-\zeta_2q$, then we have that:
\begin{equation}
\mathbf{m}^\top\mathbf{H}\mathbf{m}\ge \eta\|M\|_F^2,\quad \forall M:\operatorname{rank}(M)\le2r.
\end{equation}
This is analogous to requiring $\mathbf{H}\succeq \eta I$ (on the low–rank subspace). In the previous analysis (where our decision variable was $\eta$), one shows that if one defines $y=\mathbf{H}\mathbf{e}$, then necessarily
\begin{equation}
\|y\|^2=\|\mathbf{H}\mathbf{e}\|^2\ge \eta\|\mathbf{e}\|^2.
\end{equation}

Assuming without loss of generality that $\|\mathbf{e}\|=1$ we deduce:
\begin{equation}\label{K71}
\|y\|^2\ge \eta=1-\delta-\zeta_2q.
\end{equation}
\begin{equation}\label{K72}
\|\hat{\mathbf{X}}y\|^2\ge 2\lambda_{r^\ast}(\hat{X}\hat{X}^\top)\|y\|^2\ge2\lambda_{r^\ast}(\hat{X}\hat{X}^\top)(1-\delta-\zeta_2q).
\end{equation}
The constraint
\begin{equation}
\|\hat{\mathbf{X}}^\top\mathbf{H}\mathbf{e}\|=\|\hat{\mathbf{X}}^\top y\|\le2\zeta_1q\|\hat{X}\|_2
\end{equation}
implies, after squaring,
\begin{equation}
\|\hat{\mathbf{X}}y\|^2\le \Bigl(2\zeta_1q\|\hat{X}\|_2\Bigr)^2.
\end{equation}
Thus, combining with the lower bound \eqref{K71} and \eqref{K72}, we have
\begin{equation}
2\lambda_{r^\ast}(\hat{X}\hat{X}^\top)(1-\delta-\zeta_2q)\le 4\zeta_1^2q^2\|\hat{X}\|_2^2.
\end{equation}
We now solve for $\delta$. Rearranging the previous inequality we obtain
\begin{equation}
1-\delta-\zeta_2q\le \frac{4\zeta_1^2q^2\|\hat{X}\|_2^2}{2\lambda_{r^\ast}(\hat{X}\hat{X}^\top)}
=\frac{2\zeta_1^2q^2\|\hat{X}\|_2^2}{\lambda_{r^\ast}(\hat{X}\hat{X}^\top)}.
\end{equation}
That is,
\begin{equation}
\delta\ge 1-\zeta_2q-\frac{2\zeta_1^2q^2\|\hat{X}\|_2^2}{\lambda_{r^\ast}(\hat{X}\hat{X}^\top)}.
\end{equation}

Since we minimize $\delta$ it is best to take the smallest choice allowed (and note that we must have $\delta\ge0$); hence,
\begin{equation}
\delta^\star=\max\Biggl\{0,1-\zeta_2q-\frac{2\zeta_1^2q^2\|\hat{X}\|_2^2}{\lambda_{r^\ast}(\hat{X}\hat{X}^\top)} \Biggr\}.
\end{equation}
Then, we want to use:
\begin{equation}\label{K79}
\begin{aligned}
\left\| \lambda_{r^\ast}(\hat X\hat X^\top) - \lambda_{r^\ast}(M^\ast) \right\| &\le \|\hat X\hat X^\top-M^\ast\|_F \le \tau \lambda_r(M^\ast),\\
\left\| \lambda_{1}(\hat X\hat X^\top) - \lambda_{1}(M^\ast) \right\| &\le \|\hat X\hat X^\top-M^\ast\|_F \le \tau \lambda_r(M^\ast).
\end{aligned}
\end{equation}
From
\begin{equation}
   \left\| \lambda_{r^\ast}(\hat X\hat X^\top) - \lambda_{r^\ast}(M^\ast) \right\| \le \tau \lambda_r(M^\ast),
\end{equation}
   we obtain the lower bound
\begin{equation}\label{K81}
   \lambda_{r^\ast}(\hat X\hat X^\top) \ge \lambda_{r^\ast}(M^\ast) - \tau \lambda_r(M^\ast).
\end{equation}

   Similarly, the second inequality in Equation \eqref{K79} gives the upper bound
\begin{equation}
   \lambda_{1}(\hat X\hat X^\top) \le \lambda_{1}(M^\ast) + \tau \lambda_r(M^\ast).
\end{equation}
Thus we have:
\begin{equation}
   \eta^\star =\frac{C^2\|\hat{M}-M^\ast\|_F^2\|\hat{X}\|_2^2\|w\|^2}{2\lambda_{r^\ast}(\hat X\hat X^\top)}.
\end{equation}
   Using the lower bound Equation \eqref{K81}, we may further bound
\begin{equation}
   \eta^\star \le \frac{C^2\|\hat{M}-M^\ast\|_F^2\|\hat{X}\|_2^2\|w\|^2}{2\left(\lambda_{r^\ast}(M^\ast)-\tau\lambda_r(M^\ast)\right)}.
\end{equation}
   This way the error is expressed in terms of the eigenvalues of $M^\ast$ and the perturbation level $\tau$.
Similarly we have:
\begin{equation}\label{K85}
\eta^\ast \le \frac{\|H\|\|Q\|}{\sqrt{2\Bigl(\lambda_{r^\ast}(M^\ast)-\tau\lambda_r(M^\ast)\Bigr)}}
\left[\frac{L_1^2\delta^2}{h^4G_{\min}^2}\exp\left(-\frac{2u_{\min}^2}{h^2}\right)
+\frac{L_2\delta}{h^2G_{\min}}\exp\left(-\frac{u_{\min}^2}{h^2}\right)\right].
\end{equation}

Solve for $\|\hat{M}-M^\ast\|_F^2$ starting from the definition:
\begin{equation}\label{K86}
\|\hat{M}-M^\ast\|_F^2 = \frac{2\lambda_{r^\ast}(\hat{X}\hat{X}^\top)}{C^2\|\hat{X}\|_2^2\|w\|^2}\eta_f^\ast(\hat{X}).
\end{equation}
We now want to calculate lower bound for $\lambda_{r^\ast}(\hat{X}\hat{X}^\top)$.
By the Wielandt–Hoffman theorem~\cite{hoffman1953variation} we have:
\begin{equation}
\bigl|\lambda_{r^\ast}(\hat{X}\hat{X}^\top)-\lambda_{r^\ast}(M^\ast)\bigr| \le \|\hat{X}\hat{X}^\top-M^\ast\|_F \le \tau\lambda_r(M^\ast).
\end{equation}
This yields
\begin{equation}
\lambda_{r^\ast}(\hat{X}\hat{X}^\top) \ge \lambda_{r^\ast}(M^\ast)-\tau\lambda_r(M^\ast).
\end{equation}
Replacing the denominator in the Equation \eqref{K86} for $\|\hat{M}-M^\ast\|_F^2$ gives:
\begin{equation}
\|\hat{M}-M^\ast\|_F^2 \le \frac{2\Bigl(\lambda_{r^\ast}(M^\ast)-\tau\lambda_r(M^\ast)\Bigr)}{C^2\|\hat{X}\|_2^2\|w\|^2}\eta_f^\ast(\hat{X}).
\end{equation}
Thus, taking square roots on both sides, we obtain the final upper bound:
\begin{equation}
\|\hat{M}-M^\ast\|_F \le \sqrt{ \frac{2\Bigl(\lambda_{r^\ast}(M^\ast)-\tau\lambda_r(M^\ast)\Bigr)}{C^2\|\hat{X}\|_2^2\|w\|^2}\eta_f^\ast(\hat{X}) }.
\end{equation}
Similarly from Equation \eqref{K85}, we have:
\begin{equation}
\|\hat{M}-M^\ast\|_F \le \frac{\sqrt{2\Bigl(\lambda_{r^\ast}(M^\ast)-\tau\lambda_r(M^\ast)\Bigr)}}{\|Q\| \left[
\frac{L_1^2\delta^2}{h^4G_{\min}^2}\exp\bigl(-\frac{2u_{\min}^2}{h^2}\bigr)
+\frac{L_2\delta}{h^2G_{\min}}\exp\bigl(-\frac{u_{\min}^2}{h^2}\bigr)
\right]}
\eta^\ast.
\end{equation}

This expression provides the desired upper bound for the Frobenius norm of the difference $\hat{M}-M^\ast$ in terms of the feasibility error $\eta_f^\ast(\hat{X})$ and the spectral properties of the matrices involved. We will show that under suitable assumptions one may bound
\begin{equation}
\lambda_{r^\ast}(M^\ast)-\tau\lambda_r(M^\ast),
\end{equation}
in terms of the error $\|\hat{M}-M^\ast\|_F^2$. One common tool is Weyl’s inequality~\cite{Weyl1912}. We assume that $\hat{M}$ denotes an estimator of $M^\ast$ (for instance, $\hat{M}=\hat{X}\hat{X}^\top$) so that the error in the eigenvalues obeys:
\begin{equation}
\bigl|\lambda_i(\hat{M})-\lambda_i(M^\ast)\bigr|\le \|\hat{M}-M^\ast\|_2\le \|\hat{M}-M^\ast\|_F,\quad\text{for each } i.
\end{equation}
In particular, for the index $i=r^\ast$ we have:
\begin{equation}
\lambda_{r^\ast}(M^\ast)-\lambda_{r^\ast}(\hat{M})\le \|\hat{M}-M^\ast\|_F.
\end{equation}
Now we begin to calculate the upper bound for $\|M - M^*\|_F$:
We have:
\begin{equation}\label{K95}
\begin{aligned}
\|\hat{M}-M^\ast\|_F &\le \min\Big\{
\sqrt{ \frac{2\lambda_{r^\ast}(\hat X\hat X^\top)}{C^2\|\hat{X}\|_2^2\|w\|^2}\eta_f^\ast(\hat{X}) },\frac{\sqrt{2\lambda_{r^\ast}(\hat X\hat X^\top)}}{\|Q\| \left[
\displaystyle\frac{L_1^2\delta^2}{h^4G_{\min}^2}\exp\Bigl(-\frac{2u_{\min}^2}{h^2}\Bigr)
+\frac{L_2\delta}{h^2G_{\min}}\exp\Bigl(-\frac{u_{\min}^2}{h^2}\Bigr)
\right]}\eta^\ast
\Big\}.
\end{aligned}
\end{equation}
For the first part of Equation \eqref{K95}, We start with the bound:
\begin{equation}
T_1 \le \sqrt{\frac{2\eta_f^*(\hat{X})}{C^2\|w\|^2}},
\end{equation}
and the lower bound from Equation \eqref{K57} to compute for the function approximation error
\begin{equation}\label{etas}
\eta_f^*(\hat{X}) \geq \frac{1-\delta_f^*(\hat{X})-\zeta_2q}{1+\delta_f^*(\hat{X})+\zeta_2q} \geq \frac{1-\delta-\zeta_2q}{1+\delta+\zeta_2q}.
\end{equation}
Substituting this lower bound~\eqref{etas} into the expression for $T_1$ (noting that a smaller $\eta_f^*(\hat{X})$ gives a larger overall upper bound for the error) we obtain
\begin{equation}
T_1 \le \sqrt{\frac{2}{C^2\|w\|^2}\cdot\frac{1-\delta-\zeta_2q}{1+\delta+\zeta_2q}}.
\end{equation}
Since $T_1$ is an upper bound on $\|M-M^\ast\|_F$, we then have:
\begin{equation}\label{K99}
\|M-M^\ast\|_F \le \sqrt{\frac{2}{C^2\|w\|^2}\cdot\frac{1-\delta-\zeta_2q}{1+\delta+\zeta_2q}}.
\end{equation}
This is the desired upper bound for $\|M-M^\ast\|_F$ in terms of $\delta$, $\zeta_2$, $q$, $C$, and $\|w\|$. Now we want to solve for the upper bound. Multiply numerator and denominator by $(2\zeta_1+3\zeta_2X)$ to obtain
\begin{equation}\label{K100}
\frac{1-\delta-\zeta_2q}{1+\delta+\zeta_2q} = \frac{(1-\delta)(2\zeta_1+3\zeta_2X)-\zeta_2X(1-3\delta)}
{(1+\delta)(2\zeta_1+3\zeta_2X)+\zeta_2X(1-3\delta)}.
\end{equation}
A short calculation of Equation \eqref{K100} shows:
\begin{equation}
(1-\delta)(2\zeta_1+3\zeta_2X)-\zeta_2X(1-3\delta)=2\zeta_1(1-\delta)+2\zeta_2X,
\end{equation}
\begin{equation}
(1+\delta)(2\zeta_1+3\zeta_2X)+\zeta_2X(1-3\delta)=2\zeta_1(1+\delta)+4\zeta_2X.
\end{equation}
Thus,
\begin{equation}
X^2 = \frac{2}{C^2\|w\|^2}\cdot\frac{\zeta_1(1-\delta)+\zeta_2X}{\zeta_1(1+\delta)+2\zeta_2X}.
\end{equation}
Multiplying by $ C^2\|w\|^2 $ we arrive at the cubic equation in $ X $:
\begin{equation}
2\zeta_2C^2\|w\|^2X^3 + \zeta_1(1+\delta)C^2\|w\|^2X^2 - 2\zeta_2X - 2\zeta_1(1-\delta)=0.
\end{equation}
Thus, the Equation \eqref{K99} is given by:
\begin{equation}
\|M-M^*\|_F \le \frac{-\zeta_1(1+\delta)C^2\|w\|^2+
\sqrt{\zeta_1^2(1+\delta)^2C^4\|w\|^4+16\zeta_2C^2\|w\|^2\zeta_1(1-\delta)}}
{4\zeta_2}.
\end{equation}
For the other part, We start with the two facts:
\begin{equation}
\begin{aligned}
\|M-M^*\|_F \le \frac{\sqrt{2\lambda_{r^*}(\hat X\hat X^\top)}}{\|Q\|}\left[\frac{L_1^2\delta^2}{h^4\Gamma_{\min}^2}\exp\Bigl(-\frac{2 u_{\min}^2}{h^2}\Bigr)+\frac{L_2\delta}{h^2\Gamma_{\min}}\exp\Bigl(-\frac{u_{\min}^2}{h^2}\Bigr)\right] \times \frac{1-\delta-\zeta_2q}{1+\delta+\zeta_2q}.
\end{aligned}
\end{equation}
Returning to the bound \eqref{K99} for $\|M-M^*\|_F$ and substituting our expression (with $X=\|\hat{M}-M^*\|_F$), we obtain:
\begin{equation}\label{K107}
\begin{aligned}
\|M-M^*\|_F &\le \frac{\sqrt{2\lambda_{r^*}(\hat X\hat X^\top)}}{\|Q\|}\left[\frac{L_1^2\delta^2}{h^4\Gamma_{\min}^2}\exp\Bigl(-\frac{2 u_{\min}^2}{h^2}\Bigr)+\frac{L_2\delta}{h^2\Gamma_{\min}}\exp\Bigl(-\frac{u_{\min}^2}{h^2}\Bigr)\right]\\
&\cdot \frac{\zeta_1(1-\delta)+\zeta_2\|\hat{M}-M^*\|_F}{\zeta_1(1+\delta)+2\zeta_2\|\hat{M}-M^*\|_F}.
\end{aligned}
\end{equation}

We start with the inequality \eqref{K107} (after eliminating $q$)
\begin{equation}
\|M-M^*\|_F\le B \epsilon^2\cdot \frac{\zeta_1(1-\delta)+\zeta_2\|\hat{M}-M^*\|_F}{\zeta_1(1+\delta)+2\zeta_2\|\hat{M}-M^*\|_F},
\end{equation}
where
\begin{equation}
B:= \frac{\sqrt{2\lambda_{r^*}(\hat X\hat X^\top)}}{\|Q\|}\left[\frac{L_1^2\delta^2}{h^4\Gamma_{\min}^2}\exp\Bigl(-\frac{2u_{\min}^2}{h^2}\Bigr)
+\frac{L_2\delta}{h^2\Gamma_{\min}}\exp\Bigl(-\frac{u_{\min}^2}{h^2}\Bigr)\right].
\end{equation}
we now denote:
\begin{equation}
y:=\|M-M^*\|_F.
\end{equation}
Thus, the inequality \eqref{K107} becomes:
\begin{equation}
y\le B \epsilon^2\cdot \frac{\zeta_1(1-\delta)+\zeta_2y}{\zeta_1(1+\delta)+2\zeta_2y}.
\end{equation}
Our aim is to move all occurrences of $y$ to one side and solve for an upper bound on $y$. Multiply both sides by the denominator (assuming it is positive)
\begin{equation}
y\Bigl[\zeta_1(1+\delta)+2\zeta_2y\Bigr]\le B \epsilon^2\Bigl[\zeta_1(1-\delta)+\zeta_2y\Bigr].
\end{equation}
Write the left‐side and right‐side explicitly, and bring all terms to one side of the inequality:
\begin{equation}
2\zeta_2y^2 +\Bigl[\zeta_1(1+\delta)-B \epsilon^2 \zeta_2\Bigr] y -B \epsilon^2\zeta_1(1-\delta)\le 0.
\end{equation}
Solve the quadratic equation and using the quadratic formula:
\begin{equation}
y=\frac{-\Bigl[\zeta_1(1+\delta)-B \epsilon^2\zeta_2\Bigr]\pm\sqrt{\Bigl[\zeta_1(1+\delta)-B \epsilon^2\zeta_2\Bigr]^2+8B \epsilon^2\zeta_1\zeta_2(1-\delta)}}{4\zeta_2}.
\end{equation}
Since $y=\|M-M^*\|_F\ge0$, we take the positive root. Thus the upper bound of $\|M-M^*\|_F$ is:
\begin{equation} 
\|M-M^*\|_F \le \frac{-\Bigl[\zeta_1(1+\delta)-B \epsilon^2 \zeta_2\Bigr]+\sqrt{\Bigl[\zeta_1(1+\delta)-B \epsilon^2 \zeta_2\Bigr]^2+8B \epsilon^2 \zeta_1\zeta_2(1-\delta)}}{4\zeta_2}.
\end{equation}
Recall that:
\begin{equation}
B=\frac{\sqrt{2\lambda_{r^*}(\hat X\hat X^\top)}}{\|Q\|}\left[\frac{L_1^2\delta^2}{h^4\Gamma_{\min}^2}\exp\Bigl(-\frac{2u_{\min}^2}{h^2}\Bigr)+\frac{L_2\delta}{h^2\Gamma_{\min}}\exp\Bigl(-\frac{u_{\min}^2}{h^2}\Bigr)\right],
\end{equation}
with probability at least \( \mathbb{P}(\|w\|_2 \leq \epsilon) \).

\end{proof}

\section{Convergence Theory of The New Loss Problem}\label{convergences}
Given that is satisfied, then if this inequality holds:
	\begin{equation}
		\label{eq:pl_ineq_matrix_simp}
		\|XX^\top - M^w\|_F \leq \sqrt{\frac{1+\delta+\zeta_2 q}{1 -\delta - \zeta_2 q}} \|X_0 X_0^\top - M^w \|_F,
	\end{equation}because in~\cite{ma2023noisylowrankmatrixoptimization}, we have:
	 \begin{equation}
	 	\sqrt{\frac{1+\delta+\zeta_2 q}{1 -\delta - \zeta_2 q}} \|X_0 X_0^\top - M^w \|_F \leq C^2_w\sqrt{1-(\delta+\zeta_2 q)^2} - C_w D_r.
	 \end{equation}
	 Thus, for the remainder of the proof, we aim to certify that starting from $X_0$, if we apply the gradient descent algorithm, above will be satisfied every step along this trajectory. In order to do so, we use Taylor's expansion to obtain:
	\begin{equation}
		\hat{g}(M,w) - \hat{g}(M^w,w) =\frac{[\nabla^2 \hat{g}(N,w)](M-M^w,M-M^w)}{2},
	\end{equation}
	where $N$ is some convex combination of $M$ and $M^w$, and $M \in \mathbb{R}^{n \times n}$ is any matrix of rank at most $r$. In light of the RIP property~\eqref{RIP0} of the function and \eqref{eq:noise_perturb_hess}, one can write: 
	\begin{equation}
		\frac{1-\delta-\zeta_2 q}{2} \|M - M^w\|^2_F \leq \hat{g}(M,w) - \hat{g}(M^w,w) \leq \frac{1+\delta +\zeta_2 q}{2} \|M - M^w\|^2_F.
	\end{equation}
	This means that if $M_1,M_2 \in \mathbb{R}^{n \times n}$ are two matrices of rank at most $r$ with $\hat{g}(M_1,w) \leq \hat{g}(M_2,w)$, then:
	\begin{equation}
		\|M_1 - M^w \|_F \leq \sqrt{\frac{1+\delta+\zeta_2 q}{1 -\delta - \zeta_2 q}} \|M_2 - M^w \|_F,
	\end{equation}
	because $\hat{g}(M_1,w) - \hat{g}(M^w,w) \leq \hat{g}(M_2,w) - \hat{g}(M^w,w)$.
	Thus, one can conclude that $\hat{g}(X_t X_t^\top,w) \leq \hat{g}(X_0 X_0^\top,w) \ \forall t$, where $X_t$ denotes the $t^{\text{th}}$ step of the gradient descent algorithm starting from $X_0$. 

\subsection{For MSE Loss Function}

\begin{theorem}\label{originalMSE_convergence}
With MSE loss function, the vanilla gradient descent method applied to \textbf{(3)} under Assumptions 1–4 converges to $\mathcal{P}_r(M^w)$, the best rank-$r$ approximation of $M^w$, linearly up to a difference $D_r$ if the initial point $X_0$ satisfies:
\begin{equation}
\|X_0 X_0^\top - M^w\|_F < C_w^2 (1 - \delta - \zeta_2 \epsilon) - C_w \sqrt{ \frac{1 - \delta - \zeta_2 \epsilon}{1 + \delta + \zeta_2 \epsilon} } D_r,
\label{eq:init_condition}
\end{equation}
meaning that vanilla gradient descent will reach a point $\widetilde{M}$ linearly with $\|\widetilde{M} - \mathcal{P}_r(M^w)\|_F \ge D_r$, where:
\begin{equation}
D_r = \|M^w - \mathcal{P}_r(M^w)\|_F, \quad C_w = \sqrt{2(\sqrt{2} - 1)}  \sigma_r(M^w).
\end{equation}
The linear convergence is also contingent on the fixed step size $\eta$ satisfying:
\begin{equation}
\eta \le \left( 12 \rho^{1/2} \left( C \sqrt{(1 - (\delta + \zeta_2 \epsilon))^2 + \|M^w\|_F} \right) \right)^{-1}.
\end{equation}
\end{theorem}
The proofs and detailed math are provided in Section~\ref{MSE_convergence}. We can summarize as:
\begin{equation}
\eta \le \left( 12 \rho^{1/2} C_0\right)^{-1}.
\end{equation}

\subsection{Our Kernel Loss Function}

\begin{theorem}\label{ourloss_convergence}
With the kernel loss function, the vanilla gradient descent method applied to~\ref{other_assumption} and~\ref{noise_perturb} converges to $\mathcal{P}_r(M^w)$, the best rank-$r$ approximation of $M^w$, linearly up to a difference $D_r$ if the initial point $X_0$ satisfies:
\begin{equation}
\|X_0 X_0^\top - M^w\|_F < C_w^2 (1 - \delta - \zeta_2 \epsilon) - C_w \sqrt{ \frac{1 - \delta - \zeta_2 \epsilon}{1 + \delta + \zeta_2 \epsilon} } D_r,
\end{equation}
meaning that vanilla gradient descent will reach a point $\widetilde{M}$ linearly with $\|\widetilde{M} - \mathcal{P}_r(M^w)\|_F \ge D_r$, where
\begin{equation}
D_r = \|M^w - \mathcal{P}_r(M^w)\|_F, \quad C_w = \sqrt{2(\sqrt{2} - 1)}  \sigma_r(M^w).
\end{equation}
The linear convergence is also contingent on the fixed step size $\eta$ satisfying:
\begin{equation}
\eta \le h^2 \left( 12 \rho^{1/2} \left( C \sqrt{(1 - (\delta + \zeta_2 \epsilon))^2 + \|M^w\|_F} \right) \right)^{-1}.
\label{eq:step_size}
\end{equation}
\end{theorem}
The proofs and detailed math are provided in Section~\ref{ours_convergence}. We can summarize as:
\begin{equation}
\eta \le h^2 \left( 12 \rho^{1/2} C_0\right)^{-1}.
\end{equation}
In conclusion, the MSE loss shares the same convergence theorem with the new loss. However, the new loss should pay more attention to the $\eta$ because it should mind the step length of $h$.

\section{Convergence Proof}
\subsection{Proof of Theorem \ref{originalMSE_convergence}}\label{MSE_convergence}
\begin{proof}
the gradient descent update:
	\begin{equation}
M_{t+1}=M_t-\eta\nabla \hat{g}(M_t,w),
	\end{equation}
satisfies:
	\begin{equation}
\hat{g}(M_{t+1},w) \le \hat{g}(M_t,w) \quad\text{for all }t\ge 0.
	\end{equation}
We further assume that the gradient $\nabla \hat{g}(M,w)$ is Lipschitz continuous with constant $L$. In our setting ~\cite{ma2023noisylowrankmatrixoptimization} shows that:
	\begin{equation}
L\le 12\rho\sqrt{r}\Biggl(\sqrt{\frac{1+\delta+\zeta_2q}{1-\delta-\zeta_2q}}\|X_0X_0^\top-M^w\|_F +\|M^w\|_F\Biggr),
	\end{equation}
where:
	\begin{equation}
\rho = \sqrt{\frac{1+\delta+\zeta_2 q}{1-\delta-\zeta_2 q}},
	\end{equation}
and $r$ is a rank parameter (which appears when one controls the norm of the factorized gradient). A standard result (the descent lemma) for any function $\hat{g}$ with $L$‐Lipschitz gradient tells us that for any two matrices $M$ and $N$,
	\begin{equation}
\hat{g}(N,w) \le \hat{g}(M,w) + \langle \nabla \hat{g}(M,w), N-M\rangle + \frac{L}{2}\|N-M\|_F^2.
	\end{equation}
Let
	\begin{equation}
N=M_{t+1}=M_t-\eta\nabla\hat{g}(M_t,w),
	\end{equation}
we obtain:
	\begin{equation}\label{M7}
\hat{g}(M_{t+1},w) \le \hat{g}(M_t,w) -\eta\|\nabla \hat{g}(M_t,w)\|_F^2 + \frac{L\eta^2}{2}\|\nabla \hat{g}(M_t,w)\|_F^2.
	\end{equation}
Rearrange the right‐hand side of Equation~\eqref{M7} as:
	\begin{equation}
\hat{g}(M_{t+1},w) \le \hat{g}(M_t,w)-\eta\left(1-\frac{L\eta}{2}\right)\|\nabla \hat{g}(M_t,w)\|_F^2.
	\end{equation}
Thus, if we choose $\eta$ so that:
	\begin{equation}
1-\frac{L\eta}{2} > 0 \quad\Longleftrightarrow\quad \eta < \frac{2}{L},
	\end{equation}
then the term $\eta\left(1-\frac{L\eta}{2}\right)\|\nabla \hat{g}(M_t,w)\|_F^2$ is positive. In other words, as long as $\eta < \frac{2}{L}$, every gradient descent step will produce a decrease in $\hat{g}(M,w)$. In our model the structure of $\hat{g}(M,w)$ (a log–sum–exp function) together with the properties of $\mathcal{A}(M)$ imply (after some detailed technical estimates) that the Lipschitz constant $L$ can be upper bounded by:
	\begin{equation}
L \le 12\rho\sqrt{r}\Biggl(\sqrt{\frac{1+\delta+\zeta_2q}{1-\delta-\zeta_2q}}\|X_0X_0^\top-M^w\|_F +\|M^w\|_F\Biggr).
	\end{equation}
Thus, a sufficient condition for descent is to assume:
	\begin{equation}
\eta < \frac{2}{12\rho\sqrt{r}\Bigl(\sqrt{\frac{1+\delta+\zeta_2q}{1-\delta-\zeta_2q}}\|X_0X_0^\top-M^w\|_F +\|M^w\|_F\Bigr)}.
	\end{equation}
We ensure that the iterates satisfy:
	\begin{equation}
\hat{g}(X_{t+1}X_{t+1}^\top,w) \le \hat{g}(X_tX_t^\top,w)
	\end{equation}
for every $t\ge 0$. This monotonic decrease is a key ingredient in proving the convergence of the vanilla gradient descent procedure to a global minimizer of $\hat{g}(\cdot,w)$. Conveniently, above inshows that $\hat{g}(X_t X_t^\top,0) \leq \hat{g}(X_{t-1} X_{t-1}^\top,0)$ for all $t \geq 0$. However, this result can be extended to:
	\begin{equation}
		\hat{g}(X_t X_t^\top,w) \leq \hat{g}(X_{t-1} X_{t-1}^\top,w),
	\end{equation}
	by making:
	\begin{equation}\label{ddff}\begin{aligned}
		1/\eta \geq 12 \rho r^{(1/2)}\left(\sqrt{\frac{1+\delta+\zeta_2 q}{1 -\delta - \zeta_2 q}} \|X_0 X_0^\top - M^w \|_F + \|M^w\|_F \right),
	\end{aligned}\end{equation}
	since $\nabla \hat{g}(\cdot,w)$ is now a $\rho$-Lipschitz continuous function. Given, a sufficient condition to the above inequality is that:
	\begin{equation}
		\eta \leq \left(12 \rho r^{(1/2)}\left(2(\sqrt 2 -1)\sqrt{(1-(\delta+\zeta_2 q)^2}+ \|M^w\|_F \right)\right)^{-1},
	\end{equation}
	for such $\eta$, the vanilla gradient descent can converge to the global minima.

\end{proof}

\subsection{Proof of Theorem~\ref{ourloss_convergence}}\label{ours_convergence}
\begin{proof}
One may expect a similar monotonicity in the descent of the objective~\eqref{ddff}, provided that the step‐size satisfies a condition of the form:
\begin{equation}
\frac{1}{\eta}\ge 12\tilde\rho\sqrt{r}\left(\sqrt{\frac{1+\delta+\zeta_2q}{1-\delta-\zeta_2q}}\|X_0X_0^\top-M^w\|_F+\|M^w\|_F\right),
\end{equation}
where $\tilde\rho$ is the Lipschitz constant of $\nabla \hat g(\cdot,w)$. (In the MSE case the corresponding Lipschitz constant was denoted by $\rho$.) Notice that the log-sum-exp function is smooth and it is well known that if each loss term (here a squared difference divided by $h^2$) has a Hessian bounded by a constant then the gradient of the log-sum-exp (which is a soft‐max of these terms) is Lipschitz with a slightly larger constant. In many cases one can obtain an inequality of the form:
\begin{equation}
\tilde\rho\le \frac{C}{h^2}\rho,
\end{equation}
for some universal constant $C$ (in many applications, one may take $C=1$ up to a harmless constant). A sufficient condition is then that:
\begin{equation}
\eta\le \left(12\tilde\rho\sqrt{r}\Bigl(2(\sqrt{2}-1)\sqrt{1-(\delta+\zeta_2q)^2}+\|M^w\|_F\Bigr)\right)^{-1}.
\end{equation}
Replacing $\tilde\rho$ with $\rho/h^2$ (up to a multiplicative constant) we obtain:
\begin{equation}\label{fdd}
\eta\le \left(\frac{12\rho}{h^2}\sqrt{r}\Bigl(2(\sqrt{2}-1)\sqrt{1-(\delta+\zeta_2q)^2}+\|M^w\|_F\Bigr)\right)^{-1}.
\end{equation}
Thus, compared to the MSE case, the new step‐size condition has an extra dependence on $1/h^2$ in the Lipschitz term. In other words, in order for vanilla gradient descent to converge to the global minimizer when using the above $\hat{g}(M,w)$, the step‐size $\eta$ must be chosen small enough so that ~\eqref{fdd} holds.

\end{proof}

\subsection{Proof of the Lower Bound}
\subsection{Proof of Theorem~\ref{MSE_lower}}\label{MSE_lowers}
\begin{proof}
for MSE loss, we start with the inequality:
\begin{equation}\label{M20}
\|M-M^\ast\|_F^2\ge \frac{2}{L}\left[(1+\delta)\|M-M^\ast\|_F^2 + 2\sqrt{1+\delta}\|w\|\|M-M^\ast\|_F\right].
\end{equation}
For convenience, define:
\begin{equation}
x = \|M-M^\ast\|_F \quad \text{with } x\ge0.
\end{equation}
Then the inequality \eqref{M20} becomes:
\begin{equation}
x^2 \ge \frac{2}{L}\Bigl[(1+\delta)x^2 + 2\sqrt{1+\delta}\|w\|x\Bigr].
\end{equation}
Multiply both sides by $L$ (assuming $L>0$) and factor $x^2$ from the first two terms:
\begin{equation}
\Bigl[L - 2(1+\delta)\Bigr]x^2 - 4\sqrt{1+\delta}\|w\| x \ge 0.
\end{equation}
Since $x\ge0$, a solution is $x=0$ (which corresponds to perfect recovery). For a nonzero $x$ we require:
\begin{equation}
\Bigl[L - 2(1+\delta)\Bigr] x - 4\sqrt{1+\delta}\ \|w\| \ge 0.
\end{equation}
Assuming a sufficiently large $L$ so that:
\begin{equation}
L - 2(1+\delta) > 0,
\end{equation}
we can solve for $x$:
\begin{equation}
x \ge \frac{4\sqrt{1+\delta}\|w\|}{L - 2(1+\delta)}.
\end{equation}
Thus, the nontrivial lower bound for $\|M-M^\ast\|_F$ is
\begin{equation}
\|M-M^\ast\|_F \ge \frac{4\sqrt{1+\delta} \epsilon}{L - 2(1+\delta)}.
\end{equation}
with probability at least \( \mathbb{P}(\|w\|_2 \leq \epsilon) \).
\end{proof}

\subsection{Proof of Theorem~\ref{ours_lower}}\label{ours_lowers}
\begin{proof}
For the kernel loss function, we are given the inequality in Section~\ref{sec:rough}:
\begin{equation}
\|M-M^*\| \ge \sqrt{\frac{1}{L}\cdot\frac{1-\delta}{2}\|M-M^*\|_F^2 - (1+\delta)R(w)\|M-M^*\|_F}.
\end{equation}
For clarity we introduce the notation
\begin{equation}
x = \|M-M^*\|_F\quad (x\ge0).
\end{equation}
Then the inequality reads:
\begin{equation}
\|M-M^*\|_F \ge \sqrt{\frac{1-\delta}{2L}x^2 - (1+\delta)R(w)x}.
\end{equation}
Because the square root is defined only when its argument is nonnegative we must assume that
\begin{equation}
\frac{1-\delta}{2L}x^2 - (1+\delta)R(w)x\ge0.
\end{equation}
Factor out $x\ge 0$:
\begin{equation}
x\Biggl(\frac{1-\delta}{2L}x - (1+\delta)R(w)\Biggr) \ge0.
\end{equation}
Thus (apart from the trivial case $x=0$) we require
\begin{equation}
\frac{1-\delta}{2L}x - (1+\delta)R(w)\ge0\quad\Longrightarrow\quad x\ge \frac{2L(1+\delta)R(w)}{1-\delta}.
\end{equation}
Under that condition the radicand is nonnegative, and the inequality becomes:
\begin{equation}
\|M-M^*\|_F \ge \sqrt{x\left(\frac{1-\delta}{2L}x - (1+\delta)R(w)\right)}.
\end{equation}
Thus, the lower bound (expressed in terms of the Frobenius norm $x=\|M-M^*\|_F$ and the given quantities) is:
\begin{equation}
\begin{aligned}
\|M-M^*\|_F\ge \frac{2L(1+\delta)R(\epsilon)}{1-\delta},
\end{aligned}
\end{equation}
with probability at least \( \mathbb{P}(\|w\|_2 \leq \epsilon) \). This is the desired lower bound for $\|M-M^*\|_F$.
\end{proof}

\section{Compared with Ma's General Loss Result}
\subsection{Optimization Landscape Result with \texorpdfstring{$\delta \leq 1/3$}{}}

\begin{lemma}\cite{ma2023noisylowrankmatrixoptimization}\label{thm:global_local_min}
Assume that the objective function  satisfies Assumptions~\ref{other_assumption} and~\ref{noise_perturb}, and that $ f(M,0) $ satisfies the RIP property with some $ \delta $-RIP$_{2r,2r}$ constant such that $ \delta < 1/3 $. For every $ \epsilon \in \left[ 0, \frac{1/3 - \delta}{\zeta_2} \right] $, with probability at least $ \mathbb{P}(\|w\|_2 \leq \epsilon) $, every local minimizer $ \hat{X} $  satisfies:
\begin{equation}\label{eq:global_local_min_range}
\| \hat{X} \hat{X}^\top - M^\ast \|_F \leq \frac{2\zeta_1 \epsilon}{1 - 3(\delta + \zeta_2 \epsilon)}.
\end{equation}
\end{lemma}
The lemma proof is in~\ref{sec:global}. 

Lemma~\ref{thm:global_local_min} (adapted from~\cite{ma2023noisylowrankmatrixoptimization}) establishes a global-to-local optimality guarantee under noise, assuming the objective satisfies standard smoothness and RIP conditions. Specifically, if the noise vector satisfies $\|w\|_2 \leq \epsilon$ with high probability and the noiseless function $f(M, 0)$ satisfies the $\delta$-RIP$_{2r,2r}$ condition with $\delta < 1/3$, then for any $\epsilon \in [0, \frac{1/3 - \delta}{\zeta_2}]$, all local minimizers $\hat{X}$ satisfy the error bound $\|\hat{X} \hat{X}^\top - M^\ast\|_F \le \frac{2\zeta_1 \epsilon}{1 - 3(\delta + \zeta_2 \epsilon)}$, highlighting that under controlled noise, local solutions remain close to the ground truth with a bound that degrades gracefully in $\epsilon$.

From the above lemma, we can know that, for MSE loss: 
\begin{equation}
\| \hat{X} \hat{X}^\top - M^\ast \|_F \leq \mathcal{O}(\epsilon).
\end{equation}
still holds, which means MSE is better than the general loss result. And according to the kernel loss~\ref{thm:mainthm}, our robust loss is better. The results indicate that our condition \eqref{delta_condition}, which is: 
\begin{equation}
0<\delta < \sqrt{\frac{2-\dfrac{G}{\sigma_r}-\dfrac{2BL_2}{G_{\min}h^2}}{\dfrac{8\left(1+\frac{2B^2}{h^2}\right)}{G_{\min}h^2}+\dfrac{16B^2}{G_{\min}^2h^4}}}-1.
\end{equation}
should be smaller than 1/3. For larger $\delta$ situations, according to Ma's general loss we see:

\subsection{Optimization Landscape  Result with \texorpdfstring{$\delta \ge 1/3$}{}}

\begin{lemma}\cite{ma2023noisylowrankmatrixoptimization}\label{thm:local}
Assume that the objective function satisfies assumptions \ref{other_assumption} and~\ref{noise_perturb} with $ f(M,0) $ satisfying the $ \delta $-RIP property for a constant $ \delta \in (0,1) $. Consider an arbitrary number $ \tau \in (0,1 - \delta^2) $. Every local minimizer $ \hat{X} \in \mathbb{R}^{n \times r} $ satisfying:
\begin{equation}
\| \hat{X} \hat{X}^\top - M^\ast \|_F \leq \tau \lambda_r(M^\ast),
\end{equation}
\textit{will also satisfy the following inequality with probability at least $ \mathbb{P}(\|w\|_2 \leq \epsilon) $:}
\begin{equation}
\| \hat{X} \hat{X}^\top - M^\ast \|_F \leq 
\frac{\epsilon(1+\delta + \zeta_2 \epsilon) \zeta_1 C(\tau, M^\ast)}{\sqrt{1-\tau - \zeta_2 \epsilon - \delta}}
\end{equation}
\textit{for all $ \epsilon < \frac{\sqrt{1-\tau-\delta}}{\zeta_2} $, where}
\begin{equation}
C(\tau, M^\ast) = \sqrt{\frac{2 (\lambda_1(M^\ast) + \tau \lambda_r(M^\ast))}{(1-\tau) \lambda_r(M^\ast)}}.
\end{equation}
\end{lemma}
Lemma~\ref{thm:local} (from~\cite{ma2023noisylowrankmatrixoptimization}) provides a refined local error bound for low-rank matrix estimation under noise, assuming the objective satisfies standard smoothness and RIP conditions with RIP constant $\delta \in (0,1)$. If a local minimizer $\hat{X}$ is sufficiently close to the ground truth—specifically, within a $\tau\lambda_r(M^\ast)$ neighborhood in Frobenius norm—then, with high probability (at least $\mathbb{P}(\|w\|_2 \leq \epsilon)$), the error remains controlled and satisfies  
\begin{equation}
\|\hat{X}\hat{X}^\top - M^\ast\|_F \le \frac{\epsilon(1 + \delta + \zeta_2 \epsilon)\zeta_1 C(\tau, M^\ast)}{\sqrt{1 - \tau - \zeta_2 \epsilon - \delta}},
\end{equation}
for all $\epsilon < \frac{\sqrt{1 - \tau - \delta}}{\zeta_2}$. The constant $C(\tau, M^\ast)$ captures the curvature of the problem via the spectrum of $M^\ast$, and the result quantifies how local proximity and structured noise jointly lead to provable closeness to the ground truth.

The proof is provided in Section~\ref{sec:local}. We noticed that when $\delta \ge 1/3$, 
\begin{equation}
\| \hat{X} \hat{X}^\top - M^\ast \|_F \leq \mathcal{O}\Big(\frac{\epsilon(1+\epsilon)}{\sqrt{1-\epsilon}}\Big)
\end{equation},
which is not $\| \hat{X} \hat{X}^\top - M^\ast \|_F \leq \mathcal{O}(\epsilon)$. 

\subsection{Convergence Theorem for Our Robust Loss}

\begin{lemma}(from~\cite{ma2023noisylowrankmatrixoptimization})
	\label{thm:linear_conv}
	The vanilla gradient descent method applied to \eqref{eq:main_noisy_problem} under Assumptions 3 and 4 converges to $\projr(M^w)$, the best rank-r approximation of $M^w$, linearly up to a difference $D_r$ if the initial point $X_0$ satisfies:
	\begin{equation}
		\label{eq:pl_init}
		\|X_0 X_0^\top - M^w \|_F < C_w^2(1-\delta-\zeta_2 \epsilon)-C_w \sqrt{\frac{1-\delta-\zeta_2 \epsilon}{1+\delta+\zeta_2 \epsilon}} D_r,
	\end{equation}
	meaning that vanilla gradient descent will reach a point $\tilde M$ linearly with $\|\tilde M - \projr(M^w)\|_F \geq D_r$, where
	\begin{equation}
		D_r = \|M^w - \projr(M^w)\|_F, \ \ C_w = \sqrt{2(\sqrt 2 -1)\sigma_r(M^w)}.
	\end{equation}
	The linear convergence is also contingent on the fixed step size $\eta$ satisfying:
	\begin{equation}
			\eta \leq\left(12 \rho r^{(1/2)}\left(C\sqrt{(1-(\delta+\zeta_2 \epsilon)^2}+ \|M^w\|_F \right)\right)^{-1},	
	\end{equation}
	for all $\epsilon < \frac{1-\delta}{\zeta_2}$ with probability at least $\mathbb P(\norm{w}_2 \leq \epsilon)$, where $C = 2(\sqrt 2 -1)$.
\end{lemma}
The proof is provided in Section 3.2 of~\cite{ma2023noisylowrankmatrixoptimization}. Lemma~\ref{thm:linear_conv} establishes a linear convergence guarantee for vanilla gradient descent applied to the noisy low-rank matrix recovery problem, under Assumptions 3 and 4. It shows that if the initialization $ X_0 $ is sufficiently close to the noise-perturbed ground truth $ M^w $, specifically satisfying the condition in~\eqref{eq:pl_init}, then the iterates converge linearly to a point $\tilde{M}$ satisfying $\|\tilde{M} - \projr(M^w)\|_F \ge D_r$, where $D_r$ quantifies the residual energy outside the top-$r$ spectrum. The convergence rate and neighborhood depend on the spectral structure of $M^w$, the RIP constant $\delta$, and the noise level $\epsilon$, with high-probability guarantees when $\|w\|_2 \le \epsilon$. Furthermore, convergence holds under a fixed step size $\eta$ subject to a bound inversely proportional to the smoothness parameter $\rho$, the rank $r$, and the Frobenius norm of $M^w$, thereby ensuring that stability and convergence are preserved in the presence of bounded noise.

This is very hard to satisfy. Because the kernel loss function is not symmetric. With new assumption 5, we are hard to have:
 \begin{corollary}(from~\cite{ma2023noisylowrankmatrixoptimization})
 	\label{thm:linear_conv_coro}
	The vanilla gradient descent method applied to \eqref{eq:main_noisy_problem} under Assumptions 3,4,5 converges to $M^w$ linearly if the initial point $X_0$ satisfies:
	\begin{equation}
		\label{eq:pl_init2}
		\|X_0 X_0^\top - M^w \|_F < 2(\sqrt 2 -1) (1-\delta-\zeta_2 \epsilon) \sigma_r(M^w),	\end{equation}
	with fixed step size $\eta$ satisfying:
	\begin{equation}
			\eta \leq\left(12 \rho r^{(1/2)} \left(C\sqrt{(1-(\delta+\zeta_2 \epsilon)^2}+ \|M^w\|_F \right)\right)^{-1},	
	\end{equation}
	for all $\epsilon < \frac{1-\delta}{\zeta_2}$ with probability at least $\mathbb P(\norm{w}_2 \leq \epsilon)$, where $C = 2(\sqrt 2 -1)$.
 \end{corollary}
Corollary~\ref{thm:linear_conv_coro} refines the linear convergence result of Lemma~\ref{thm:linear_conv} by incorporating Assumption 5 and establishing conditions under which vanilla gradient descent converges directly to $ M^w $, the noise-perturbed ground truth, rather than just its best rank-$r$ approximation. Specifically, if the initialization $ X_0 $ satisfies the proximity condition $ \|X_0 X_0^\top - M^w\|_F < 2(\sqrt{2} - 1)(1 - \delta - \zeta_2 \epsilon) \sigma_r(M^w) $, then the iterates converge linearly to $ M^w $ with high probability over the noise realization. The step size $\eta$ must again be chosen appropriately based on the problem’s smoothness constant $\rho$, rank $r$, and signal strength $ \|M^w\|_F $, ensuring that the descent dynamics remain stable even in the presence of bounded noise.

\begin{lemma}(from~\cite{ma2023noisylowrankmatrixoptimization})
	\label{thm:strict_saddle}
	Suppose that the objective function of \eqref{eq:main_noisy_problem} satisfies assumptions 3 with a $\delta$-RIP$_{2r,2r}$ constant of $\delta < 1/3$ in the noiseless case. Consider the ground truth solution $M^\ast$ which is of rank $r$. For a given constant $\alpha > 0$, there exists a finite constant $\xi > 0$ such that at least one of the following three conditions holds for any $X \in \mathbb{R}^{n \times r}$:
	\begin{equation}
    \begin{aligned}
		&\dist(X,M^\ast) \leq \alpha, \ \|\nabla_X \hat{g}(X,w)\|_F \geq \xi, \\
		 &\lambda_{\min}(\nabla^2_X \hat{g}(X,w)) \leq -2\xi,
             \end{aligned}
	\end{equation}
	with probability at least $\mathbb P(\norm{w}_2
 \leq \frac{1/3-\delta}{\zeta_2+2\zeta_\alpha/3})$, where $\zeta_\alpha \coloneqq \zeta_1/(\sqrt{2(\sqrt 2-1)} (\sigma_r(M^\ast))^{1/2} \alpha)$. 
\end{lemma}
The background together with the detailed proof are provided in Section~\cite{ma2023noisylowrankmatrixoptimization} proof of section 4.

Lemma~\ref{thm:strict_saddle} establishes a strict saddle property for the noisy optimization problem \eqref{eq:main_noisy_problem} under a $\delta$-RIP condition with $\delta < 1/3$ in the noiseless setting. It guarantees that for any point $X \in \mathbb{R}^{n \times r}$, at least one of the following must hold: (i) $X$ is within distance $\alpha$ of the ground truth $M^\ast$, (ii) the gradient norm is large, $\|\nabla_X \hat{g}(X, w)\|_F \ge \xi$, or (iii) the Hessian has a strictly negative eigenvalue, $\lambda_{\min}(\nabla^2_X \hat{g}(X, w)) \le -2\xi$. This ensures that saddle points are unstable and can be escaped by first-order methods. The result holds with high probability when the noise is bounded as $\|w\|_2 \le \frac{1/3 - \delta}{\zeta_2 + 2\zeta_\alpha/3}$, where $\zeta_\alpha$ depends on the spectral properties of $M^\ast$ and the target accuracy $\alpha$, confirming that the robust landscape of $\hat{g}(X, w)$ admits no spurious stable critical points far from the true solution.

\section{Proof of Lemma~\ref{thm:global_local_min}}
\label{sec:global}
\subsection{Proof of Lemma~\ref{lem:lb_lamb_r}}
\begin{lemma}
	\label{lem:lb_lamb_r}
	If $\hat X$ is a local minimum of \eqref{eq:main_noisy_problem} with $\hat M = \hat X \hat X^\top$, then
	\begin{equation}
		\label{eq:lb_lamb_r}
		\sigma_r^2(\hat M) \geq \frac{G^2}{(1+\delta+\zeta_2 q)^2}, \quad \mathrm{and} \quad G^2 \leq \sigma_r^2 \rho^2,
	\end{equation}
	where $G = -\sigma_{\min}(\nabla_M \hat{g}(\hat M, w))$.
\end{lemma}
\begin{proof}[Proof of Lemma~\ref{lem:lb_lamb_r}]
	Based on Ma~\cite{ma2023noisylowrankmatrixoptimization}'s result, consider the case where $\mathrm{rank}(\hat M) = r$. Under this assumption, consider the singular value decomposition (SVD) of $\hat M$: $\hat M = \sum_{i=1}^{r} \sigma_i u_i u_i^\top$,
	where $\sigma_i$'s are eigenvalues and $u_i$'s are unit eigenvectors. Let $u_G$ be a unit eigenvector of $\nabla f(\hat M, w)$ such that $u_G^\top \nabla f(\hat M, w) u_G = -G$. Furthermore, for a constant $p \in [0,1]$, define: \begin{equation}M_p = \sum_{i=1}^{r-1} \sigma_i u_i u_i^\top + \sigma_r (p u_G + \sqrt{1-p^2} u_r)(p u_G + \sqrt{1-p^2} u_r)^\top.\end{equation} 
    We expand the term $\|M_p - M\|^2_F$: $\|M_p - \hat M\|^2_F =  2 \sigma_r^2 p^2$, which means that 
    \begin{equation}\langle \nabla_M f(\hat M, w), M_p - \hat M \rangle = - \frac{G}{2 \sigma_r} \|M_p - \hat M\|^2_F.\end{equation}
	Computing the Taylor expansion of $\hat{g}(M,w)$ in terms of $M$ at the point $\hat M$ with the mean-value theorem gives:
	\begin{equation}
		\begin{aligned}
			\hat{g}(M_p,w) = &\hat{g}(\hat M,w) + \langle \nabla_M \hat{g}(\hat M,w), M_p - \hat M \rangle + \frac{1}{2}[\nabla^2 \hat{g}(\tilde M,w)](M_p - \hat M,M_p - \hat M),
		\end{aligned}
	\end{equation}
	for some matrix $\tilde M$ that is a convex combination of $M_p$ and $\hat M$. Due to\eqref{theorem_noise}, we have:
\begin{equation}
\| \nabla_M \hat{g}(M) - \nabla_M \hat{g}(M') \|_F \leq \rho \|M - M'\|_F.
\end{equation}
This inequality implies that the Hessian $\nabla^2 \hat{g}(M, w)$ is bounded by $\rho$ in the sense that:
\begin{equation}
\| \nabla^2 \hat{g}(M, w) \|_F \leq \rho.
\end{equation}
Thus, we get the following bound for $\hat{g}(M_p, w)$:
\begin{equation}
\hat{g}(M_p, w) \leq \hat{g}(\hat{M}, w) + \langle \nabla_M \hat{g}(\hat{M}, w), M_p - \hat{M} \rangle + \frac{1}{2} \rho \| M_p - \hat{M} \|_F^2.
\end{equation}
So we have:
\begin{equation}
G^2 \leq \sigma_r^2 \rho^2.
\end{equation}
Based on our the kernel loss \eqref{main0}, we still have the original result: $\sigma_r^2(\hat M) \geq \frac{G^2}{(1+\delta+\zeta_2 q)^2}, \mathrm{and}  G^2 \leq \sigma_r^2 \rho^2$.

\end{proof}

\section{Proof of Theorem~\ref{thm:local}}
\label{sec:local}

Given a matrix $\hat X$, we aim to find the smallest $\delta$ such that there is an instance of the problem with this RIP constant for which $\hat X$ is a local minimizer that is not associated with the ground truth. For notational convenience, we denote this optimal value as $\delta^\ast(\hat X)$. Namely, $\delta^\ast(\hat X)$ is the optimal value to the following optimization problem:
\begin{equation}\label{eq:intro_lmi}
\begin{aligned}
&\min_{\delta,f(\cdot,w)} \quad  \delta \\
&\text{s.t.} \quad  \text{$\hat{X}$ is a local minimizer of $f(\cdot,w)$} , \\
& \text{$f(\cdot,0)$ satisfies the $\delta$-$\RIP_{2r}$ property}.
\end{aligned}
\end{equation}
By the above optimization problem, we know that $\delta \geq \delta^\ast(\hat X)$ for all local minimizers $\hat X$ of $f(\cdot,w)$, where $\delta$ is the best RIP constant of the problem. Since \eqref{eq:intro_lmi} is difficult to analyze, we replace its two constraints with some necessary conditions, thus forming a relaxation of the original problem with its optimal value being a lower bound on $\delta^\ast(\hat X)$.

To find a necessary condition replacing the two constraints, we introduce the following lemma. This is the first lemma that captures the necessary conditions of a critical point of \eqref{eq:main_noisy_problem}, a problem where random noise is considered.
\begin{lemma}
	\label{lem:noisy_foc_nec}
	Assume that the objective function $f(M,w)$ of \eqref{eq:main_noisy_problem} satisfies all assumptions in Section, and that $\hat X$ is a first-order critical point of \eqref{eq:main_noisy_problem}. Then, $\hat X$ must satisfy the following conditions for some symmetric matrix $\mathbf{H} \in \mathbb{R}^{n^2 \times n^2}$:
	\begin{enumerate}
		\item $\|\hat{\mathbf X}^\top \mathbf{H} \mathbf{e} \| \leq 2 \zeta_1 q \|\hat X\|_2$.
		\item $\mathbf{H}$ satisfies the $(\delta+\zeta_2 q)$-RIP$_{2r,2r}$ property, which means that the inequality
		\begin{equation}
		\label{eq:foc_nec_lem_rip}
			(1-\delta-\zeta_2 q) \|M\|^2_F \leq \mathbf m^\top \mathbf{H} \mathbf m \leq (1+\delta+\zeta_2 q)\|M\|^2_F 
		\end{equation}
		holds for every matrix $M \in \mathbb{R}^{n \times n}$ with $\mathrm{rank}(M)\leq 2r$, where$\mathbf m = \vecc(M)$ and $\mathbf{e} = \vecc(\hat X\hat X^\top-M^\ast)$, $\mathbf{\hat X}$ is defined as per Section.
	\end{enumerate}
\end{lemma}
Given Lemma~\ref{lem:noisy_foc_nec}, we can obtain a relaxation of problem \eqref{eq:intro_lmi}, namely the following optimization problem:
	\begin{equation}\label{eq:local_lmi_delta_2r}
\begin{aligned}
\min_{\delta,\mathbf H} \quad & \delta \\
\st \quad & \norm{\hat{\mathbf X}^\top \mathbf H \mathbf e} \leq 2 \zeta_1 q \norm{\hat X}_2, \\
& (1-\delta-\zeta_2 q) \|M\|^2_F \leq \mathbf m^\top \mathbf{H} \mathbf m \leq \\ &(1+\delta+\zeta_2 q)\|M\|^2_F, \ \ \forall M: \mathrm{rank}(M) \leq 2r .
\end{aligned}
\end{equation}
where $\mathbf m = \vecc(M)$. Note that since the second constraint is hard to deal with, so we solve the following problem that has the same optimal value:
	\begin{equation}\label{eq:local_lmi_delta}
\begin{aligned}
\min_{\delta,\mathbf H} \quad & \delta \\
\st \quad & \norm{\hat{\mathbf X}^\top \mathbf H \mathbf e} \leq 2 \zeta_1 q \norm{\hat X}_2, \\
& (1-\delta-\zeta_2 q)I_{n^2} \preceq \mathbf H \preceq (1+\delta+ \zeta_2 q)I_{n^2}.
\end{aligned}
\end{equation}
If the optimal value of \eqref{eq:local_lmi_delta} is denoted as $\delta_f^\ast(\hat X)$, then we know that $\delta_f^\ast(\hat X) \leq \delta^\ast(\hat X) \leq \delta$ due to \eqref{eq:local_lmi_delta_2r} being a relaxation of \eqref{eq:intro_lmi}. By further lower-bounding $\delta_f^\ast(\hat X)$ with an expression in terms of $\|\hat X \hat X^\top - M^\ast\|_F$, we can obtain an upper bound on $\|\hat X \hat X^\top - M^\ast\|_F$. 


\begin{lemma}
\label{lem:noisy_foc_combined}
Let $\hat{X}$ be a first-order critical point of \eqref{eq:main_noisy_problem}, and suppose that $f(M,w)$ satisfies all assumptions stated in~\ref{other_assumption}. Then there exists a symmetric matrix $\mathbf{H} \in \mathbb{R}^{n^2 \times n^2}$ such that:
\begin{itemize}[left = 0em]
\item $\|\hat{\mathbf{X}}^\top \mathbf{H} \mathbf{e}\| \leq 2 \zeta_1 q \|\hat{X}\|_2$.
\item $\mathbf{H}$ satisfies the $(\delta + \zeta_2 q)$-RIP$_{2r,2r}$ property: 
\begin{equation}
(1-\delta-\zeta_2 q)\|M\|_F^2 \leq \mathbf{m}^\top \mathbf{H}\mathbf{m}
\leq (1+\delta+\zeta_2 q)\|M\|_F^2,
\end{equation}
for every $M \in \mathbb{R}^{n\times n}$ with $\mathrm{rank}(M)\leq 2r$, where $\mathbf{m} = \mathrm{vec}(M)$ and $\mathbf{e} = \mathrm{vec}\bigl(\hat{X}\hat{X}^\top - M^\ast\bigr)$.
\end{itemize}
From these conditions, one can form the relaxation of \eqref{eq:intro_lmi} . 
\end{lemma}
Let $\delta_f^\ast(\hat{X})$ be the optimal value of that relaxation. It follows that:
\begin{equation}
\delta_f^\ast(\hat{X}) \leq \delta^\ast(\hat{X}) \leq \delta.
\end{equation}
By bounding $\delta_f^\ast(\hat{X})$ from below in terms of $\|\hat{X}\hat{X}^\top - M^\ast\|_F$, one obtains an upper bound on $\|\hat{X}\hat{X}^\top - M^\ast\|_F$.

\begin{proof}[Proof of Lemma~\ref{lem:noisy_foc_combined}]
	Similar to the last section, we first define $\hat M = \hat X \hat X^\top$. Since $\hat X$ is a first-order critical point, it follows from that $\nabla_X h(\hat X,w) = 0$. Thus,
	\begin{equation}
	\label{eq:foc_nec_lem_help_1}
		0 = \langle \nabla_X h(\hat X,w), U \rangle = \langle \nabla_M f(\hat M,w), \hat X U^\top+U \hat X^\top \rangle,
	\end{equation}
	for an arbitrary $U \in \mathbb{R}^{n \times r}$. Let $u = \vecc(U)$. Next, we define the function $g(\cdot): \mathbb{R}^{n \times n} \mapsto \mathbb{R}$:
	\begin{equation}
		g(V) = \langle \nabla_M f(V,w), \hat X U^\top+U \hat X^\top \rangle,
	\end{equation}
	for all $V \in \mathbb{R}^{n \times n}$. Then, $g(\hat M) = 0$ due to \eqref{eq:foc_nec_lem_help_1}. By the mean-value theorem (MTV), we have:
	\begin{equation}
		\begin{aligned}
			g(\hat M) - g(M^\ast) &= \int_0^1 \langle \nabla g(tM^\ast+(1-t)\hat M), \hat M- M^\ast \rangle \mathrm{d} t \\
			&= \int_0^1 [ \nabla_M^2 f(tM^\ast+(1-t)\hat M)](\hat M- M^\ast,\hat X U^\top+U \hat X^\top)\mathrm{d} t \\
			&= \mathbf{e}^\top \mathbf{H} \mathbf{\hat X} u
		\end{aligned}
	\end{equation}
	where $\mathbf{H} \in \RR^{n^2 \times n^2}$ is a symmetric matrix that is independent of $U$ and satisfies:
	\begin{equation}
		\vecc(K)^\top \mathbf{H} \vecc(L) = \int_0^1 [ \nabla_M^2 f(tM^\ast+(1-t)\hat M)](K,L) \mathrm{d} t
	\end{equation}
	for all $K, L \in \RR^{n \times n}$. This means:
	\begin{equation}
		\mathbf{e}^\top \mathbf{H} \mathbf{\hat X} u = g(
		\hat M) - g(M^\ast).
	\end{equation}
	Taking the absolute value of both sides and upper-bounding the right-hand side gives:
	\begin{equation}
		\begin{aligned}
			|\mathbf{e}^\top \mathbf{H} \mathbf{\hat X} u| &= \|g(
		\hat M) - g(M^\ast)\| \leq \|g(M^\ast)\| \\
		&\leq \zeta_1 q \|\hat X U^\top+U \hat X^\top\|_F \\
		&\leq 2 \zeta_1 q \|\hat X U^\top\|_F \\
		&=2 \zeta_1 q \sqrt{\tr(\hat X \hat X^\top U U^\top)} \\
		&\leq 2 \zeta_1 q \|\hat X\|_2 \|u\|,
		\end{aligned}
	\end{equation}
	where the second line follows from combining and \eqref{eq:noise_perturb_grad}, and the fourth line follows from the cyclic property of trace operators. Choosing $u = \hat{\mathbf X}^\top \mathbf{H} \mathbf{e}$ can simplify the above inequality to 
	\begin{equation}
		\|\hat{\mathbf X}^\top \mathbf{H} \mathbf{e} \| \leq 2 \zeta_1 q \|\hat X\|_2.
	\end{equation}
	Furthermore, the $\delta$-RIP$_{2r,2r}$ property of the objective function means that:
	\begin{equation}
		(1-\delta) \|M\|^2_F \leq [\nabla^2 f(\xi,0)](M,M) \leq (1+\delta) \|M\|^2_F
	\end{equation}
	for all $M$ with $\mathrm{rank}(M) \leq 2r$. Combining with the fact that
	\begin{equation}
		\| \vecc(M)^\top \mathbf{H} \vecc(M) - [\nabla^2 f(\xi,0)](M,M) \| \leq \zeta_2 q \|M\|^2_F,
	\end{equation}
	gives \eqref{eq:foc_nec_lem_rip}.
\end{proof}

\begin{proof}[Proof of the Theorem]
The proof is provided in~\cite{ma2023noisylowrankmatrixoptimization}'s proof od section 3.2. There is some difference, so we provide a slightly different proof here.

To analyze the local condition in Theorem~\ref{thm:continue}, it is helpful to replace the parameter $\delta$ in problem~\eqref{eq:local_lmi_delta} with a new scalar $\eta$ and consider the alternative optimization program
\begin{equation}\label{eq:local_lmi_eta_rewritten}
\begin{aligned}
\max_{\eta,\widehat{\mathbf H}} \quad & \eta \\
\text{s.t.}\quad 
& \|\widehat{\mathbf X}^{\top}\widehat{\mathbf H}\mathbf e\|\le 2\zeta_1 q\|\widehat X\|_2, \\
& \eta I_{n^2}\preceq \widehat{\mathbf H}\preceq I_{n^2}.
\end{aligned}
\end{equation}

Any feasible pair $(\delta,\mathbf H)$ for~\eqref{eq:local_lmi_delta} immediately generates a feasible point for~\eqref{eq:local_lmi_eta_rewritten} through
\begin{equation}
\eta=\frac{1-\delta-\zeta_2 q}{1+\delta+\zeta_2 q},
\qquad
\widehat{\mathbf H}=\frac{1}{1+\delta+\zeta_2 q}\mathbf H.
\end{equation}
If $\eta_f^*(\widehat X)$ denotes the optimal objective value of~\eqref{eq:local_lmi_eta_rewritten}, then
\begin{equation}\label{eq:eta_bound_rewritten}
\eta_f^*(\widehat X)\ge \frac{1-\delta_f^*(\widehat X)-\zeta_2 q}{1+\delta_f^*(\widehat X)+\zeta_2 q}\ge \frac{1-\delta-\zeta_2 q}{1+\delta+\zeta_2 q},
\end{equation}
because every local minimizer satisfies $\delta_f^*(\widehat X)\le \delta^*(\widehat X)\le \delta$.  
Thus the remaining task is to obtain an upper bound on $\eta_f^*(\widehat X)$.

To achieve this, we examine the dual of~\eqref{eq:local_lmi_eta_rewritten}, which takes the form
\begin{equation}\label{eq:dual_eta_rewritten}
\begin{aligned}
\min_{U_1,U_2,G,\lambda,y}\quad 
& \operatorname{tr}(U_2)+4\zeta_1^2 q^2\|\widehat X\|_2^2\lambda+\operatorname{tr}(G) \\
\text{s.t.}\quad 
& \operatorname{tr}(U_1)=1, \\
& (\widehat{\mathbf X}y)\mathbf e^{\top}+\mathbf e(\widehat{\mathbf X}y)^{\top}=U_1-U_2, \\
& \begin{pmatrix} G & -y \\ -y^{\top} & \lambda \end{pmatrix}\succeq 0,\\
& U_1\succeq 0,\quad U_2\succeq 0.
\end{aligned}
\end{equation}

Define
\begin{equation}
M=(\widehat{\mathbf X}y)\mathbf e^{\top}+\mathbf e(\widehat{\mathbf X}y)^{\top},
\end{equation}
and decompose $M$ as $M=[M]_+ - [M]_-$, where both parts are positive semidefinite.  
A feasible dual solution can be constructed by choosing
\begin{equation}
y^\ast=\frac{y}{\operatorname{tr}([M]_+)}, \qquad
U_1^\ast=\frac{[M]_+}{\operatorname{tr}([M]_+)}, \qquad
U_2^\ast=\frac{[M]_-}{\operatorname{tr}([M]_+)},
\end{equation}
and
\begin{equation}
\lambda^\ast=\frac{\|y^\ast\|}{2\zeta_1 q\|\widehat X\|_2}, \qquad
G^\ast=\frac{y^\ast (y^\ast)^{\top}}{\lambda^\ast}.
\end{equation}
Evaluating the dual objective under this choice gives
\begin{equation}\label{eq:dual_obj_rewritten}
\frac{\operatorname{tr}([M]_-)+4\zeta_1 q\|\widehat X\|_2\|y\|}{\operatorname{tr}([M]_+)}.
\end{equation}

Assume $\widehat X$ satisfies $\|\widehat X\widehat X^{\top}-M^\ast\|_F\le \tau\lambda_r(M^\ast)$.  
Then $\widehat X\neq 0$, and for any nonzero $y$ satisfying that $\widehat X^{\top}\operatorname{mat}(y)$ is symmetric, one has
\begin{equation}\label{eq:y_lower}
\|\widehat{\mathbf X}y\|^{2}\ge 2\lambda_{r^\ast}(\widehat X\widehat X^{\top})\|y\|^2.
\end{equation}
Perturbation bounds imply
\begin{equation}
|\lambda_{r^\ast}(\widehat X\widehat X^\top)-\lambda_{r^\ast}(M^\ast)|\le \tau\lambda_r(M^\ast),
\qquad
|\lambda_1(\widehat X\widehat X^\top)-\lambda_1(M^\ast)|\le \tau\lambda_r(M^\ast).
\end{equation}
Combining these with~\eqref{eq:y_lower}, we obtain
\begin{equation}\label{eq:ratio_C_bound}
\frac{2\|\widehat X\|_2\|y\|}{\|\widehat{\mathbf X}y\|}\le 
\sqrt{\frac{2(\lambda_1(M^\ast)+\tau\lambda_r(M^\ast))}{(1-\tau)\lambda_r(M^\ast)}}\equiv C(\tau,M^\ast).
\end{equation}

Let $\theta$ denote the angle between $\widehat{\mathbf X}y$ and $\mathbf e$.  
Then
\begin{equation}
\operatorname{tr}([M]_+)=\|\widehat{\mathbf X}y\|\|\mathbf e\|(1+\cos\theta),
\qquad
\operatorname{tr}([M]_-)=\|\widehat{\mathbf X}y\|\|\mathbf e\|(1-\cos\theta).
\end{equation}
Substituting these into~\eqref{eq:dual_obj_rewritten} together with~\eqref{eq:ratio_C_bound} gives
\begin{equation}
\eta_f^*(\widehat X)\le \frac{1-\cos\theta+2\zeta_1 q\, C(\tau,M^\ast)/\|\mathbf e\|}{1+\cos\theta}.
\end{equation}
Combining this with~\eqref{eq:eta_bound_rewritten} yields
\begin{equation}\label{eq:e_norm_intermediate}
\|\mathbf e\|\le \frac{(1+\delta+\zeta_2 q)\zeta_1 q\, C(\tau,M^\ast)}{\cos\theta-\zeta_2 q-\delta}.
\end{equation}

To bound $\cos\theta$, note that bounding $\sin^2\theta$ suffices.  
Following the standard orthogonal decomposition of $\widehat X$ and $Z$, one eventually arrives at the estimate
\begin{equation}
\sin^2\theta\le \frac{\tau}{2-\tau}\le \tau.
\end{equation}
Since $\tau<1$, this implies
\begin{equation}
\cos\theta\ge \sqrt{1-\tau}.
\end{equation}
Substituting this into~\eqref{eq:e_norm_intermediate} completes the argument.

\end{proof}

\section{Proof of Theorem~\ref{thm:linear_conv}}
\label{sec:linear_cov}
First and foremost, we restate this lemma from :
\begin{lemma}
	\label{lem:dis_conversion}
	For any matrix $X \in \mathbb{R}^{n \times r}$, given a positive semidefinite matrix $M \in \mathbb{R}^{n \times n}$ of rank $r$, we have:
	\begin{equation}
		\label{eq:dist_change}
		\|XX^\top - M\|^2_F \geq 2 (\sqrt{2} -1) \sigma_r(M) (\dist(X,M))^2.
	\end{equation}
\end{lemma}
Also, given Assumption 5, we have
\begin{equation}
	\label{eq:go_nabla_zero}
	\nabla_M f(M^w,w) = 0.
\end{equation}
First, we establish that the PL inequality holds in a neighborhood of the global minimizer.
\begin{lemma}
	\label{lem:pl_ineq}
	Consider the global minimizer $M^w$ of \eqref{eq:main_noisy_problem}. There exists a constant $\mu > 0$ such that the PL inequality:
	\begin{equation}
		\label{eq:pl_ineq}
		\frac{1}{2} \| \nabla_X h(X,w)\|^2_F \geq \mu (h(X,w)-f(\projr(M^w), w)),
	\end{equation}
	holds for all $X \in \mathbb{R}^{n \times r}$ satisfying:
	\begin{align}
		\label{eq:pl_region}
			\dist(X,M^w) < \max\{\sqrt{2(\sqrt 2 -1)} \sqrt{1-(\delta+\zeta_2 q)^2} (\sigma_r(M^w))^{1/2}-D_r,0 \}
	\end{align}
	and
	\begin{equation}
		D_r \leq \dist(X,\projr(M^w)),
	\end{equation}
	for $q < (1-\delta)/\zeta_2$. 
\end{lemma}
	
\begin{proof}[Proof of Lemma~\ref{lem:pl_ineq}]
	We prove the Lemma when $C_w \sqrt{1-(\delta+\zeta_2 q)^2}-D_r > 0$, since otherwise it is trivial. Denote $M \coloneqq XX^\top$. First, we fix a constant $\tilde C$ such that:
\begin{equation}\label{eq:tildeC_ineq}
	\dist(X,M^w) \leq \tilde C < C_w \sqrt{1-(\delta+\zeta_2 q)^2} -D_r.
\end{equation}
Then, we define $q_1$ and $q_2$ as follows:
\begin{equation}
	\begin{aligned}
		q_1 = \sqrt{1- \frac{\tilde C^2}{2(\sqrt 2 -1)\sigma_r(M^w)}}, q_2 = \frac{\sqrt 2 \mu'}{\sigma_r(M^w)^{1/2}-\tilde C}.
	\end{aligned}
\end{equation}
Now, both $q_1$ and $q_2$ are nonnegative resulting from the assumption above. Furthermore, we know that $\delta+\zeta_2 q < \sqrt{1- \frac{\tilde C^2}{2(\sqrt 2-1)\sigma_r(M^w)}}$ from \eqref{eq:tildeC_ineq}, then
\begin{equation}
	\label{eq:pl_ineq_q1q2}
	\frac{1-\delta-\zeta_2 q}{1+\delta + \zeta_2 q} > \frac{1-q_1+q_2}{1+q_1},
\end{equation}
for some small enough $\mu'$. Define $\mu = (\mu')^2/(1+\delta+\zeta_2 q + 2\rho)$. First, we make the assumption that:
\begin{equation}
	\label{eq:inverse_pl}
	\frac{1}{2} \| \nabla_X h(X,w)\|^2_F < \mu (h(X,w)-f(\projr(M^w), w)).
\end{equation}
From this assumption, we have:
\begin{equation}\begin{aligned}
	&\mu (h(X,w)-f(\projr(M^w), w)) \\
    &\leq \mu \left( \langle \nabla_M f(\projr(M^w),w), M- \projr(M^w) \rangle + \frac{1+\delta+\zeta_2 q}{2}\|M-\projr(M^w)\|^2_F \right) \\
	&\leq \mu \left( \rho \|M^w -\projr(M^w)\|_F \|M-\projr(M^w)\|_F + \frac{1+\delta+\zeta_2 q}{2}\|M-\projr(M^w)\|^2_F \right) \\
	&\leq \mu \left( \rho \|M-\projr(M^w)\|^2_F + \frac{1+\delta+\zeta_2 q}{2}\|M-\projr(M^w)\|^2_F \right),
\end{aligned}\end{equation}
due to Taylor's theorem and \eqref{eq:noise_perturb_hess}. So then \eqref{eq:inverse_pl} leads to:
\begin{equation}
	\frac{1}{2} \| \nabla h(X,w)\|^2_F < \mu(\frac{ (1+\delta+\zeta_2q)}{2} + \rho)\|M - \projr(M^w)\|_F^2.
\end{equation}
Therefore, 
\begin{equation}
	\| \nabla h(X,w)\|_F \leq \mu' \|M-\projr(M^w)\|_F.
\end{equation}
Then consider the following optimization problem:
\begin{equation}
	\label{eq:pl_lmi}
	\begin{aligned}
		\min_{\delta,\mathbf{H} \in \mathbb{S}^{n^2}} \quad & \delta \\
\st \quad & \norm{\hat{\mathbf X}^\top \mathbf H \mathbf e} \leq \mu' \|\mathbf e\|, \\
& \text{$\mathbf H$ satisfies the $(\delta+\zeta_2 q)$-$\RIP_{2r}$ property},
	\end{aligned}
\end{equation}
where $\mathbf e = \vecc(XX^\top - \projr(M^w))$. If we denote the optimal value of \eqref{eq:pl_lmi} as $\delta^\ast_f(X,\mu')$, then $\delta^\ast_f(X,\mu') \leq \delta$ because the constraints of \eqref{eq:pl_lmi} are necessary conditions for \eqref{eq:inverse_pl}, according to Lemma 12 of . Therefore,
\begin{equation}
	\frac{1-\delta - \zeta_2 q}{1+\delta+\zeta_2 q} \leq \frac{1-\delta^\ast_f(X,\mu') - \zeta_2 q}{1+\delta^\ast_f(X,\mu')+\zeta_2 q}.
\end{equation}
and:
\begin{equation}\label{eqn:delta_eta}
\eta_f^\ast(\hat X) \geq \frac{1-\delta_f^\ast(\hat X) - \zeta_2 q}{1+\delta_f^\ast(\hat X) + \zeta_2 q}\geq \frac{1-\delta - \zeta_2 q}{1+\delta + \zeta_2 q},
\end{equation}
Moreover, by the same logic of \eqref{eqn:delta_eta}, we know that $\eta_f^\ast(X,\mu') \geq \frac{1-\delta_f^\ast(X,\mu') - \zeta_2 q}{1+\delta_f^\ast(X,\mu') + \zeta_2 q}$, where $\eta_f^\ast(X,\mu')$ is the optimal value of the optimization problem:
\begin{equation}
\begin{aligned}
\max_{\eta,\hat{\mathbf H}} \quad & \eta \\
\st \quad & \norm{\hat{\mathbf X}^\top\hat{\mathbf H}\mathbf e} \leq \mu' \|\mathbf e\|, \\
& \eta I_{n^2} \preceq \hat{\mathbf H} \preceq I_{n^2}.
\end{aligned}
\end{equation}
Gives:
\begin{equation}
	\eta_f^\ast(X,\mu') \leq \frac{1-q_1+q_2}{1+q_1},
\end{equation}
therefore making a contradiction to \eqref{eq:pl_ineq_q1q2}, subsequently proving \eqref{eq:pl_ineq}. 
\end{proof}

\begin{proof}
If we certify that:
	\begin{equation}
		\label{eq:pl_ineq_matrix}
		\frac{\|XX^\top - M^w\|_F}{C_w} <C_w \sqrt{1-(\delta+\zeta_2 q)^2}-D_r
	\end{equation}
	for any given $X \in \mathbb{R}^{n \times r}$, then a direct substitution can certify that \eqref{eq:pl_region} holds for $X$, since by Lemma~\ref{lem:dis_conversion},
	\begin{equation}
		\dist(X,M) \leq \frac{\|XX^\top - M^w\|_F}{C_w}.
	\end{equation}
	Therefore, the certification of\eqref{eq:pl_ineq_matrix} means that the PL inequality \eqref{eq:pl_ineq} holds for this given $X$. Given that is satisfied, then if this inequality holds:
	\begin{equation}
		\|XX^\top - M^w\|_F \leq \sqrt{\frac{1+\delta+\zeta_2 q}{1 -\delta - \zeta_2 q}} \|X_0 X_0^\top - M^w \|_F,
	\end{equation}
	 \eqref{eq:pl_ineq_matrix} will also hold, because:
	 \begin{equation}
	 	\sqrt{\frac{1+\delta+\zeta_2 q}{1 -\delta - \zeta_2 q}} \|X_0 X_0^\top - M^w \|_F \leq C^2_w\sqrt{1-(\delta+\zeta_2 q)^2} -C_w D_r.
	 \end{equation}
	 Thus, for the remainder of the proof, we aim to certify that starting from $X_0$, if we apply the gradient descent algorithm, \eqref{eq:pl_ineq_matrix_simp} will be satisfied every step along this trajectory. In order to do so, we use Taylor's expansion and \eqref{eq:go_nabla_zero} to obtain
	\begin{equation}
		f(M,w) - f(M^w,w) =\frac{[\nabla^2 f(N,w)](M-M^w,M-M^w)}{2},
	\end{equation}
	where $N$ is some convex combination of $M$ and $M^w$, and $M \in \mathbb{R}^{n \times n}$ is any matrix of rank at most $r$. In light of the RIP property of the function and Equation \eqref{eq:noise_perturb_hess}, one can write: 
	\begin{equation}
		\frac{1-\delta-\zeta_2 q}{2} \|M - M^w\|^2_F \leq f(M,w) - f(M^w,w) \leq \frac{1+\delta +\zeta_2 q}{2} \|M - M^w\|^2_F.
	\end{equation}
	This means that if $M_1,M_2 \in \mathbb{R}^{n \times n}$ are two matrices of rank at most $r$ with $f(M_1,w) \leq f(M_2,w)$, then:
	\begin{equation}
		\|M_1 - M^w \|_F \leq \sqrt{\frac{1+\delta+\zeta_2 q}{1 -\delta - \zeta_2 q}} \|M_2 - M^w \|_F,
	\end{equation}
	because $f(M_1,w) - f(M^w,w) \leq f(M_2,w) - f(M^w,w)$. Thus, one can conclude that $f(X_t X_t^\top,w) \leq f(X_0 X_0^\top,w) \ \forall t$, where $X_t$ denotes the $t^{\text{th}}$ step of the gradient descent algorithm starting from $X_0$. Hence, \eqref{eq:pl_ineq_matrix_simp} follows for all $X_t$. Conveniently, Lemma 11 inshows that $f(X_t X_t^\top,0) \leq f(X_{t-1} X_{t-1}^\top,0)$ for all $t \geq 0$. However, this result can be extended to:
	\begin{equation}
		f(X_t X_t^\top,w) \leq f(X_{t-1} X_{t-1}^\top,w),
	\end{equation}
	by making
	\begin{equation}\begin{aligned}
		1/\eta \geq 12 \rho r^{(1/2)}\left(\sqrt{\frac{1+\delta+\zeta_2 q}{1 -\delta - \zeta_2 q}} \|X_0 X_0^\top - M^w \|_F + \|M^w\|_F \right),
	\end{aligned}\end{equation}
	since $\nabla f(\cdot,w)$ is now a $\rho$-Lipschitz continuous function. Given, a sufficient condition to the above inequality is that:
	\begin{equation}
		\eta \leq \left(12 \rho r^{(1/2)}\left(2(\sqrt 2 -1)\sqrt{(1-(\delta+\zeta_2 q)^2}+ \|M^w\|_F \right)\right)^{-1}.
	\end{equation}
	This finally means that the PL inequality \eqref{eq:pl_ineq} is established for the entire trajectory starting from $X_0$. Now, applying Theorem 1 ingives:
	\begin{equation}
		h(X_t,w)-f(\projr(M^w),w) \leq (1-\mu \eta)^\top(h(X_0,w)-f(\projr(M^w),w)),
	\end{equation}
	which implies a linear convergence as desired.
	
\end{proof}

\section{Proof for Theorem~\ref{thm:strict_saddle}}
\label{sec:strict_saddle}

First and foremost, we replace $\delta$ with $\delta+\zeta_2 q$ in all of the proofs since in our noisy formulation, the problem is $(\delta+\zeta_2 q)$-RIP$_{2r,2r}$ instead. Then, we introduce the following Lemma in $\nabla_M f(M^\ast,w) \neq 0$ in the noisy formulation:
\begin{lemma}
	\label{lem:new_lemma6}
	Given a constant $\epsilon > 0$, an arbitrary $X \in \mathbb{R}^{n \times r}$, and the ground truth solution $M^\ast \in \mathbb{R}^{n \times n}$ of \eqref{eq:main_noisy_problem}, if 
	\begin{equation}\label{eq:M_fro_bound}
		\|XX^\top\|^2_F \geq \max \left\{ \frac{2(1+\delta+\zeta_2 q )}{1-\delta - (\zeta_2+\zeta_D) q} \|M^\ast\|^2_F, (\frac{2\lambda \sqrt r}{1-\delta - (\zeta_2+\zeta_D) q})^{4/3} \right\},
	\end{equation}
	then
	\begin{equation}
		\|\nabla_X h(X,w)\|_F \geq \lambda,
	\end{equation}
	where $\zeta_D = \zeta_1/D$ and $D$ is a constant such that 
	\begin{equation}\label{eq:D_bound}
		D^2 \leq (\frac{2\lambda \sqrt r}{1-\delta - (\zeta_2+\zeta_D) q})^{4/3}.
	\end{equation}
\end{lemma}
Note that such $D$ exists since we first require that $1-\delta - (\zeta_2+\zeta_D) q \geq 0$, meaning that $\frac{q\zeta_1}{1-\delta-q\zeta_2 }\leq D$. Moreover, a sufficient condition to \eqref{eq:D_bound} is that $D \leq (2\lambda \sqrt r)^{2/3} $, which can be simultaneously satisfied when $\lambda$ is chosen properly. The introduction of the lower bound $D$ will not affect the remainder of the proof of Theorem~\ref{thm:strict_saddle}, since in the later steps, we only require the existence of a constant $C$ such that $\|XX^\top\|_F \leq C^2$ when $\|\nabla_X h(X,w)\|_F \leq \lambda$. Therefore, Lemma~\ref{lem:new_lemma6} perfectly fits this role.
\begin{proof}[Proof of Lemma~\ref{lem:new_lemma6}]
		Denote $M \coloneqq XX^\top$. Using the RIP property and \eqref{eq:noise_perturb_grad}, we have:
		\begin{equation}\begin{aligned}
			\langle \nabla_M f(M), M \rangle &= \int_0^1 [\nabla^2 f(M^\ast+s(M-M^\ast),w)][M-M^\ast,M] \de s+ \langle \nabla_M f(M^\ast,w), M \rangle \\
			&\geq (1-\delta-\zeta_2 q) \|M\|_F^2 - (1+\delta+\zeta_2 q) \|M^\ast\|_F \|M\|_F- \zeta_1 q \|M\|_F \\
			&= (1-\delta-\zeta_2 q) \|M\|_F^2 - (1+\delta+\zeta_2 q) \|M^\ast\|_F \|M\|_F- \zeta_D q D \|M\|_F \\
			& \geq (1-\delta-(\zeta_2+\zeta_D) q) \|M\|_F^2 - (1+\delta+\zeta_2 q)\|M^\ast\|_F \|M\|_F \\
			& \geq \frac{1-\delta-(\zeta_2+\zeta_D) q}{2} \|M\|_F^2,
		\end{aligned}\end{equation}
		where the second last inequality results from \eqref{eq:D_bound}, which implies that $D \leq \|M\|_F$; and the last inequality follows from \eqref{eq:M_fro_bound}. Then combining the fact that $\|X\|_F \leq \sqrt r \|M\|^{1/2}_F$, and $\|\nabla_X h(X,w)\|_F \geq \frac{\langle \nabla h(X,w),X \rangle }{\|X\|_F}$ yields the desired fact that
		\begin{equation}
			\begin{aligned}
				\|\nabla_X h(X,w)\|_F \geq \frac{\langle \nabla h(X,w),X \rangle }{\|X\|_F} &= \frac{\langle \nabla_M f(M), M \rangle}{\|X\|_F} \\
				&\geq \frac{(1-\delta-(\zeta_2+\zeta_D) q) \|M\|_F^2}{2\sqrt r \|M\|^{1/2}_F} \\
				&= \frac{1-\delta-(\zeta_2+\zeta_D) q}{2\sqrt r} \|M\|^{3/2}_F \\
				&\geq \lambda.
			\end{aligned}
		\end{equation}
\end{proof} 

Then, utilizing Lemma~\ref{lem:new_lemma6}, we can prove in the same fashion to obtain:
\begin{equation}\begin{aligned}
	&\langle \nabla_M f(M,w), M^\ast - M \rangle \\
    &\leq -(1-\delta - \zeta_2 q) \|M-M^\ast\|^2_F - \langle \nabla_M f(M^\ast,w), M-M^\ast \rangle \\
	&\leq -(1-\delta - \zeta_2 q) \|M-M^\ast\|^2_F + \zeta_1 q \|M-M^\ast\|_F \\
	& \leq -(1-\delta - \zeta_2 q) \|M-M^\ast\|^2_F + \zeta_\alpha q (\sqrt{2(\sqrt 2-1)} (\sigma_r(M^\ast))^{1/2} \alpha) \|M-M^\ast\|_F \\
	& \leq -(1-\delta - (\zeta_2-\zeta_\alpha) q) \|M-M^\ast\|^2_F
\end{aligned}\end{equation}
for any $M \in \mathbb{R}^{n \times n}$ that satisfies the requirements in Lemma 7 of . This is because $\|M-M^\ast\|_F \geq (\sqrt{2(\sqrt 2-1)} (\sigma_r(M^\ast))^{1/2} \alpha)$ by the assumption of $\alpha$ and Lemma~\ref{lem:dis_conversion}.

The above change will only affect the constant $c$ in Lemma 7, and the new $c$ will become
\begin{equation}
	c = (\sqrt r \|M^\ast\|_F)^{-1} (\sqrt 2 -1)(1-\delta-(\zeta_2-\zeta_\alpha)q)\sigma_r(M^\ast).
\end{equation}
Since the exact value of $c$ is irrelevant and we only need to prove its existence, the rest of the proof follows from the existing procedure. Note that $c > 0$ is guaranteed by the assumption of noise in Theorem \eqref{thm:strict_saddle}. Then, we proceed to show that it also be proved similarly, except for one key difference, which is:
\begin{equation}
	K \coloneqq (1-3\delta -(3 \zeta_2 +2 \zeta_\alpha) q)(\sqrt 2 -1) \sigma_r(M^\ast) \alpha^2.
\end{equation}
To verify this statement, we leverage the inequality:
\begin{equation}
	-\phi(\bar M) \geq f(M,w) - f(M^\ast,w) - (\delta+\zeta_2 q)\|M-M^\ast\|_F^2,
\end{equation}
and furthermore we now have that:
\begin{equation}\begin{aligned}
	&f(M,w) - f(M^\ast,w) \\
    &\geq \langle \nabla_M f(M^\ast,w), M-M^\ast \rangle + \frac{1-\delta-\zeta_2 q}{2} \|M-M^\ast\|_F^2 \\
	&\geq \frac{1-\delta-\zeta_2 q}{2} \|M-M^\ast\|_F^2 - \zeta_1 q \|M-M^\ast\|_F^2 \\
	&\geq \frac{1-\delta-\zeta_2 q}{2} \|M-M^\ast\|_F^2 - \zeta_\alpha q (\sqrt{2(\sqrt 2-1)} (\sigma_r(M^\ast))^{1/2} \alpha) \|M-M^\ast\|_F^2 \\
	&\geq \frac{1-\delta-(\zeta_2+2\zeta_\alpha) q}{2} \|M-M^\ast\|_F^2
\end{aligned}\end{equation}
for the same reason elaborated above. Combining the above two inequalities leads to:
\begin{equation}
	-\phi(\bar M) \geq \frac{1-3\delta -(3\zeta_2 + 2\zeta_\alpha)q}{2} \|M-M^\ast\|_F^2 \geq K.
\end{equation}
As assumed in Theorem~\ref{thm:strict_saddle}, since $q < \frac{1/3-\delta}{\zeta_2+2\zeta_\alpha/3}$, we know that $K >0$. Finally, we choose $C = (\frac{2(1+\delta+\zeta_2 \epsilon )}{1-\delta - (\zeta_2+\zeta_D) \epsilon} \|M^\ast\|^2_F)^{1/4}$ and invoke Lemmas 6-8 to complete the proof of Theorem~\ref{thm:strict_saddle}. Note the $\epsilon$ here is the same $\epsilon$ appeared in the statement of Theorem~\ref{thm:strict_saddle}.

\section{Theoretical Study of The Combined Loss}\label{sec:combined_theory}
\subsection{Theoretical Study}

\begin{theorem}[Combined Loss Bound]
\label{compositeloss1}
Let $r_i = Y_i - \mathcal{A}(XX^\top)_i$ for $i=1,\dots,n,$ and assume there is a constant $B>0$ such that $r_i^2 \le B$ for all $i.$ Define the combined loss:
\begin{equation}
L_{\mathrm{combined}}(X) = (1-\lambda) \cdot \frac{1}{n} \sum_{i=1}^{n} \Bigl[-\log \Bigl(\tfrac{1}{n}\sum_{j=1}^{n} \exp\Bigl(-\tfrac{\bigl((r_j - r_i)^2\bigr)}{h^2}\Bigr)\Bigr)\Bigr]
+\lambda \cdot \frac{1}{n}\sum_{i=1}^{n} r_i^2.
\end{equation}
Then consider a matrix sensing or low-rank recovery problem where $Y = \mathcal{A}(M^*) + w,$ and $\mathcal{A}$ satisfies a restricted isometry property (RIP)–like condition: there exists $\delta\in (0,1)$ such that, for every symmetric matrix $E$ of rank at most $2r,$ then with probability at least \( \mathbb{P}(\|w\|_2 \leq \epsilon) \) we have:
\begin{equation}
\|XX^\top - M^*\|_F 
\le \frac{1}{\sqrt{1-\delta}}
\sqrt{\frac{L_{\mathrm{combined}}^*}{\lambda} + \frac{1}{n}\epsilon^2}.
\end{equation}
\end{theorem}
The detailed proof and simplified analysis are provided in Appendix~\ref{composite_loss}. We can use similar results as above to show the result. More detailed analysis will be provided in the later version. 

\subsection{Proof of Theorem~\ref{compositeloss1}}\label{composite_loss}

\begin{proof}
The composite loss with random noise is given in \eqref{combined}:
\begin{equation}
\begin{aligned}
&\mathbf{L_{\mathrm{combined}}}(X) \\
&= (1-\lambda) \cdot n^{-1} \sum_{i=1}^n \left( -\log \frac{1}{n} \sum_{j=1}^n \exp\left(-\frac{((Y_j - (\mathcal{A}( X X^\top)_j)) - (Y_i - \mathcal{A}( X X^\top)_i))^2}{h^2}\right) \right) \\
&+ \lambda \cdot \frac{1}{n}\sum_{i=1}^n \bigl(Y_i - \mathcal{A}(X X^\top)_i\bigr)^2.
\end{aligned}
\end{equation}
The composite loss is more robust to different setting. For example, when the $\| w \|$ is small (MSE is better, or when $\| w \|$ is medium, (ours is large). 
\begin{equation}
r_i = Y_i - \mathcal{A}(XX^\top)_i,\quad i=1,\dots,n.
\end{equation}
For convenience, assume that there exists a constant $B>0$ such that $r_i^2 \le B,\quad \text{for all } i=1,\dots,n.$ We recall that the loss is defined as:
\begin{equation}\label{comb}
\begin{aligned}
L_{\mathrm{combined}}(X) &= (1-\lambda) \cdot \frac{1}{n} \sum_{i=1}^{n} \Bigl[-\log \Bigl(\frac{1}{n}\sum_{j=1}^{n} \exp\Bigl(-\frac{\bigl((r_j - r_i)^2\bigr)}{h^2}\Bigr)\Bigr)\Bigr] \\
&\quad + \lambda \cdot \frac{1}{n}\sum_{i=1}^{n} r_i^2.
\end{aligned}
\end{equation}
We now bound each term separately. Since $\frac{1}{n}\sum_{i=1}^n r_i^2 \le B,$ it follows that $\lambda \cdot \frac{1}{n}\sum_{i=1}^{n} r_i^2 \le \lambda B.$ For any $i$ and $j$, because $r_i^2 \le B$ we have $\|r_j - r_i\| \le \|r_j\| + \|r_i\| \le \sqrt{B} + \sqrt{B} = 2\sqrt{B}.$ $(r_j - r_i)^2 \le (2\sqrt{B})^2 = 4B.$ Then, for every $i$ and $j$, $\exp\left(-\frac{(r_j - r_i)^2}{h^2}\right) \ge \exp\left(-\frac{4B}{h^2}\right).$
Hence, for each fixed $i$ we obtain:
\begin{equation}
\frac{1}{n}\sum_{j=1}^{n} \exp\left(-\frac{(r_j - r_i)^2}{h^2}\right) \ge \frac{1}{n}\cdot n\exp\left(-\frac{4B}{h^2}\right) = \exp\left(-\frac{4B}{h^2}\right).
\end{equation}
Taking the negative logarithm yields:
\begin{equation}
-\log\left(\frac{1}{n}\sum_{j=1}^{n}\exp\left(-\frac{(r_j - r_i)^2}{h^2}\right)\right) \le -\log\left(\exp\left(-\frac{4B}{h^2}\right)\right)
=\frac{4B}{h^2}.
\end{equation}
Since this bound holds for each $i$, averaging over $i$ gives:
\begin{equation}\label{boundsf}
\frac{1}{n}\sum_{i=1}^{n}\left[-\log\left(\frac{1}{n}\sum_{j=1}^{n}\exp\left(-\frac{(r_j - r_i)^2}{h^2}\right)\right)\right]
\le \frac{4B}{h^2}.
\end{equation}
Putting the bounds~\eqref{bounds} into the combined loss~\eqref{comb}, we conclude that:
\begin{equation}
L_{\mathrm{combined}}(X)
\le (1-\lambda)\frac{4B}{h^2} + \lambda B.
\end{equation}
Assuming that we are in a matrix sensing/low‐rank recovery setting. In such settings one assumes Assumption~\ref{other_assumption} and~\ref{noise_perturb}. The observed vector is modeled as $Y = \mathcal{A}(M^*) + w,$
where $w$ is a noise vector. The combined loss is given by
\begin{equation}
   \begin{aligned}
   &L_{\mathrm{combined}}(X) = (1-\lambda) \cdot \frac{1}{n} \sum_{i=1}^{n}\Biggl[-\log\Bigl(\frac{1}{n}\sum_{j=1}^{n}\exp\Bigl(-\frac{\bigl((Y_j - \mathcal{A}(XX^\top)_j)-(Y_i-\mathcal{A}(XX^\top)_i)\bigr)^2}{h^2}\Bigr)\Bigr)\Biggr] \\
   & + \lambda\cdot \frac{1}{n}\sum_{i=1}^{n}\Bigl(Y_i-\mathcal{A}(XX^\top)_i\Bigr)^2.
   \end{aligned}
\end{equation}
Further assume that $\mathcal{A}$ satisfies an RIP–like property; that is, there exists $\delta\in (0,1)$ such that for every symmetric matrix $E$ of rank at most $2r$ one has $(1-\delta)\|E\|_F^2 \le \frac{1}{n}\|\mathcal{A}(E)\|_2^2 \le (1+\delta)\|E\|_F^2.$ In our setting we set $E = XX^\top-M^*.$ Denote the quadratic term in the loss by:
\begin{equation}
f_{\mathrm{quad}}(X) = \frac{1}{n}\sum_{i=1}^n\Bigl(Y_i-\mathcal{A}(XX^\top)_i\Bigr)^2 = \frac{1}{n}\|\mathcal{A}(XX^\top-M^*) - w\|_2^2.
\end{equation}
For simplicity (and without loss of generality for a bound) imagine that the optimization is such that the loss is nearly achieved; then the quadratic term is small. In particular, note that:
\begin{equation}
\frac{1}{n}\|\mathcal{A}(XX^\top-M^*)\|_2^2 \le f_{\mathrm{quad}}(X) + \frac{1}{n}\|w\|_2^2.
\end{equation}
If the noise is controlled (or even absent) and if the combined loss is small near a stationary point (say, bounded above by $L_{\mathrm{combined}}^*$), then in particular
\begin{equation}
\frac{1}{n}\|\mathcal{A}(XX^\top-M^*)\|_2^2 \le \frac{L_{\mathrm{combined}}^*}{\lambda} + \frac{1}{n}\|w\|_2^2.
\end{equation}
(In the combined loss the quadratic part appears weighted by $\lambda$.) For clarity we define
\begin{equation}
\varepsilon^2 := \frac{L_{\mathrm{combined}}^*}{\lambda} + \frac{1}{n}\|w\|_2^2.
\end{equation}
Thus,
\begin{equation}
\frac{1}{n}\|\mathcal{A}(XX^\top-M^*)\|_2^2 \le \varepsilon^2.
\end{equation}
The RIP property~\eqref{RIP0} guarantees that:
\begin{equation}
(1-\delta)\|XX^\top-M^*\|_F^2 \le \frac{1}{n}\|\mathcal{A}(XX^\top-M^*)\|_2^2 \le \varepsilon^2.
\end{equation}
Recalling the definition of $\varepsilon$, this yields
\begin{equation}
\|XX^\top-M^*\|_F \le \frac{1}{\sqrt{1-\delta}}\sqrt{\frac{L_{\mathrm{combined}}^*}{\lambda} + \frac{1}{n}\|w\|_2^2}.
\end{equation}
So we have:
\begin{equation}
\|XX^\top-M^*\|_F \le \frac{1}{\sqrt{1-\delta}}\sqrt{\frac{L_{\mathrm{combined}}^*}{\lambda} + \frac{1}{n}\epsilon^2}.
\end{equation}
with probability at least \( \mathbb{P}(\|w\|_2 \leq \epsilon) \).
\end{proof}

\section{Detailed Experimental Results}\label{exp:result}

In this section as above Section~\ref{empirical}, we provide a detailed example to the theoretical results. Assume that $w \in \mathbb{R}^m$ is a $0.05/\sqrt{m}$-sub-Gaussian vector. According to Lemma 1 in \cite{jin2019short}, this choice of $w$ satisfies:
\begin{equation}
1 - 2e^{-\frac{\epsilon^2}{16m\sigma^2}} \leq \mathbb{P}(\|w\|_2 \leq \epsilon),
\end{equation}
with $\sigma = 0.05$. We consider the problem of 1-bit Matrix Completion, which is a low-rank matrix optimization task commonly appearing in recommendation systems with binary inputs~\cite{davenport2014one,ghadermarzy2018learning}. The objective function is given by~\ref{main0}. It is straightforward to verify that the new loss satisfies the assumptions in~\ref{other_assumption} and~\ref{noise_perturb} with $\zeta_1 = 1$ and $\zeta_2 = 0$. In Figures~\ref{fig:enter-label1} to \ref{fig:enter-label6}, we numerically demonstrate and compare the bounds in Theorem~\ref{thm:mainthm} and Theorem~\ref{thm:mainthm_MSE}, with parameters $n=40$, $r=5$. 

\begin{table}[!ht]
\centering
\caption{
$\delta < 1/3$: Real vs. Numerical Errors for the new loss. 
Symbols: $\epsilon$ = probability upper bound, 
$E_{\text{real}} = \|XX^\top - M^*\|_F$, 
$E_{\text{emp}} =$ empirical error, 
$L$ = noise Lipschitz constant, 
$H$ = noise Hessian constant.
}

\setlength{\tabcolsep}{10pt} 
\renewcommand{\arraystretch}{1.0}
\begin{tabular}{rrrrr}
\toprule
\hline
$\epsilon$  
& $E_{\text{real}}$ 
& $E_{\text{emp}}$
& $L$
& $H$ \\
\midrule
0.5 & 0.131347 & 0.245299 & 67.189408 & 62.946586 \\
0.6 & 0.010286 & 0.249747 & 67.795918 & 63.546119 \\
0.7 & 0.010633 & 0.267247 & 70.130878 & 66.388579 \\
0.8 & 0.220458 & 0.272450 & 70.810290 & 69.285569 \\
0.9 & 0.127725 & 0.239321 & 66.365657 & 63.077514 \\
\hline
\bottomrule
\end{tabular}
\end{table}

\begin{figure}[!ht]
    \centering
    \includegraphics[width=\linewidth]{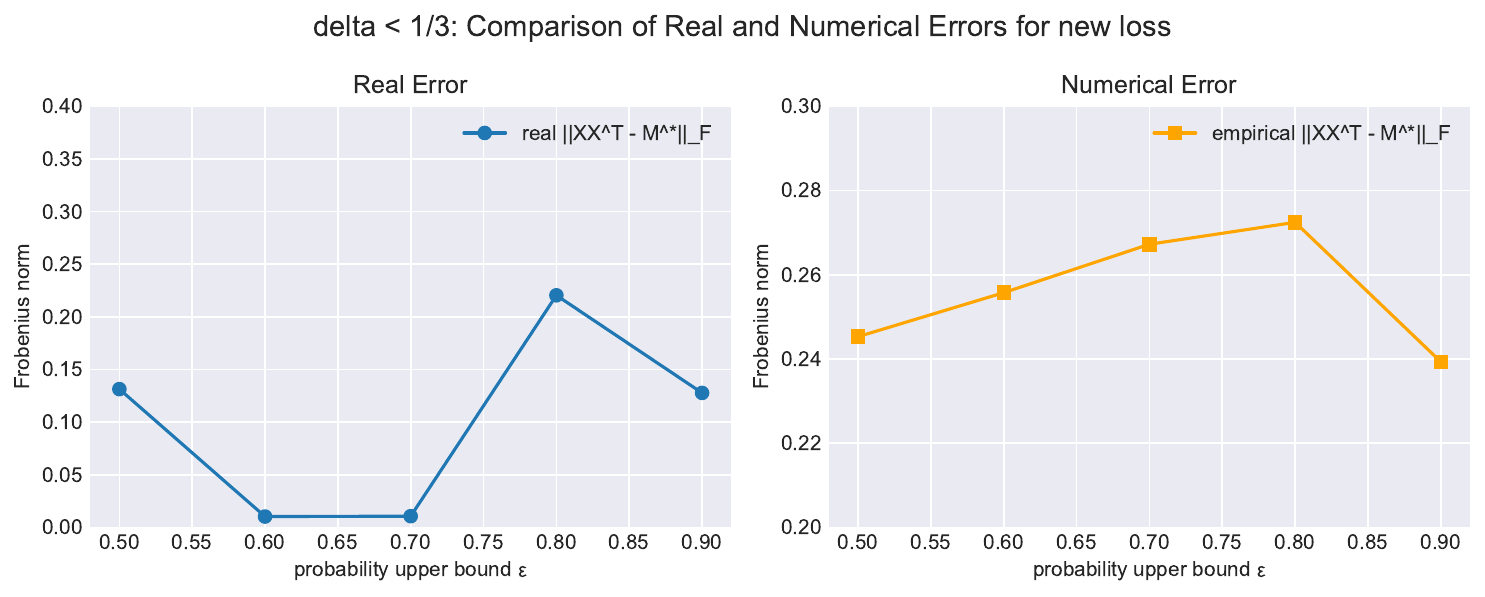}
    \caption{$\delta < 1/3$: Comparison of Real and Numerical Errors for new loss}
    \label{fig:enter-label1}
\end{figure}

\begin{table}[!ht]
\centering
\caption{
$\delta < 1/3$: Real vs. Numerical Errors for the MSE loss. 
Symbols: $\epsilon$ = probability upper bound, 
$E_{\text{real}} = \|XX^\top - M^*\|_F$, 
$E_{\text{emp}} =$ empirical error,
$L$ = noise Lipschitz constant, 
$H$ = noise Hessian constant.
}
\label{tab:setting2}

\setlength{\tabcolsep}{10pt}
\renewcommand{\arraystretch}{1.0}
\begin{tabular}{rrrrr}
\toprule
\hline
$\epsilon$  
& $E_{\text{real}}$ 
& $E_{\text{emp}}$
& $L$
& $H$ \\
\midrule
0.6  & 0.152224 & 0.047391 & 112.673616 & 12.670269 \\
0.65 & 0.117224 & 0.068880 & 111.240372 & 12.566897 \\
0.7  & 0.135116 & 0.095493 & 111.873168 & 12.837108 \\
0.75 & 0.199956 & 0.127546 & 121.697785 & 13.048538 \\
0.8  & 0.173894 & 0.165297 & 115.219867 & 12.904689 \\
0.85 & 0.175881 & 0.208955 & 109.081901 & 13.065476 \\
0.9  & 0.144509 & 0.258679 & 112.820543 & 13.490991 \\
0.95 & 0.164552 & 0.314585 & 114.623283 & 13.444282 \\
\hline
\bottomrule
\end{tabular}
\end{table}

\begin{figure}[!ht]
    \centering
    \includegraphics[width=\linewidth]{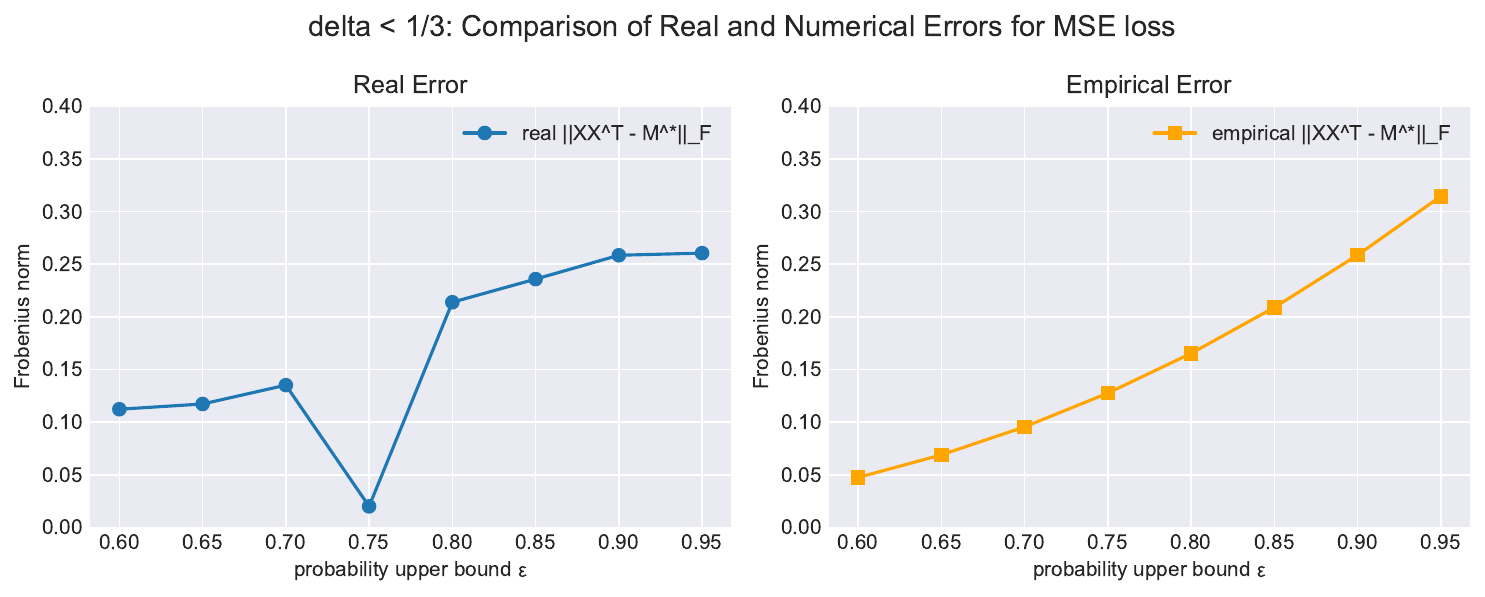}
    \caption{$\delta < 1/3$: Comparison of Real and Numerical Errors for MSE loss}
    \label{fig:enter-label2}
\end{figure}

\begin{table}[!ht]
\centering
\caption{
$\delta < 1/3$: Real vs. Empirical Errors for the composite loss. 
Symbols: $\epsilon$ = probability upper bound, 
$E_{\text{real}} = \|XX^\top - M^*\|_F$, 
$E_{\text{emp}}$ = empirical error, 
$L$ = noise Lipschitz constant, 
$H$ = noise Hessian constant.
}
\renewcommand{\arraystretch}{1.0} 
\setlength{\tabcolsep}{10pt} 
\begin{tabular}{rrrrr}
\toprule
\hline
$\epsilon$  
& $E_{\text{real}}$ 
& $E_{\text{emp}}$
& $L$
& $H$ \\
\midrule
0.6  & 0.252242 & 0.542251 & 29.886229 & 21.379952 \\
0.65 & 0.136733 & 0.126368 & 14.427459 & 13.312156 \\
0.7  & 0.152275 & 0.121407 & 14.141430 & 13.780662 \\
0.75 & 0.137491 & 0.122621 & 14.211950 & 12.999850 \\
0.8  & 0.161383 & 0.123097 & 14.239475 & 12.907365 \\
0.85 & 0.171834 & 0.116507 & 13.853113 & 12.747327 \\
0.9  & 0.159968 & 0.827296 & 36.914869 & 18.623831 \\
0.95 & 0.161281 & 0.124161 & 14.300877 & 13.291003 \\
\hline
\bottomrule
\end{tabular}
\end{table}

\begin{figure}[!ht]
    \centering
    \includegraphics[width=\linewidth]{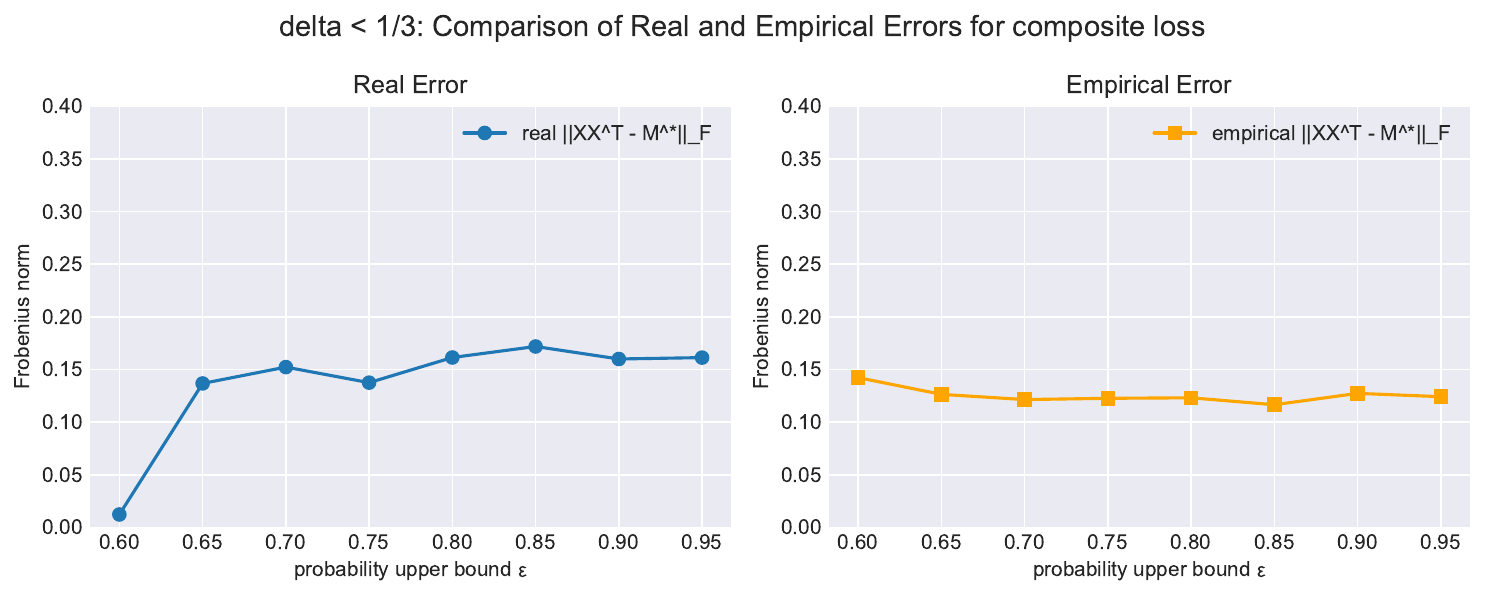}
    \caption{$\delta < 1/3$: Comparison of Real and Empirical Errors for composite loss}
    \label{fig:enter-label3}
\end{figure}

\begin{table}[!ht]
\centering
\caption{
$\delta > 1/3$: Real vs. Empirical Errors for the new loss. 
Symbols: $\epsilon$ = probability upper bound, 
$E_{\text{real}} = \|XX^\top - M^*\|_F$, 
$E_{\text{emp}}$ = empirical error, 
$L$ = noise Lipschitz constant, 
$H$ = noise Hessian constant.
}
\renewcommand{\arraystretch}{1.0} 
\setlength{\tabcolsep}{10pt} 

\begin{tabular}{rrrrr}
\toprule\hline
$\epsilon$
& $E_{\text{real}}$
& $E_{\text{emp}}$
& $L$
& $H$ \\
\midrule
0.5 & 7.102596 & 0.311050 & 7566.029696 & 4847.329349 \\
0.6 & 2.009672 & 0.345135 & 7969.801460 & 5300.450313 \\
0.7 & 0.469574 & 0.233222 & 6551.454917 & 4318.232124 \\
0.8 & 0.074324 & 0.287211 & 7270.322014 & 4534.227116 \\
0.9 & 0.601774 & 0.262751 & 6953.848751 & 4461.188565 \\
\hline\bottomrule
\end{tabular}
\end{table}

\begin{figure}[!ht]
    \centering
    \includegraphics[width=\linewidth]{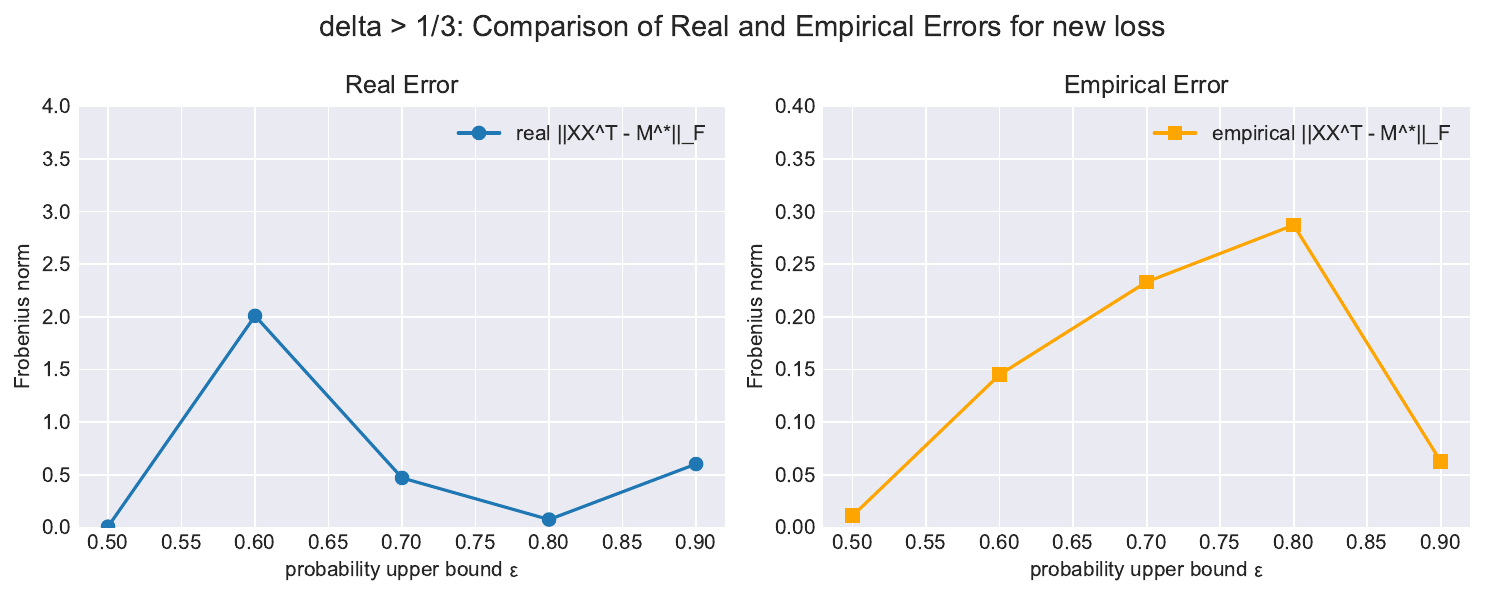}
    \caption{$\delta > 1/3$: Comparison of Real and Empirical Errors for new loss}
    \label{fig:enter-label4}
\end{figure}

\begin{table}[!ht]
\centering
\caption{
$\delta > 1/3$: Real vs. Empirical Errors for MSE loss. 
Symbols: $\epsilon$ = probability upper bound, 
$E_{\text{real}} = \|XX^\top - M^*\|_F$, 
$E_{\text{emp}}$ = empirical error, 
$L$ = noise Lipschitz constant, 
$H$ = noise Hessian constant.
}
\label{tab:setting5}

\renewcommand{\arraystretch}{1.0} 
\setlength{\tabcolsep}{10pt} 

\begin{tabular}{rrrrr}
\toprule\hline
$\epsilon$
& $E_{\text{real}}$
& $E_{\text{emp}}$
& $L$
& $H$ \\
\midrule
0.5 & 0.52011 & 0.275556 & 22.519488 & 34.472289 \\
0.6 & 0.60282 & 0.292817 & 23.214081 & 34.340328 \\
0.7 & 0.70392 & 0.307916 & 23.805072 & 34.244936 \\
0.8 & 0.86332 & 0.297505 & 23.399164 & 34.555703 \\
0.9 & 0.75957 & 0.311414 & 23.939913 & 34.760302 \\
\hline\bottomrule
\end{tabular}
\end{table}

\begin{figure}[!ht]
    \centering
    \includegraphics[width=\linewidth]{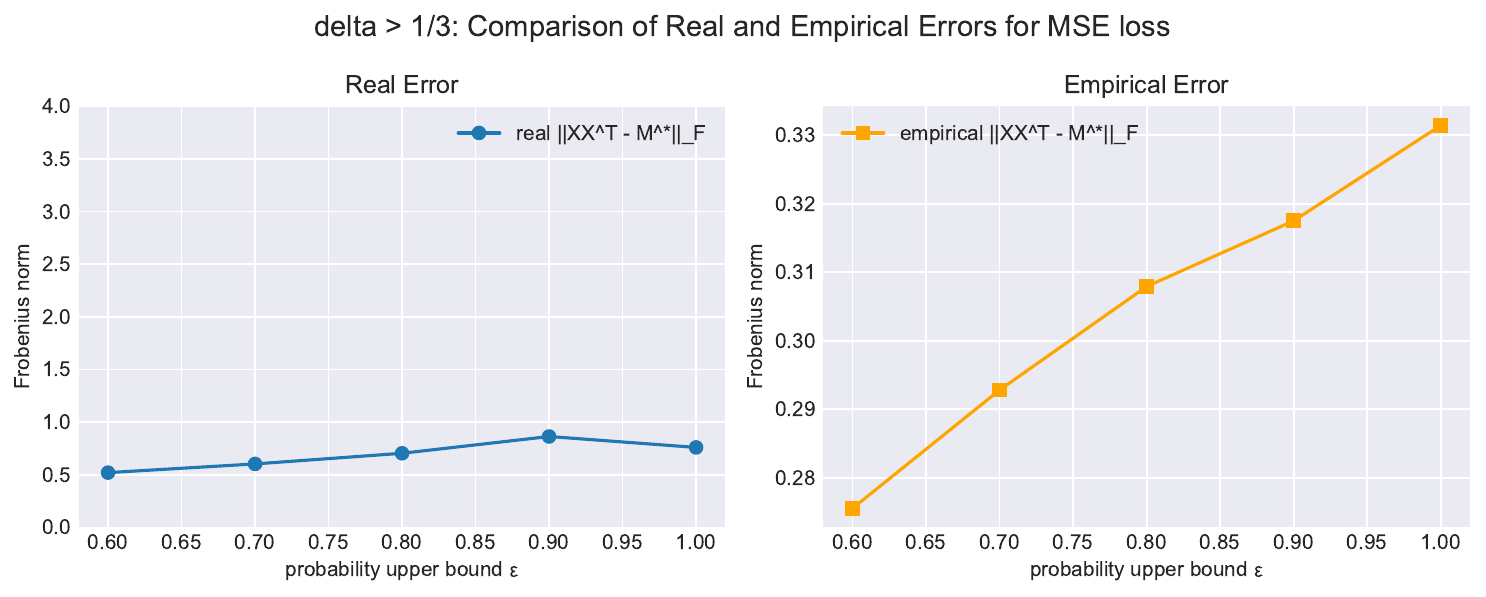}
    \caption{$\delta > 1/3$: Comparison of Real and Numerical Errors for MSE loss}
    \label{fig:enter-label5}
\end{figure}

\begin{table}[!ht]
\centering
\caption{
$\delta > 1/3$: Real vs. Empirical Errors for composite loss. 
Symbols: $\epsilon$ = probability upper bound, 
$E_{\text{real}} = \|XX^\top - M^*\|_F$, 
$E_{\text{emp}}$ = empirical error, 
$L$ = noise Lipschitz constant, 
$H$ = noise Hessian constant.
}
\label{tab:setting6}

\renewcommand{\arraystretch}{0.82} 
\setlength{\tabcolsep}{8pt} 

\begin{tabular}{rrrrr}
\toprule\hline
$\epsilon$
& $E_{\text{real}}$
& $E_{\text{emp}}$
& $L$
& $H$ \\
\midrule
0.5 & 0.100042 & 0.376493 & 26.322789 & 34.088333 \\
0.6 & 0.086220 & 0.420021 & 28.553900 & 34.980704 \\
0.7 & 0.391496 & 0.445266 & 28.626139 & 34.820591 \\
0.8 & 0.140990 & 0.452202 & 28.398186 & 34.680517 \\
0.9 & 0.340462 & 0.456870 & 28.996758 & 34.111456 \\
\hline\bottomrule
\end{tabular}
\end{table}

\begin{figure}[!ht]
    \centering
    \includegraphics[width=\linewidth]{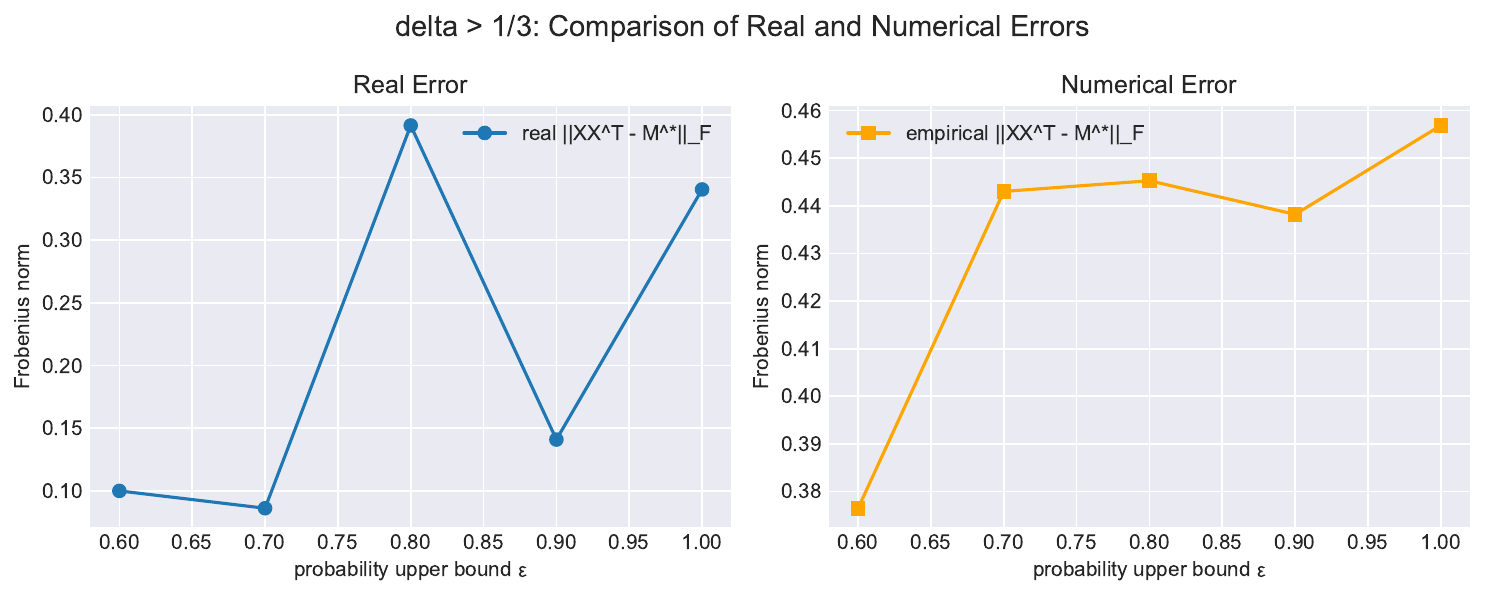}
    \caption{$\delta > 1/3$: Comparison of Real and Numerical Errors for composite loss}
    \label{fig:enter-label6}
\end{figure}

\end{document}